\documentclass[twoside,11pt]{article}

\usepackage{jmlr2e}
\usepackage{amssymb}
\usepackage{natbib}
\usepackage{color}
\usepackage{amsfonts}       
\usepackage{url}
\usepackage{amsmath}
\usepackage{amssymb}
\usepackage{wrapfig}
\usepackage{mathtools}
\usepackage{tikz}
\usepackage{subfig}
\usepackage{algorithmic}
\usepackage{subfig}
\usepackage{footmisc}
\usepackage{bibentry}
\usepackage{thmtools, thm-restate}
\nobibliography*
\usepackage{mathtools}
\DeclarePairedDelimiter{\ceil}{\lceil}{\rceil}
\usepackage{physics}
\usepackage{mathrsfs}
\mathtoolsset{showonlyrefs}
\usepackage{etoolbox}
\usepackage{makecell}
\usepackage[ruled]{algorithm2e}
\DontPrintSemicolon
\usepackage{enumerate}
\usepackage{booktabs}
\usepackage{lastpage}

\newcommand{\fS}{\mathcal{S}}
\newcommand{\fA}{\mathcal{A}}
\newcommand{\fY}{\mathcal{Y}}

\newcommand{\fF}{\mathcal{F}}
\newcommand{\fO}{\mathcal{O}}

\newcommand{\mH}{\mathbb{H}}
\newcommand{\fU}{\mathcal{U}}
\newcommand{\R}{\mathbb{R}}
\newcommand{\E}{\mathbb{E}}
\newcommand{\nsa}{{|\fS \times \fA|}}
\newcommand{\ns}{{|\fS|}}
\newcommand{\na}{{|\fA|}}
\newcommand{\ny}{{|\fY|}}

\newcommand{\tb}[1]{{\textbf{#1}}}

\newcommand{\indot}[2]{{\left<#1, #2\right>}}
\newcommand{\kl}[2]{\text{KL}\left(#1||#2\right)}
\newcommand{\adv}{\text{Adv}}
\newcommand{\ent}[1]{\mathbb{H}\left(#1\right)}

\newcounter{assucounter}
\numberwithin{assucounter}{section}
\newtheorem{assumption}[assucounter]{Assumption}

\jmlrheading{23}{2022}{1-\pageref{LastPage}}{11/21; Revised
10/22}{10/22}{21-1306}{Shangtong Zhang, Remi Tachet des Combes, Romain Laroche}
\ShortHeadings{Softmax Off-Policy Actor Critic under State Distribution Mismatch}{Zhang, Tachet des Combes, and Laroche}

\firstpageno{1}

\begin{document}

\title{Global Optimality and Finite Sample Analysis of Softmax Off-Policy Actor Critic under State Distribution Mismatch}

\author{\name Shangtong Zhang \email shangtong@virginia.edu \\
  \addr University of Virginia\\
  85 Engineer's Way, Charlottesville, VA, 22903, United States\\
  \name Remi Tachet des Combes$^\dagger$
  \email remi.tachet@microsoft.com \\
  \addr Microsoft Research Montreal \\
  6795 Rue Marconi, Suite 400, Montreal, Quebec, H2S 3J9, Canada \\
  \name Romain Laroche$^\dagger$ \email romain.laroche@microsoft.com \\
  \addr Microsoft Research Montreal \\
  6795 Rue Marconi, Suite 400, Montreal, Quebec, H2S 3J9, Canada
  }

\editor{John Shawe-Taylor}

\maketitle

\begin{abstract}
  In this paper,
  we establish the global optimality and convergence rate of an off-policy actor critic algorithm in the tabular setting without using density ratio to correct the discrepancy between the state distribution of the behavior policy and that of the target policy.
  Our work goes beyond existing works on the optimality of policy gradient methods in that
  existing works use the exact policy gradient for updating the policy parameters 
  while we use an approximate and stochastic update step.
  Our update step is not a gradient update because we do not use a density ratio to correct the state distribution,
  which aligns well with what practitioners do.
  Our update is approximate because we use a learned critic instead of the true value function.
  Our update is stochastic because at each step the update is done for only the current state action pair.
  Moreover,
  we remove several restrictive assumptions from existing works in our analysis.
  Central to our work is the finite sample analysis of a generic stochastic approximation algorithm with time-inhomogeneous update operators on time-inhomogeneous Markov chains,
  based on its uniform contraction properties.
  \let\svthefootnote\thefootnote
  \let\thefootnote\relax\footnotetext{$^\dagger$ Equal advising}
  \let\thefootnote\relax\footnotetext{$*$ This version improves the JMLR camera-ready version by removing all the projection operators in the algorithms.}
  \let\thefootnote\svthefootnote
\end{abstract}


\section{Introduction}
Policy gradient methods \citep{williams1992simple}, 
as well as their actor-critic extensions \citep{ sutton2000policy,konda2000actor},
are an important class of Reinforcement Learning  (RL, \citealt{sutton2018reinforcement}) algorithms 
and have enjoyed great empirical success \citep{silver2016mastering, mnih2016asynchronous,vinyals2019grandmaster},
which motivates the importance of the theoretical analysis of policy gradient methods.
Policy gradient and actor-critic methods are essentially stochastic gradient ascent algorithms
and,
therefore, expected to converge to stationary points under mild conditions in on-policy settings,
where an agent selects actions according to its current policy \citep{sutton2000policy,konda2000actor,kumar2019sample,zhang2020global,wu2020finite,xu2020improving,qiu2021finite}.
Off-policy learning is a paradigm where an agent learns a policy of interest, referred to as the target policy,
but selects actions according to a different policy,
referred to as the behavior policy.
Compared with on-policy learning,
off-policy learning exhibits improved sample efficiency \citep{lin1992self,sutton2011horde} and safety \citep{dulac2019challenges}.
In off-policy settings,
the density ratio,
i.e. the ratio between the state distribution of the target policy and that of the behavior policy
\citep{hallak2017consistent,gelada2019off,liu2018breaking,nachum2019dualdice,zhang2020gradientdice},
can be used to correct the state distribution mismatch between the behavior policy and the target policy.
Consequently,
convergence to stationary points of actor-critic methods in off-policy settings with density ratio has also been established \citep{liu2019off,zhang2019provably,huang2021convergence,xu2021doubly}.

The seminal work of \citet{agarwal2019optimality} goes beyond stationary points 
by establishing the global optimality of policy gradient methods in the tabular setting.
\citet{mei2020global} further provide some missing convergence rates.
Both,
however,
use the exact policy gradient instead of an approximate and stochastic gradient,
i.e.,
they assume the value function and the state distribution of the current policy are known 
and query the value function for all states at every iteration.
Despite the aforementioned limitation,
\citet{agarwal2019optimality} still lay the first step towards understanding the global optimality of policy gradient methods.
Their success has also been extended to the off-policy setting by \citet{laroche2021dr}, who,
importantly, consider off-policy actor critic methods without correcting the state distribution mismatch.
Consequently, 
the update step they perform is not a gradient.
This
aligns better with RL practices:
to achieve good performance, practitioners usually do not correct the state distribution mismatch with density ratios for large scale RL experiments \citep{wang2016sample,espeholt2018impala,vinyals2019grandmaster,schmitt2020off,zahavy2020self}.
Still, \citet{laroche2021dr} use exact and expected update steps,
instead of approximate and stochastic update steps. 

In this work,
we go beyond \citet{agarwal2019optimality,laroche2021dr} by establishing the global optimality and convergence rate of an off-policy actor critic algorithm with approximate and stochastic update steps.
Similarly,
we study the off-policy actor critic algorithm in the tabular setting with softmax parameterization of the policy.
Like \citet{laroche2021dr},
we do not use the density ratio to correct the state distribution mismatch.
We, however,
use a learned value function (i.e., approximate updates) and perform stochastic updates
for both the actor and the critic.
Further, we use the KL divergence between a uniformly random policy and the current policy as a regularization with a decaying weight for the actor update.
Our off-policy actor critic algorithm,
therefore, runs in three timescales:
the critic is updated in the fastest timescale;
the actor runs in the middle timescale;
the weight of regularization decays in the slowest timescale.
Besides the advances of using approximate and stochastic  update steps,
we also remove two restrictive assumptions.
The first assumption requires that the initial distribution of the Markov Decision Process (MDP) covers the whole state space,
which is crucial to get the desired optimality in \citet{agarwal2019optimality}. 
The second assumption requires that the optimal policy of the MDP is unique,
which is crucial to get the nonasymptotic convergence rate of \citet{laroche2021dr} for the softmax parameterization.
Thanks to the off-policy learning and the decaying KL divergence regularization,
we are able to remove those two assumptions in our analysis.

One important ingredient of our convergence results is the finite sample analysis of a generic stochastic approximation algorithm with time-inhomogeneous update operators on time-inhomogeneous Markov chains (Section~\ref{sec sa}).
Similar to \citet{chen2021lyapunov},
we rely on the use of the generalized Moreau envelope to form a Lyapunov function.
Our results, however,
extend those of \citet{chen2021lyapunov} from time-homogeneous to time-inhomogeneous Markov chains and
from time-homogeneous to time-inhomogeneous update operators.
Those extensions make our results immediately applicable to the off-policy actor-critic settings (Section~\ref{sec opac}) and
are made possible by establishing a form of uniform contraction of the time-inhomogeneous update operators.
Moreover, we demonstrate that our analysis can also be used for analyzing the soft actor-critic (a.k.a. maximum entropy RL, \citealt{nachum2017bridging,haarnoja2018soft}) under state distribution mismatch (Section~\ref{sec sac}).

\section{Background}
In this paper,
calligraphic letters denote sets and we use vectors and functions interchangeably when it does not confuse, e.g.,
let $f: \fS \to \R$ be a function; 
we also use $f$ to denote the vector in $\R^\ns$ whose $s$-th element is $f(s)$.
All vectors are column.
We use $\norm{\cdot}$ to denote the standard $\ell_2$ norm and $\indot{x}{y} \doteq x^\top y$ for the inner product in Euclidean spaces.
$\norm{\cdot}_p$ is the standard $\ell_p$ norm.
For any norm $\norm{\cdot}_m$,
$\norm{\cdot}^*_m$ denotes its dual norm.

We consider an infinite horizon MDP with a finite state space $\fS$,
a finite action space $\fA$,
a reward function $r: \fS \times \fA \to [-r_{max}, r_{max}]$ for some positive scalar $r_{max}$,
a transition kernel $p: \fS \times \fS \times \fA \to [0, 1]$,
a discount factor $\gamma \in [0, 1)$,
and an initial distribution $p_0: \fS \to [0, 1]$.
At time step 0, an initial state $S_0$ is sampled according to $p_0$.
At time step $t$,
an agent in state $S_t$ takes an action $A_t \sim \pi(\cdot | S_t)$ according to a policy $\pi: \fA \times \fS \to [0, 1]$,
gets a reward $R_{t+1} \doteq r(S_t, A_t)$,
and proceeds to a successor state $S_{t+1} \sim p(\cdot | S_t, A_t)$.
The return at time step $t$ is the random variable
\begin{align}
  G_t \doteq \sum_{i=0}^{\infty} \gamma^i R_{t+i+1},
\end{align}
which allows us to define state- and action-value functions $v_\pi$ and $q_\pi$ as
\begin{align}
  v_\pi(s) &\doteq \E[G_t | S_t = s, \pi, p], \\
  q_\pi(s, a) &\doteq \E[G_t | S_t = s, A_t = a, \pi, p].
\end{align}
The performance of the policy $\pi$ is measured by the expected discounted sum of rewards
\begin{align}
  J(\pi; p_0) \doteq \sum_s p_0(s) v_\pi(s).
\end{align}
Prediction and control are two fundamental tasks of RL.

The goal of prediction is to estimate the values $v_\pi$ or $q_\pi$.
Take estimating $q_\pi$ as an example.
Let $q_t \in \R^\nsa$ be our estimate for $q_\pi$ at time $t$.
SARSA \citep{rummery1994line} updates $\qty{q_t}$ iteratively as
\begin{align}
  \delta_t &\doteq R_{t+1} + \gamma q_{t}(S_{t+1}, A_{t+1}) - q_t(S_t, A_t), \\
  \label{eq on-policy sarsa}
  q_{t+1}(s, a) &\doteq 
  \begin{cases}
    q_t(s, a) + \alpha_t \delta_t, & (s ,a) = (S_t, A_t) \\
    q_t(s, a), & (s, a) \neq (S_t, A_t)
  \end{cases},
\end{align}
where $\delta_t$ is called the temporal difference error \citep{sutton1988learning} and $\qty{\alpha_t}$ is a sequence of learning rates.
It is proved by \cite{bertsekas1996neuro} that, under mild conditions, $\qty{q_t}$ converges to $q_\pi$ almost surely.
So far we have considered on-policy learning,
where the policy of interest is the same as the policy used in action selection.
In the off-policy learning setting, 
the goal is still to estimate $q_\pi$.
Action selection is, however,
done using a different policy $\mu$ (i.e., $A_t \sim \mu(\cdot |S_t)$).
We refer to $\pi$ and $\mu$ as the target and behavior policy respectively.
Off-policy expected SARSA \citep{de2018multi} updates $\qty{q_t}$ iteratively as
\begin{align}
  \delta_t \doteq& R_{t+1} + \gamma \sum_{a'} \pi(a' | S_{t+1}) q_{t}(S_{t+1}, a') - q_t(S_t, A_t), \\
  \label{eq off-policy esarsa}
  q_{t+1}(s, a) \doteq &
  \begin{cases}
    q_t(s, a) + \alpha_t \delta_t, & (s ,a) = (S_t, A_t) \\
    q_t(s, a), & (s, a) \neq (S_t, A_t)
  \end{cases},
\end{align}
where the target policy $\pi$,
instead of the behavior policy $\mu$,
is used to compute the temporal difference error.

The goal of control is to find a policy $\pi_*$ such that $\forall \pi, s$
\begin{align}
  \label{eq optimal policy}
  v_\pi(s) \leq v_{\pi_*}(s).
\end{align}
One common approach for control is policy gradient.
In this paper,
we consider a softmax parameterization for the policy $\pi$.
Letting $\theta \in \R^\nsa$ be the parameters of the policy,
We represent it as
\begin{align}
  \pi(a|s) \doteq \frac{\exp(\theta_{s,a})}{\sum_{a'}\exp(\theta_{s,a'})},
\end{align}
where $\theta_{s,a}$ is the $(s, a)$-indexed element of $\theta$.
Policy gradient methods then update $\theta$ iteratively as
\begin{align}
  \label{eq on-policy pg}
  \theta_{t+1} \doteq \theta_t + \beta_t \nabla_{\theta} J(\pi_{\theta_t}; p_0).
\end{align}
Here $\qty{\beta_t}$ is a sequence of learning rates and
$\pi_\theta$ emphasizes the dependence of the policy $\pi$ on its parameter $\theta$.
In the rest of the paper,
we omit the $\theta$ in $\nabla_\theta$ for simplicity.
\citet{agarwal2019optimality,mei2020global} prove that when $p_0(s) > 0$ holds for all $s$ and $\qty{\beta_t}$ is set properly,
the iterates $\qty{\theta_t}$ generated by \eqref{eq on-policy pg} satisfy
\begin{align}
  \lim_{t\to\infty} J(\theta_t; p_0) = J(\pi_*;p_0),
\end{align}
confirming the optimality of policy gradient methods in the tabular setting with exact gradients.
\citet{mei2020global} also establish a convergence rate for the softmax parameterization.

In practice,
we, however, usually do not have access to $\nabla J(\pi_{\theta_t}; p_0)$.
Fortunately, 
the policy gradient theorem \citep{sutton2000policy} asserts that
\begin{align}
  \nabla J(\pi_{\theta}; p_0) \doteq& \frac{1}{1 - \gamma} \sum_s d_{\pi_\theta, \gamma, p_0}(s) \sum_a q_{\pi_\theta}(s, a) \nabla \pi_\theta(a|s),
\end{align} 
where 
\begin{align}
  d_{\pi, \gamma, p_0} \doteq (1 - \gamma)\sum_{t=0}^\infty \gamma^t \Pr(S_t = s | p_0, \pi)
\end{align}
is the normalized discounted state occupancy measure.
Hence instead of using the gradient update \eqref{eq on-policy pg},
practitioners usually consider the following approximate and stochastic gradient update for the on-policy setting:
\begin{align}
  \label{eq on-policy ac}
  \theta_{t+1} \doteq \theta_t + \beta_t \gamma^t q_t(S_t, A_t) \nabla \log \pi_{\theta_t}(A_t | S_t),
\end{align} 
where $q_t$ is updated according to \eqref{eq on-policy sarsa}.
We refer to \eqref{eq on-policy ac} and \eqref{eq on-policy sarsa} as on-policy actor critic,
where the actor refers to $\pi_\theta$ and the critic refers to $q$.
Usually $\alpha_t$ is much larger than $\beta_t$,
i.e.,
the critic is updated much faster than the actor and the actor is, therefore, 
quasi-stationary from the perspective of the critic.
Consequently,
in the limit,
we can expect $q_t$ to converge to $q_{\pi_{\theta_t}}$,
after which
$\gamma^t q_t(S_t, A_t) \nabla \log \pi_{\theta_t}(A_t | S_t)$ becomes an unbiased estimator of $\nabla J(\pi_{\theta_t}; p_0)$ and the actor update becomes the standard stochastic gradient ascent.

In the off-policy setting,
at time step $t$,
the action selection is done according to some behavior policy $\mu_{\theta_t}$.
Here $\mu_\theta$ does not need to have the same parameterization as $\pi_\theta$,
e.g.,
$\mu_\theta$ can be a softmax policy with a different temperature,
a mixture of a uniformly random policy and a softmax policy,
or a constant policy $\mu$.
To account for the difference between $\pi_\theta$ and $\mu_\theta$,
one must reweight the actor update \eqref{eq on-policy ac} as 
\begin{align}
  \label{eq off-policy ac}
  \theta_{t+1} \doteq \theta_t + \beta_t \varrho_t \rho_t q_t(S_t, A_t) \nabla \log \pi_{\theta_t}(A_t | S_t),
\end{align}
where 
\begin{align}
  \rho_t \doteq \frac{\pi_{\theta_t}(A_t | S_t)}{\mu_{\theta_t}(A_t | S_t)}
\end{align}
is the importance sampling ratio to correct the discrepancy in action selection and
\begin{align}
  \varrho_t \doteq& \frac{d_{\pi_{\theta_t}, \gamma, p_0}(s)}{d_t(s)} \;\text{ with }\;
  d_t(s) \doteq {\Pr(S_t = s | \mu_{\theta_0}, \dots, \mu_{\theta_t})}
\end{align}
is the density ratio to correct the discrepancy in state distribution.
Thanks to $\rho_t$ and $\varrho_t$,
in the limit,
\eqref{eq off-policy ac} is still a stochastic gradient ascent algorithm following the gradient $\nabla J(\pi_{\theta_t}; p_0)$ if $q_t$ converges to $q_{\pi_{\theta_t}}$.
Theoretical analysis of variants of \eqref{eq off-policy ac} includes \citet{liu2019off,zhang2019provably,huang2021convergence,xu2021doubly}.
Practitioners,
however,
usually use only $\rho_t$ but completely ignore $\varrho_t$,
yielding variants of
\begin{align}
  \label{eq off-policy ac with mismatch}
  \theta_{t+1} \doteq \theta_t + \beta_t \rho_t q_t(S_t, A_t) \nabla \log \pi_{\theta_t}(A_t | S_t).
\end{align}
Clearly, 
\eqref{eq off-policy ac with mismatch} can no longer be regarded as a stochastic gradient ascent algorithm even if $q_t$ converges to $q_{\pi_{\theta_t}}$ because of the missing term $\varrho_t$ used to correct the state distribution.
Still, variants of \eqref{eq off-policy ac with mismatch} enjoy great empirical success \citep{wang2016sample,espeholt2018impala,vinyals2019grandmaster,schmitt2020off,zahavy2020self}.
To understand the behavior of \eqref{eq off-policy ac with mismatch},
\citet{laroche2021dr} study the following update rule:
\begin{align}
  \label{eq jh}
  \theta_{t+1} \doteq \theta_t + \beta_t \sum_s d_t(s) \sum_{a}q_{\pi_{\theta_t}}(s, a) \nabla \pi_{\theta_t}(a|s).
\end{align}
Different from \eqref{eq off-policy ac with mismatch},
where the update step is approximate and stochastic,
the update in \eqref{eq jh} is exact and expected.
\citet{laroche2021dr} prove that under mild conditions, 
the iterates $\qty{\theta_t}$ generated by \eqref{eq jh} satisfy
\begin{align}
  \lim_{t\to\infty} J(\pi_{\theta_t};p_0) = J(\pi_*; p_0).
\end{align}
If we further assume the optimal policy $\pi_*$ is unique and $\inf_{s, t} d_t(s) > 0$,
a nonasymptotic convergence rate of \eqref{eq jh} is available.

\section{Stochastic Approximation with Time-Inhomogeneous Operators on Time-Inhomogeneous Markov Chains}
\label{sec sa}
In this section,
we provide finite sample analysis of a generic stochastic approximation algorithm with time-inhomogeneous update operators on time-inhomogeneous Markov chains.
The results presented in this section are used in the analysis of critics in the rest of this work
and may be of independent interest. 

To motivate this part,
consider using off-policy expected SARSA to update the critic in off-policy actor critic.
We have
\begin{align}
  \delta_t \doteq& R_{t+1} + \gamma \sum_{a'} \pi_{\theta_t}(a' | S_{t+1}) q_{t}(S_{t+1}, a') - q_t(S_t, A_t), \\
  q_{t+1}(s, a) \doteq &
  \begin{cases}
    q_t(s, a) + \alpha_t \delta_t, & (s ,a) = (S_t, A_t) \\
    q_t(s, a), & (s, a) \neq (S_t, A_t)
  \end{cases}.
\end{align}
Equivalently,
we can rewrite the above update in a more compact form as
\begin{align}
  \label{eq critic update with F}
  q_{t+1} = q_t + \alpha_t \left(F_{\theta_t}(q_t, S_t, A_t, S_{t+1}) - q_t\right),
\end{align}
where
\begin{align}
  F_\theta(q, s_0, a_0, s_1)[s, a]
  \doteq& \mathbb{I}_{(s_0, a_0) = (s, a)} \delta_\theta(q, s_0, a_0, s_1) + q(s, a),\\
  \delta_\theta(q, s_0, a_0, s_1) 
  \doteq& r(s_0, a_0) + \gamma \sum_{a_1} \pi_{\theta}(a_1 | s_1) q(s_1, a_1) - q(s_0, a_0).
\end{align}
Here, $\mathbb{I}_{statement}(\cdot)$ is the indicator function whose value is 1 if the statement is true, and 0 otherwise.
The update \eqref{eq critic update with F} motivates us to study a generic stochastic approximation algorithm in the form of
\begin{align}
    \label{eq sa iterates}
    w_{t+1} \doteq w_t + \alpha_t (F_{\theta_t}(w_t, Y_t) - w_t + \epsilon_t).
\end{align}
Here $\qty{w_t \in \R^K}$ are the iterates generated by the stochastic approximation algorithm,
$\qty{Y_t}$ is a sequence of random variables evolving in a finite space $\fY$,
$\qty{\theta_t \in \R^L}$ is another sequence of random variables controlling the transition of $\qty{Y_t}$,
$F_\theta$ is a function from $\R^K \times \fY$ to $\R^K$ parameterized by $\theta$,
and $\qty{\epsilon_t \in \R^K}$ is a sequence of random noise. 
The analysis of critics in this paper only requires $\epsilon_t \equiv 0$.
Nevertheless, we consider a generic noise process $\qty{\epsilon_t}$ for generality.

The results in this section extend Theorem 2.1 of \citet{chen2021lyapunov} in two aspects.
First,
the operator $F_\theta$ changes every time step due to the change of $\theta$,
while \citet{chen2021lyapunov} consider a fixed operator $F$.
Second,
the random process $\qty{Y_t}$ evolves according to time-varying dynamics controlled by $\qty{\theta_t}$,
while \citet{chen2021lyapunov} assume $\qty{Y_t}$ is a Markov chain with fixed dynamics.
The introduction of $\qty{\theta_t}$ makes our results immediately applicable to the analysis of actor-critic algorithms.
We now state our assumptions.
It is worth reiterating that all the $\qty{\theta_t}$ below refers to the random sequence used in the update~\eqref{eq sa iterates}.
\begin{assumption}
  \label{assu makovian}
  (Time-inhomogeneous Markov chain)
  There exists a family of parameterized transition matrices $\Lambda_P \doteq \qty{P_{\theta} \in \R^{|\fY| \times |\fY|} | \theta \in \R^L}$ such that
\begin{align}
    \Pr(Y_{t+1} = y) = P_{\theta_{t+1}}(Y_t, y).
\end{align}
\end{assumption}
\begin{assumption}
    \label{assu uniform ergodicity}
    (Uniform ergodicity)
Let $\bar \Lambda_P$ be the closure of $\Lambda_P$.
For any $P \in \bar \Lambda_P$,
the chain induced by $P$ is ergodic.
We use $d_\theta$ to denote the invariant distribution of the chain induced by $P_\theta$.
\end{assumption}

Assumption~\ref{assu makovian} prescribes that the random process $\qty{Y_t}$ is a time-inhomogeneous Markov chain.
It is worth mentioning that Assumption~\ref{assu makovian} does not prescribe how the transition matrices depend on $\qty{\theta_t}$.
It does not restrict $\qty{\theta_t}$ to be deterministic either.
An exemplary parameterization we use in the context of off-policy actor critic will be shown later in \eqref{eq mu example}.
Assumption~\ref{assu uniform ergodicity} prescribes the ergodicity of the Markov chains we consider and was also previously used in the analysis of RL algorithms both in the on-policy \citep{marbach2001simulation} and off-policy settings \citep{zhang2021breaking}.
We will show later that Assumption~\ref{assu uniform ergodicity} is easy to fulfill in our off-policy actor critic setting.
Assumption~\ref{assu uniform ergodicity} implicitly claims that all the matrices in $\bar \Lambda_P$ are stochastic matrices.
This is indeed trivial to prove.
Pick any $P_\infty \in \bar \Lambda_P$.
Since $\bar \Lambda_P$ is the closure of $\Lambda_P$,
there must exist a sequence $\qty{P_n}$ such that $P_n \in \Lambda_P$ and $\lim_{n\to\infty} P_n = P_\infty$.
It is then easy to see that $P_\infty(y, y') \in [0, 1]$ and 
\begin{align}
  \sum_{y'} P_\infty(y, y') = \sum_{y'} \lim_{n\to\infty} P_n(y, y') = \lim_{n\to\infty} \sum_{y'} P_n(y, y') = 1.
\end{align}
In other words, $P_\infty$ is a stochastic matrix.
One important consequence of Assumption~\ref{assu uniform ergodicity} is uniform mixing.
\begin{restatable}{lemma}{uniformmixing}
  \label{lem uniform mixing}
  (Uniform ergodicity implies uniform mixing)
Let Assumption~\ref{assu uniform ergodicity} hold. 
Then, there exist constants $C_0 > 0$ and $\tau \in (0, 1)$,
independent of $\theta$,
such that for any $n > 0$,
\begin{align}
    \label{eq uniform mixing}
    \sup_{y, \theta} \sum_{y'} \abs{P^n_\theta(y, y') -  d_\theta(y')} \leq C_0 \tau^n.
\end{align}
\end{restatable}
\noindent The proof of Lemma~\ref{lem uniform mixing} is provided in Section~\ref{sec proof lem uniform mixing}.
The result in Lemma~\ref{lem uniform mixing} is referred to as uniform mixing since it demonstrates that for any $\theta$,
the chain induced by $P_\theta$ mixes geometrically fast,
with a common rate $\tau$.
For a specific $\theta$,
the existence of a $\theta$-dependent mixing rate $\tau_\theta$ is a well-known result when the chain is ergodic, see, e.g., Theorem 4.9 of \citet{levin2017markov}.
In Lemma~\ref{lem uniform mixing},
we further conclude to the existence of a $\theta$-independent rate.
The ergodicity on the closure $\bar \Lambda_P$ is key to our proof.
If we make ergodicity assumption only on $\Lambda_P$, 
it might be possible to find a sequence $\qty{\theta_t}$ such that the corresponding rates $\qty{\tau_{\theta_t}}$ converges to $1$.
We remark that \eqref{eq uniform mixing} usually appears as a technical assumption directly in many existing works concerning time-inhomogeneous Markov chains, see, e.g., \citet{zou2019finite,wu2020finite}.
In this paper, 
we prove that \eqref{eq uniform mixing} is a consequence of Assumption~\ref{assu uniform ergodicity},
with the help of the extreme value theorem exploiting the compactness of $\bar \Lambda_P$.
We will show in the next section that Assumption~\ref{assu uniform ergodicity} can easily be fulfilled.

\begin{assumption}
    \label{assu uniform contraction}
(Uniform contraction)
For any $\theta \in \R^L$,
define $F_\theta: \R^K \to \R^K$ as 
\begin{align}
    \bar F_\theta(w) \doteq \sum_{y \in \fY} d_\theta(y) F_\theta(w, y).
\end{align}
Then, there exists a constant $\kappa \in (0, 1)$ and a norm $\norm{\cdot}_c$ such that for all $\theta, w, w'$,
\begin{align}
    \norm{\bar F_\theta(w) - \bar F_{\theta}(w')}_c \leq \kappa \norm{w - w'}_c.
\end{align}
We use $w_\theta^*$ to denote the unique fixed point of $\bar F_\theta$.
\end{assumption}
The existence and uniqueness of $w_\theta^*$ follows from the Banach fixed point theorem. 
Assumption~\ref{assu uniform contraction} is another major development beyond \citet{chen2021lyapunov}.
The fact that both $\norm{\cdot}_c$ and $\kappa$ are independent of $\theta$ makes it possible to design a Lyapunov function for our time-inhomogeneous Markov chain.
We will show later that our critic updates indeed satisfy this uniform contraction assumption.

\begin{assumption}
    \label{assu regularization}
(Continuity and boundedness)
There exist positive constants \\$L_F, L_F', L_F'', U_F, U_F', U_F'', L_w, U_w, L_P$ such that for any $w, w', y, y'$ and any time step $t, k$, almost surely,
\begin{enumerate}[(i).]
    \item $\norm{F_{\theta_t}(w, y) - F_{\theta_t}(w', y)}_c \leq L_F \norm{w - w'}_c$ 
    \item $\norm{F_{\theta_t}(w, y) - F_{\theta_k}(w, y)}_c \leq L_F' \norm{\theta_t - \theta_k}_c \left(\norm{w}_c + U_F'\right)$
    \item $\norm{F_{\theta_t}(0, y)}_c \leq U_F$
    \item $\norm{\bar F_{\theta_t}(w) - \bar F_{\theta_k}(w)}_c \leq L_F'' \norm{\theta_t - \theta_k}_c (\norm{w}_c + U_F'')$ 
    \item $\norm{w^*_{\theta_t} - w^*_{\theta_k}}_c \leq L_w \norm{\theta_t - \theta_k}_c$ 
    \item $\sup_{t} \norm{w^*_{\theta_t}}_c \leq U_w $
    \item $\abs{ P_{\theta_t}(y, y') - P_{\theta_k}(y, y') } \leq L_P \norm{\theta_t - \theta_k}_c$
\end{enumerate}
\end{assumption}
\begin{assumption}
  \label{assu mds}
  (Noise)
  Let $\fF_t$ be the $\sigma$-algebra generated by \\$\qty{(w_i, Y_i, \epsilon_i, \theta_i)}_{0 \leq i \leq t-1} \cup \qty{w_t, \theta_t}$, we have
  \begin{enumerate}[(i).]
    \item $\E\left[\epsilon_t \mid \fF_t\right] = 0, \forall t$
    \item There exist positive constants $U_\epsilon, U_\epsilon'$ such that $\forall t, \norm{\epsilon_t}_c \leq U_\epsilon \norm{w_t}_c + U_\epsilon'$
  \end{enumerate}
\end{assumption}
Assumptions~\ref{assu regularization} and \ref{assu mds} are natural extensions of the counterparts in \citet{chen2021lyapunov} from time-homogeneous to time-inhomogeneous Markov chains and
from time-homogeneous to time-inhomogeneous operators.
\begin{assumption}
    \label{assu twotimescale}
    (Two timescales)
    The learning rate $\qty{\alpha_t}$ has the form
    \begin{align}
      \alpha_t \doteq \frac{\alpha}{(t+t_0)^{\epsilon_\alpha}}, 
    \end{align}
    where $\epsilon_\alpha \in (0.5, 1), \alpha > 0, t_0 > 0$ are constants to be tuned.
    Define another sequence $\qty{\beta_t}$ such that
    \begin{align}
    \beta_t \doteq \frac{\beta}{(t+t_0)^{\epsilon_\beta}},
    \end{align}
    where $\epsilon_\beta \in (\epsilon_\alpha, 1], \beta \in (0, \alpha)$ are constants to be tuned.
    Then there exists a constant $L_\theta > 0$ such that $\forall t$, almost surely,
    \begin{align}
      \label{eq konda actor update}
      \norm{\theta_{t+1} - \theta_t}_c \leq \beta_t L_\theta.
    \end{align}
\end{assumption}
Assumption~\ref{assu twotimescale} ensures that the iterates $\qty{w_t}$ evolve sufficiently faster than the change in the dynamics of the chain (i.e., the change of $\qty{\theta_t}$). 
In the off-policy actor critic setting we consider in next section,
$\qty{\alpha_t}$ and $\qty{\beta_t}$ are the learning rates for the critic and the actor respectively.
Though Assumption~\ref{assu twotimescale} explicitly prescribes the form of the sequences $\qty{\alpha_t}$ and $\qty{\beta_t}$,
those are indeed only one of many possible forms (one could e.g., use different $t_0$ for $\qty{\alpha_t}$ and $\qty{\beta_t}$), we consider these particular forms to ease presentation.
We remark that condition in~\eqref{eq konda actor update} is also used in \citet{konda2002thesis},
which gives the asymptotic convergence analysis of the canonical on-policy actor critic with linear function approximation.
We are now ready to state our main results.
\begin{restatable}{theorem}{thmsaconvergence}
    \label{thm sa convergence}
    Let Assumptions~\ref{assu makovian} - \ref{assu twotimescale} hold.
    For any
    \begin{align}
      \epsilon_w \in (0, \min\qty{2(\epsilon_\beta - \epsilon_\alpha), \epsilon_\alpha}),
    \end{align}
    if $t_0$ is sufficiently large,
    then $\forall t$,
    \begin{align}
      \E\left[\norm{w_t - w^*_{\theta_t}}_c^2\right] =\fO\left(\frac{1}{(t+t_0)^{\epsilon_w}}\right).
    \end{align}
\end{restatable}
\noindent See Section~\ref{sec proof thm sa convergence} for the proof of Theorem~\ref{thm sa convergence} and the constants hidden by $\fO(\cdot)$. 
In particular,
we clearly document $t_0$'s dependencies.
One could alternatively set $t_0$ to 0,
then the convergence rate in Theorem~\ref{thm sa convergence} applies only for sufficiently large $t$.
When both the Markov chain and the update operator are time-homogeneous,  
\citet{chen2021lyapunov} demonstrate a convergence rate $\fO\left(\frac{1}{t^{\epsilon_\alpha}}\right)$.
When $2(\epsilon_\beta - \epsilon_\alpha) > \epsilon_\alpha$ holds,
our convergence rate of $\fO\left(\frac{1}{t^{\epsilon_w}}\right)$ can be arbitrarily close to $\fO\left(\frac{1}{t^{\epsilon_\alpha}}\right)$.

\section{Off-Policy Actor Critic with Decaying KL Regularization}
\label{sec opac}
We analyze the optimality of an off-policy actor critic algorithm without correction of the state distribution mismatch (Algorithm~\ref{alg opac}).
Our analysis provides,
to some extent,
a theoretical justification for the practice of ignoring this correction.

\begin{algorithm}
  $S_0 \sim p_0(\cdot)$ \;
  $t \gets 0$ \;
  \While{True}{
    Sample $A_t \sim \mu_{\theta_t}(\cdot | S_t)$ \;
    Execute $A_t$, get $R_{t+1}, S_{t+1}$ \;
    $\delta_t \gets R_{t+1} + \gamma \sum_{a'}\pi_{\theta_t}(a'|S_{t+1}) q_t(S_{t+1}, a')- q_t(S_t, A_t)$ \;
    $q_{t+1}(s, a) \gets \begin{cases}
      q_t(s, a) + \alpha_t \delta_t, & (s, a) = (S_t, A_t) \\
      q_t(s, a), &\text{otherwise}
    \end{cases}$ \;
    $\rho_t \gets \frac{\pi_{\theta_t}(A_t|S_t)}{\mu_{\theta_t}(A_t | S_t)}$ \;
    $\theta_{t+1} \gets \theta_t + \beta_t \bigl(\rho_t  \nabla_\theta \log \pi_{\theta_t}(A_t | S_t) q_t(S_t, A_t) - \lambda_t \nabla_{\theta} \kl{\fU_\fA}{\pi_{\theta_t}(\cdot | S_t)}  \bigr)$ \;
    $t \gets t + 1$ \;
  }
  \caption{\label{alg opac}Off-Policy Actor-Critic with Decaying KL Regularization}
\end{algorithm}
In Algorithm~\ref{alg opac}, 
the target policy $\pi_\theta$ is a softmax policy.
At each time step $t$,
we sample an action $A_t$ according to the behavior policy $\mu_{\theta_t}$.
Importantly,
though the behavior policy is also solely determined by $\theta$,
the parameterization of $\mu_\theta$ can be arbitrarily different from $\pi_\theta$. 
After obtaining the reward $R_{t+1}$ and the successor state $S_{t+1}$,
we update the critic with off-policy expected SARSA,
where $\pi_{\theta_t}$ is used as the target policy for bootstrapping. 
We then update the actor similarly to \eqref{eq off-policy ac with mismatch} without correcting the state distribution mismatch.
The update to $\theta_t$ in Algorithm~\ref{alg opac} is different from \eqref{eq off-policy ac with mismatch} in 
that we use the KL divergence between a uniformly randomly distribution $\fU_\fA$ and the current policy $\pi_{\theta_t}(\cdot | S_t)$ as regularization,
with a decaying weight $\lambda_t$.
The KL divergence is introduced to ensure that the target policy $\pi_\theta$ is sufficiently explorative such that 
there are no bad stationary points (cf. Theorem 5.2 of \citet{agarwal2019optimality}).
In practice,
the entropy of the policy is often used to regularize the policy update \citep{williams1991function,mnih2016asynchronous}.
Here we use the KL divergence instead of the entropy mainly for technical consideration.
We refer the reader to Remark 5.2 of \citet{agarwal2019optimality} for more discussion about this choice. 
The decaying weight $\lambda_t$ is introduced to ensure that, in the limit, the target policy $\pi_\theta$ can still converge to a deterministic policy,
which is a necessary condition for optimality.

Algorithm~\ref{alg opac} runs in three timescales.
The critic runs in the fastest timescale such that it can provide accurate signal for the actor update,
which runs in the middle timescale.
It is then expected that the actor would converge to stationary points whose suboptimality is controlled by $\lambda_t$,
which decays in the slowest timescale.
Finally,
as $\lambda_t$ diminishes,
the suboptimality of the actor decays to 0.
To achieve this three timescale setting,
we make the following assumptions.
\begin{assumption}
  \label{assu three timescales}
  (Three timescales)
  The learning rates $\qty{\alpha_t}, \qty{\beta_t}$ 
  and the weights of KL regularization $\qty{\lambda_t}$
  have the forms
  \begin{align}
    \alpha_t &\doteq \frac{\alpha}{(t+t_0)^{\epsilon_\alpha}}, \,
    \beta_t \doteq \frac{\beta}{(t+t_0)^{\epsilon_\beta}}, \,
    \lambda_t \doteq \frac{\lambda}{(t+t_0)^{\epsilon_\lambda}},
  \end{align} 
  where $0.5 < \epsilon_\alpha < \epsilon_\beta \leq 1, \epsilon_\lambda >0, \alpha > \beta > 0, \lambda > 0, t_0 > 0$ are constants to be tuned. 
\end{assumption}
\begin{assumption}
  \label{assu lr decay speed}
  (Learning rates)
  $2(1-\epsilon_\beta) < \min\qty{2(\epsilon_\beta - \epsilon_\alpha), \epsilon_\alpha}, 0 \leq \epsilon_\lambda < \frac{1 - \epsilon_\beta}{2}$
\end{assumption}
We remark that Assumptions~\ref{assu three timescales} and~\ref{assu lr decay speed} are only one of many possible forms of learning rates 
and we choose this particular form to ease presentation.
To ensure each update to $\theta_t$ does not change the dynamics of the induced Markov chain too fast,
we impose the following assumption on the parameterization of $\mu_\theta$.
\begin{assumption}
  \label{assu lipschitz mu}
  (Lipschitz continuity) There exists $L_\mu > 0$ such that $\forall \theta, \theta', a, s$,
  \begin{align}
    \norm{\mu_\theta(a|s) - \mu_{\theta'}(a|s)} &\leq L_\mu \norm{\theta - \theta'}.
  \end{align}
\end{assumption}
We remark that given the softmax parameterization of $\pi_\theta$,
it is well-known (see, e.g., Lemma~1 of \citealt{wang2020finite}) that $\pi_\theta$ is also Lipschitz continuous,
i.e.,
there exists $L_\pi > 0$ such that $\forall \theta, \theta', a, s$
\begin{align}
    \label{eq pi lipschitz}
    \norm{\pi_\theta(a|s) - \pi_{\theta'}(a|s)} &\leq L_\pi \norm{\theta - \theta'}.
\end{align}
To ensure sufficient exploration,
we impose the following assumption on the behavior policy.
\begin{assumption}
  \label{assu mu uniform ergodicity}
  (Uniform ergodicity)
  Let $\bar \Lambda_\mu$ be the closure of $\qty{\mu_\theta \mid \theta \in \R^\nsa}$.
  For any $\mu \in \bar \Lambda_\mu$,
  the chain induced by $\mu$ is ergodic and $\mu(a|s) > 0$.
\end{assumption}
Assumption~\ref{assu mu uniform ergodicity} is easy to fulfill in practice.
Assuming the chain induced by a uniformly random policy is ergodic,
which we believe is a necessary condition for any assumption regarding ergodicity,
one possible choice for $\mu_\theta$ is to mix an arbitrary behavior policy $\mu_\theta'$ satisfying the Lipschitz continuous requirement with the uniformly random policy,
i.e.,
\begin{align}
  \label{eq mu example}
  \mu_\theta(\cdot | s) \doteq (1 - \epsilon) \fU_\fA + \epsilon \mu_\theta'(\cdot | s)
\end{align}
with any $\epsilon \in (0, 1)$.
From now on,
we use $d_{\mu} \in \R^\ns$ to denote the invariant state distribution of the chain induced by a policy $\mu$
and also overload $d_\mu \in \R^\nsa$ to denote the invariant state action distribution under policy $\mu$. 
With all assumptions stated,
we are ready to present our convergence results.

\subsection{Convergence of the Critic}
In this section,
we study the convergence of the critic by invoking Theorem~\ref{thm sa convergence}
with the update to $q_t$ in Algorithm~\ref{alg opac} expressed as \eqref{eq critic update with F}.
Assumption~\ref{assu uniform contraction} requires us to study the expected operator
\begin{align}
  \bar F_\theta(q) \doteq \sum_{s, a, s'} d_{\mu_\theta}(s) \mu_\theta(a|s) p(s'|s, a) F_\theta(q, s, a, s').
\end{align}
Simple algebraic manipulation yields
\begin{align}
  \label{eq definition of bar f theta}
  \bar F_{\theta}(q) &=  D_{\mu_\theta}(r + \gamma P_{\pi_\theta} q - q) + q \\
  &=(I - D_{\mu_\theta}(I - \gamma P_{\pi_\theta})) q  + D_{\mu_\theta} r,
\end{align}
where $D_{\mu_\theta} \in \R^{\nsa \times \nsa}$ is a diagonal matrix with $D_{\mu_\theta}((s, a), (s, a)) \doteq d_{\mu_\theta}(s) \mu_\theta(a|s)$
and $P_{\pi_\theta} \in \R^{\nsa \times \nsa}$ is the state-action pair transition matrix under policy $\pi_\theta$, i.e.,
\begin{align}
  P_{\pi_\theta}((s, a), (s', a')) \doteq p(s'|s, a)\pi_\theta(a'|s').
\end{align}
We now verify Assumption~\ref{assu uniform contraction} with Lemma \ref{lem uniform contraction opac}.
\begin{restatable}{lemma}{lemuniformcontractionopac}
  \label{lem uniform contraction opac}
  (Uniform contraction)
  Let Assumption~\ref{assu mu uniform ergodicity} hold.
  Then, there exists an $\ell_p$ norm and a constant $\kappa \in (0, 1)$ such that for any $\theta, q, q' \in \R^\nsa$,
  \begin{align}
    \norm{\bar F_\theta(q) - \bar F_\theta(q')}_p \leq \kappa \norm{q - q'}_p.
  \end{align}
  Further, $q_{\pi_\theta}$ is the unique fixed point of $\bar F_\theta$.
\end{restatable}
\noindent The proof of Lemma~\ref{lem uniform contraction opac} is provided in Section~\ref{sec proof lem uniform contraction opac}.
Next,
we are able to prove the convergence of the critic.
\begin{restatable}{proposition}{propcriticconvergence}
  \label{prop critic convergence}
  (Convergence of the critic)
  Let Assumptions~\ref{assu three timescales}, \ref{assu lipschitz mu}, and \ref{assu mu uniform ergodicity} hold.
  For any
  \begin{align}
    \epsilon_q \in (0, \min\qty{2(\epsilon_\beta - \epsilon_\alpha), \epsilon_\alpha}),
  \end{align}
  if $t_0$ is sufficiently large,
  the iterates $\qty{q_t}$ generated by Algorithm~\ref{alg opac} satisfy
  \begin{align}
    \E\left[\norm{q_t - q_{\pi_{\theta_t}}}^2_p\right] = \fO\left(\frac{1}{t^{\epsilon_q}}\right).
  \end{align}
\end{restatable}
\noindent The proof of Proposition~\ref{prop critic convergence} is provided in Section~\ref{sec proof prop critic convergence}.
Proposition~\ref{prop critic convergence} confirms that the critic is able to track the true value function in the limit,
where the dependence between the convergence rate and the mixing parameter of the Markov chains are hidden in $\fO\left(\cdot\right)$.
Similar trackability has also been established in \citet{konda2002thesis,zhang2019provably,wu2020finite}.
Those, however,
rely on the uniform negative-definiteness of the limiting update matrix.
\citet{konda2002thesis} proves that the uniform negative-definiteness holds in the on-policy actor critic with linear function approximation (Lemma 4.18 of \citealt{konda2002thesis}) and establishes this trackability asymptotically.
\citet{wu2020finite} assume the uniform negative-definiteness holds (the second half of Assumption 4.1 of \citealt{wu2020finite}) in the on-policy actor critic with linear function approximation and establish this trackability nonasymptotically.
\citet{zhang2019provably} achieve this uniform negative-definiteness via introducing extra ridge regularization and using full gradients (cf. Gradient TD, \citealt{sutton2009fast})
instead of semi-gradients (cf. TD, \citealt{sutton1988learning}) for the critic update in the off-policy actor critic with function approximation
and achieve this trackability asymptotically.
In our off-policy actor critic setting,
the limiting update matrix of the critic can be computed as
\begin{align}
  D_{\mu_\theta}(\gamma P_{\pi_\theta} - I).
\end{align}
To achieve the desired uniform negative-definiteness,
we would need to prove that there exists a constant $\zeta > 0$ such that for all $x, \theta$,
\begin{align}
  x^\top D_{\mu_\theta}(\gamma P_{\pi_\theta} - I) x \leq -\xi \norm{x}^2.
\end{align}
We,
however,
do not expect the above inequality to hold without making strong assumptions.
Instead,
we resort to uniform contraction.
As demonstrated by Lemma~\ref{lem uniform contraction opac} and Proposition~\ref{prop critic convergence},
uniform contraction is indeed an effective alternative tool for establishing such trackability.
Moreover,
\citet{khodadadian2021finitesample} establish this trackability for a natural actor critic \citep{kakade2001natural} with a Lyapunov method in a quasi-off-policy setting.
The setting \citet{khodadadian2021finitesample} consider is a quasi-off-policy setting in that they prescribe a special form of the behavior policy such that the difference between the behavior policy and the target policy diminishes as time progresses.
By contrast,
we work on a general off-policy setting in that at any time step the behavior policy can always be arbitrarily different from the target policy.
A weaker trackability of the critic can be obtained with the results from \citet{chen2021lyapunov} directly without using our extension (i.e., Theorem~\ref{thm sa convergence}),
as done by \citet{chen2021finitesample,khodadadian2021finite} in their analysis of a natural actor critic.
However,
since \citet{chen2021lyapunov} require both the dynamics of the Markov chain and the update operator to be fixed,
\citet{chen2021finitesample,khodadadian2021finite}
have to keep both the behavior policy and the target policy (actor) fixed when updating the critic.
That being said,
\citet{chen2021finitesample,khodadadian2021finite}
have an inner loop for updating the critic and an outer loop for updating the actor. 
For the critic to be sufficiently accurate,
the inner loop has to take sufficiently many steps.
\citet{chen2021finitesample,khodadadian2021finite}, therefore, have a flavor of bi-level optimization.
Further, 
as long as the steps of the inner loop is finite,
the bias from using a learned critic instead of the true value function will not diminish in the limit.
This bias eventually translates into a suboptimality of the policy that will not vanish in the limit.
By contrast,
Theorem~\ref{thm sa convergence} allows us to consider multi-timescales directly without incurring nested loops,
which ensures that the bias from the critic diminishes in the limit. 

\subsection{Convergence of the Actor}
With the critic able to track the true value function,
we are now ready to present the optimality of the actor.
\begin{restatable}{theorem}{optimalityAC}\label{thm:optimality-AC}
  (Optimality of the actor)
  Let Assumptions~\ref{assu three timescales} - \ref{assu mu uniform ergodicity} hold.
  Fix  
  \begin{align}
    \epsilon_q \in \Big(2(1-\epsilon_\beta), \min\qty{2(\epsilon_\beta - \epsilon_\alpha), \epsilon_\alpha}\Big).
  \end{align}
  Let $t_0$ be sufficiently large.
  For the iterates $\qty{\theta_t}$ generated by Algorithm~\ref{alg opac} and any $t > 0$,
  if $k$ is uniformly randomly selected from the set $\qty{\ceil{\frac{t}{2}}, \ceil{\frac{t}{2}} + 1, \dots, t}$ where $\ceil{\cdot}$ is the ceiling function,
  then 
  \begin{align}
    \label{eq optimality actor sub optimality}
    J(\pi_{\theta_k}; p_0) \geq J(\pi_*; p_0) - \fO\left(\lambda_k \right)
  \end{align}
  holds with probability at least
  \begin{align}
    \label{eq optimality actor high prob}
    1 - \fO\left(\frac{1}{t^{1-\epsilon_\beta - 2\epsilon_\lambda}} + \frac{\log^2 t}{t^{\epsilon_\beta- 2\epsilon_\lambda}} + \frac{1}{t^{\epsilon_q- 2\epsilon_\lambda }}\right),
  \end{align}
  where $\pi_*$ can be any optimal policy.
\end{restatable}
\noindent The proof of Theorem~\ref{thm:optimality-AC} is provided in Section~\ref{sec proof thm:optimality-AC}.
We remark that the $\frac{1}{2}$ in $\ceil{\frac{t}{2}}$ is purely ad-hoc.
We can use any positive constant smaller than 1 and the new rate will be different from the current one in only the constants hidden by $\fO(\cdot)$.
We now optimize the selection of $\epsilon_\alpha$ and $\epsilon_\beta$.
Let $\epsilon_0$ by any positive scalar sufficiently close to 0
and set
\begin{align}
  \label{eq optimal lr}
  \epsilon_\beta = \frac{3}{4} + \epsilon_0, \,\epsilon_\alpha = \frac{1}{2} + \epsilon_0, \,
  \epsilon_q = \frac{1}{2} - \epsilon_0.
\end{align}
Then the high probability in \eqref{eq optimality actor high prob} becomes
\begin{align}
  1 - \fO\left(t^{-\left(\frac{1}{4} - \epsilon_0 - 2\epsilon_\lambda \right)}\right)
\end{align}
and the suboptimality in \eqref{eq optimality actor sub optimality} remains
\begin{align}
  J(\pi_{\theta_k}; p_0) \geq J(\pi_*; p_0) - \fO\left(k^{-\epsilon_\lambda} \right).
\end{align}
It now becomes clear that the selection of 
\begin{align}
  \label{eq optimal reg}
\epsilon_\lambda \in (0, \frac{1}{8})
\end{align}
trades off suboptimality and high probability.
When $\epsilon_\lambda$ is large,
the suboptimality diminishes quickly 
but the high probability approaches one slowly and vice versa. 
To our best knowledge,
Theorem~\ref{thm:optimality-AC} is the first to establish the global optimality and convergence rate
of a naive off-policy actor critic algorithm without density ratio correction even in the tabular setting.
We leave the improvement of the convergence rate for future work.

Importantly, Theorem~\ref{thm:optimality-AC} does not make any assumption on the initial distribution $p_0$.
By contrast,
to obtain the asymptotic optimality in \citet{agarwal2019optimality} or to obtain the convergence rate in \citet{mei2020global},
$p_0(s) > 0$ is assumed to hold for all states.
Both \citet{agarwal2019optimality} and \citet{mei2020global} leave it an open problem
whether $p_0(s) > 0$ is a necessary condition for optimality. 
Our results show that at least in the off-policy setting,
this is not necessary.
The intuition is simple.
Let $p_0'$ be another initial distribution such that $p_0'(s) > 0$ holds for all states.
Then we could optimize $J(\pi_\theta; p_0')$ instead of $J(\pi_\theta; p_0)$
since
the optimal policy w.r.t. $J(\pi_\theta; p_0')$ must also be optimal w.r.t. $J(\pi_\theta; p_0)$.
To optimize $J(\pi_\theta; p_0')$,
we would need samples starting from $p_0'$,
which is impractical in the on-policy setting
since the initial distribution of the MDP is $p_0$.
In the off-policy setting,
we can, however,
use samples starting from $p_0'$ and make corrections with the density ratio.
Since our results show that density ratio correction actually does not matter in the tabular setting we consider,
we can then simply ignore the density ratio,
yielding Algorithm~\ref{alg opac}.
\citet{agarwal2019optimality,mei2020global} refer to the assumption $p_0(s) > 0$ as the sufficient exploration assumption.
Unfortunately,
the initial distribution $p_0$ is usually considered as part of the problem and thus is not controlled by the user.
In our off-policy setting,
we instead achieve sufficient exploration by making assumptions on the behavior policy (Assumption~\ref{assu mu uniform ergodicity}),
which demonstrates the flexibility of off-policy learning in terms of exploration.
Moreover,
to obtain the nonasymptotic convergence rate
of the off-policy actor critic with exact update,
\citet{laroche2021dr} require the optimal policy $\pi_*$ to be unique.
By contrast,
Theorem~\ref{thm:optimality-AC} does not assume any such uniqueness.

\section{Soft Actor Critic}
\label{sec sac}
In this section,
we study the convergence of soft actor critic in the framework of maximum entropy RL,
which penalizes deterministic policies via adding the entropy of the policy into the reward \citep{williams1991function,mnih2016asynchronous,nachum2017bridging,haarnoja2018soft}.
The soft state value function of a policy $\pi$ is defined as
\begin{align}
  \tilde v_{\pi, \eta}(s) \doteq& \E\left[\sum_{i=0}^\infty \gamma^{i} \Big(r(S_{t+i}, A_{t+i}) + \eta \ent{\pi(\cdot | S_{t+i})}  \Big) \mid S_t = s, \pi\right] \\
  =& v_\pi(s) + \eta \E\left[\sum_{i=0}^\infty \gamma^{i}  \ent{\pi(\cdot | S_{t+i})} \mid S_t = s, \pi\right],
\end{align}
where
\begin{align}
  \ent{\pi(\cdot | s)} \doteq -\sum_a \pi(a|s) \log \pi(a|s)
\end{align}
is the entropy and $\eta \geq 0$ is the parameter controlling the strength of entropy regularization.
Correspondingly,
the soft action value function of a policy $\pi$ is defined as
\begin{align}
  \label{eq definition of soft q}
  \tilde q_{\pi, \eta}(s, a) \doteq r(s, a) + \gamma \sum_{s'} p(s'|s, a) \tilde v_{\pi, \eta}(s'),
\end{align}
which satisfies the recursive equation
\begin{align}
  \tilde q_{\pi,\eta}(s, a) = r(s, a) + \gamma \sum_{s', a'} p(s' | s, a) \pi(a'|s') \left(\tilde q_{\pi, \eta}(s, a) - \eta \log \pi(a'|s')\right).
\end{align}
The entropy regularized discounted total rewards is then
\begin{align}
  \label{eq relationship J tilde J}
  \tilde J_\eta(\pi; p_0) \doteq& \sum_s p_0(s) \tilde v_{\pi, \eta}(s) = J(\pi; p_0) + \frac{\eta}{1-\gamma} \sum_s d_{\pi, \gamma, p_0}(s) \ent{\pi(\cdot | s)}.
\end{align}
We still consider the softmax parameterization for the policy $\pi$.
Similar to the canonical policy gradient theorem,
it can be computed \citep{levine2018reinforcement} that
\begin{align}
  \nabla \tilde J_\eta(\pi_\theta; p_0) = \frac{1}{1-\gamma} \sum_s d_{\pi_\theta, \gamma, p_0}(s)  \sum_a \left(\tilde q_{\pi_\theta, \eta}(s, a) - \eta \log \pi_\theta(a|s) \right) \nabla \pi_\theta(a|s).
\end{align}
To get unbiased estimates of $\nabla \tilde J_\eta(\pi_\theta; p_0)$,
one would need to sample states from $d_{\pi_\theta, \gamma, p_0}$,
which is, however,
impractical in off-policy settings.
Practitioners, instead,
directly use states obtained by following the behavior policy (see, e.g., Algorithm~\ref{alg sac}),
yielding a distribution mismatch.
\begin{algorithm}
  $S_0 \sim p_0(\cdot)$ \;
  $t \gets 0$ \;
  \While{True}{
    Sample $A_t \sim \mu_{\theta_t}(\cdot | S_t)$ \;
    Execute $A_t$, get $R_{t+1}, S_{t+1}$ \;
    $\delta_t \gets R_{t+1} + \gamma \sum_{a'}\pi_{\theta_t}(a'|S_{t+1}) \left( q_t(S_{t+1}, a') - \lambda_t \log \pi_{\theta_t}(a' | S_{t+1}) \right)- q_t(S_t, A_t)$ \;
    $q_{t+1}(s, a) \gets \begin{cases}
      q_t(s, a) + \alpha_t \delta_t, & (s, a) = (S_t, A_t) \\
      q_t(s, a), &\text{otherwise}
    \end{cases}$ \;
    $\theta_{t+1} \gets \theta_t + \beta_t \sum_{a} \pi_{\theta_t}(a | S_t)  \nabla_\theta \log \pi_{\theta_t}(a | S_t) \Big( q_t(S_t, a) - \lambda_t \log \pi_{\theta_t}(a | S_t) \Big)$ \;
    $t \gets t + 1$ \;
  }
  \caption{\label{alg sac} Expected Soft Actor-Critic}
\end{algorithm}
In Algorithm~\ref{alg sac},
we still consider the learning rates specified in Assumption~\ref{assu three timescales}
and consider Assumptions~\ref{assu lipschitz mu} and \ref{assu mu uniform ergodicity} for the behavior policy.
Importantly,
in Algorithm~\ref{alg sac},
we consider expected actor updates \citep{ciosek2017expected}
that update the policy for all actions instead of just the executed action $A_t$.
This is mainly for technical consideration.
If we use stochastic update akin to Algorithm~\ref{alg opac},
the update to $\theta_t$ in Algorithm~\ref{alg sac} will have the term $\log \pi_{\theta_t}(A_t|S_t)$.
As $\qty{\lambda_t}$ decreases over time,
we would expect that $\pi_{\theta_t}$ becomes more and more deterministic.
Consequently,
$\abs{\log\pi_{\theta_t}(A_t|S_t)}$ tends to go to infinity,
imposing additional challenges in verifying \eqref{eq konda actor update} unless we ensure $\qty{\lambda_t}$ decays sufficiently fast (e.g., using $\epsilon_\lambda > 1 - \epsilon_\beta$)
such that $\abs{\lambda_t \log\pi_{\theta_t}(A_t|S_t)}$ remains bounded.
By using expected updates instead,
we are able to verify \eqref{eq konda actor update} without imposing any additional condition on $\qty{\lambda_t}$.
We remark that Algorithm~\ref{alg sac} makes expected updates across only actions.
At each time step,
Algorithm~\ref{alg sac} still update the policy only for the current state.
Algorithm~\ref{alg sac} shares the same spirit of the canonical soft actor critic algorithm (Algorithm~1 in \citealt{haarnoja2018soft}).
\citet{haarnoja2018soft} derive the canonical soft actor critic algorithm from a soft policy iteration perspective,
where the policy evaluation of the soft value function 
and the policy improvement of the actor are performed alternatively.
Importantly,
during the soft policy iteration,
both the policy evaluation and the policy improvement steps
are assumed to be fully executed.
By contrast,
the soft actor critic algorithm conduct only several gradient steps for both the policy evaluation and the policy improvement.
As a consequence,
the results concerning the optimality of the soft policy iteration in \citet{haarnoja2018soft} do not apply to soft actor critic.
The convergence of soft actor critic with a fixed regularization weight ($\eta$) remains an open problem, 
and convergence with a decaying regularization, to optimality even more so.
In this work,
we instead derive the soft actor critic algorithm from the 
policy gradient 
perspective directly,
akin to the canonical actor critic,
and
establish its convergence.\footnote{Following existing works, e.g., \citet{konda2002thesis,zhang2019provably,wu2020finite,xu2021doubly},
by convergence of the actor,
we mean that the gradients converge to 0.}

We first study the convergence of $\qty{q_t}$ in Algorithm~\ref{alg sac}.
Different from Algorithm~\ref{alg opac},
the iterates $\qty{q_t}$ now depend on not only $\theta_t$ but also $\lambda_t$.
In light of this,
we consider their concatenation and
define 
\begin{align}
  \zeta_t \doteq \mqty[\lambda_t \\ \theta_t], \, \zeta \doteq \mqty[\eta \\ \theta].
\end{align}
Here $\zeta$ is the placeholder for $\zeta_t$ used for defining functions. 
The update of $\qty{q_t}$ in Algorithm~\ref{alg sac} can then be expressed in a compact way as
\begin{align}
  q_{t+1} = q_t + \alpha_t (F_{\zeta_t}(q_t, S_t, A_t, S_{t+1}) - q_t),
\end{align}
where
\begin{align}
  F_\zeta(q, s_0, a_0, s_1)[s, a] &\doteq  \delta_\zeta(q, s_0, a_0, s_1) \mathbb{I}_{(s_0, a_0) = (s, a)} + q(s, a), \\
  \delta_\zeta(q, s_0, a_0, s_1) &\doteq r(s_0, a_0) + \gamma \sum_{a_1} \pi_{\theta}(a_1 | s_1) \left(q(s_1, a_1) - \eta \log \pi_\theta(a_1 | s_1)\right) - q(s_0, a_0).
\end{align}
We can then establish the convergence of $\qty{q_t}$ similarly to Proposition~\ref{prop critic convergence}.
\begin{restatable}{proposition}{propcriticconvergencesac}
  \label{prop critic convergence sac}
  (Convergence of the critic)
  Let Assumptions~\ref{assu three timescales}, \ref{assu lipschitz mu}, and \ref{assu mu uniform ergodicity} hold.
  Then there exists an $\ell_p$ norm such that
  for any
  \begin{align}
    \epsilon_q \in (0, \min\qty{2(\epsilon_\beta - \epsilon_\alpha), \epsilon_\alpha}),
  \end{align}
  if $t_0$ is sufficiently large,
  the iterates $\qty{q_t}$ generated by Algorithm~\ref{alg opac} satisfy
  \begin{align}
    \E\left[\norm{q_t - \tilde q_{\pi_{\theta_t}, \lambda_t}}^2_p\right] = \fO\left(\frac{1}{t^{\epsilon_q}}\right).
  \end{align}
\end{restatable}
\noindent The proof of Proposition~\ref{prop critic convergence sac} is provided in Section~\ref{sec proof prop critic convergence sac} and is more convoluted than that of Proposition~\ref{prop critic convergence} 
since we now need to verify 
the assumptions of Theorem~\ref{thm sa convergence} for the concatenated vector $\zeta_t$ instead of just $\theta_t$.
With the help of Proposition~\ref{prop critic convergence sac},
we now establish the 
convergence of $\qty{\theta_t}$,
akin to Theorem~\ref{thm:optimality-AC}.

\begin{restatable}{theorem}{thmconvergenceactorsac}
  \label{thm convergence actor sac}
  (Convergence of the actor)
  Let Assumptions \ref{assu three timescales}, \ref{assu lipschitz mu}, and \ref{assu mu uniform ergodicity} hold.
  Fix any 
  \begin{align}
    \epsilon_q \in \Big(0, \min\qty{2(\epsilon_\beta - \epsilon_\alpha), \epsilon_\alpha}\Big).
  \end{align}
  Let $t_0$ be sufficiently large.
  Fix any $\epsilon_0 > 0$ and any state distribution $p_0'$.
  For the iterates $\qty{\theta_t}$ generated by Algorithm~\ref{alg sac} and any $t > 0$,
  if $k$ is uniformly randomly selected from the set $\qty{\ceil{\frac{t}{2}}, \ceil{\frac{t}{2}} + 1, \dots, t}$,
  then
  \begin{align}
    \norm{\nabla \tilde J_{\lambda_k}(\pi_{\theta_k}; p_0')}^2 \leq \frac{1}{k^{\epsilon_0}}
  \end{align}
  holds with at least probability
  \begin{align}
    1 - \fO\left(\frac{1}{t^{1-\epsilon_\beta - \epsilon_0}} + \frac{\log^2t}{t^{\epsilon_\beta-\epsilon_0}} + \frac{1}{t^{\epsilon_q -\epsilon_0}}\right).
  \end{align}
\end{restatable}
\noindent The proof of Theorem~\ref{thm convergence actor sac} is provided in Section~\ref{sec proof thm convergence actor sac}.
Theorem~\ref{thm convergence actor sac} confirms the convergence of the actor to stationary points,
where the additional $\epsilon_0$ trades off the rate at which the gradient vanishes and the rate at which the probability goes to one.
This $\epsilon_0$ is just to present the results and is not a hyperparameter of Algorithm~\ref{alg sac}.
To our best knowledge,
Theorem~\ref{thm convergence actor sac} is the first to establish the convergence of soft actor critic with a decaying entropy regularization weight.

Based on Theorem~\ref{thm convergence actor sac},
the following corollary gives a \emph{partial} result concerning the optimality of Algorithm~\ref{alg sac}. 
\begin{restatable}{corollary}{coroptimalityofactorsac}
  \label{cor optimality of actor sac}
  (Optimality of the actor)
  Let Assumptions \ref{assu three timescales}, \ref{assu lipschitz mu}, and \ref{assu mu uniform ergodicity} hold.
  Fix any 
  \begin{align}
    \epsilon_q \in \Big(0, \min\qty{2(\epsilon_\beta - \epsilon_\alpha), \epsilon_\alpha}\Big).
  \end{align}
  Let $t_0$ be sufficiently large.
  Let $\qty{\delta_t}$ be any positive decreasing sequence converging to 0.
  For the iterates $\qty{\theta_t}$ generated by Algorithm~\ref{alg sac} and any $t > 0$,
  if $k$ is uniformly randomly selected from the set $\qty{\ceil{\frac{t}{2}}, \ceil{\frac{t}{2}} + 1, \dots, t}$,
  then
  \begin{align}
    J(\pi_{\theta_k}; p_0) \geq J(\pi_*; p_0) - \fO\left(\lambda_k\right) - \fO\left(\frac{\delta_k}{\lambda_k \left(\min_{s, a} \pi_{\theta_k}(a|s)\right)^2}\right)
  \end{align}
  holds with at least probability
  \begin{align}
    1 - \frac{\fO\left(t^{-(1-\epsilon_\beta)} + t^{-\epsilon_\beta} \log^2t + t^{-\epsilon_q}\right)}{\delta_t},
  \end{align}
  where $\pi_*$ can be any optimal policy in \eqref{eq optimal policy}. 
\end{restatable}
\noindent The proof of Corollary~\ref{cor optimality of actor sac} is provided in Section~\ref{sec proof cor optimality of actor sac}.
The sequence $\qty{\delta_t}$ in Corollary~\ref{cor optimality of actor sac} trades off the suboptimality and the high probability.
For Corollary~\ref{cor optimality of actor sac} to be nontrivial 
(i.e., the suboptimality diminishes and the high probability approaches one),
one sufficient condition is that
\begin{align}
  \label{eq conjecture}
  \lim_{t\to\infty} \frac{t^{-(1-\epsilon_\beta)} + t^{-\epsilon_\beta} \log^2t + t^{-\epsilon_q}}{\lambda_t \left(\min_{s, a} \pi_{\theta_t}(a|s)\right)^2} = 0.
\end{align}
This requires us to study the decay rate of $\min_{s, a}\pi_{\theta_t}(a|s)$.
We conjecture that when $\lambda_t$ decays slower,
$\min_{s, a}\pi_{\theta_t}(a|s)$ also decays slower.
Consequently, we expect \eqref{eq conjecture} to hold when $\lambda_t$ decays sufficiently slow and the form of the learning rates $\alpha_t$ and $\beta_t$ are adjusted correspondingly according to the form of $\min_{s, a}\pi_{\theta_t}(a|s)$'s decay rate.
We leave the investigation of this rate for future work.

We remark that though Corollary~\ref{cor optimality of actor sac} is only a partial result,
it still advances the state of the art regarding the optimality of soft policy gradient (policy gradient in the maximum entropy RL framework) methods in \citet{mei2020global}.
Theorem~8 of \citet{mei2020global} gives a convergence rate of soft policy gradient methods,
also with a dependence on the rate at which $\min_{s, a}\pi_{\theta_t}(a|s)$ diminishes. They too leave the investigation of the rate as an open problem.
Theorem~8 of \citet{mei2020global}, however,
only considers a bandit setting with the exact soft policy gradient and leaves the general MDP setting for future work.
By contrast,
Corollary~\ref{cor optimality of actor sac} applies to general MDPs with approximate and stochastic update steps.

\section{Related Work}
Our Theorem~\ref{thm sa convergence} regarding the finite sample analysis of stochastic approximation algorithms follows the line of research of \citet{chen2020finite,chen2021lyapunov}.
In particular,
\citet{chen2020finite} consider \eqref{eq sa iterates} 
with an expected operator (i.e., $F_{\theta_t}(w_t, Y_t)$ is replaced by $\bar F(w_t)$).
\citet{chen2021lyapunov} extend \citet{chen2020finite} in that the expected operator is replaced by the stochastic operator $F(w_t, Y_t)$,
though $\qty{Y_t}$ here is a Markov chain with fixed dynamics.
We further extend \citet{chen2021lyapunov} from time-homogeneous stochastic operator and dynamics to time-inhomogeneous stochastic operator and dynamics.
This line of research depends on properties of contraction mappings.
There are also ODE-based analysis for stochastic approximation algorithms \citep{DBLP:books/sp/BenvenisteMP90,kushner2003stochastic,borkar2009stochastic} 
and we refer the reader to \citet{chen2020finite,chen2021lyapunov} for a more detailed review.

In this work,
we focus on the optimality of naive actor critic algorithms that do not use second order information.
With the help of the Fisher information,
the optimality of natural actor critic \citep{kakade2001natural,peters2008natural,bhatnagar2009natural} is also established in both on-policy settings \citep{agarwal2019optimality,wang2019neural,liu2020improved,khodadadian2021finitesample} and off-policy settings \citep{khodadadian2021finite,chen2021finitesample}.
In particular,
\citet{agarwal2019optimality,khodadadian2021finitesample,khodadadian2021finite} establish the optimality of natural actor critic in the tabular setting.
They, however, make synchronous updates to the actor.
In other words,
they update the policy for all states at each time step.
Consequently, the state distribution is not important there.
By contrast, 
the naive actor critic this work considers makes asynchronous updates to the actor.
In other words,
at each time step, 
we only update the policy for the current state.
This asynchronous update is more practical in large scale experiments. 
Moreover,
\citet{xu2021doubly} establish the convergence to stationary points of an off-policy actor critic with density ratio correction and a fixed sampling distribution.
To study the optimality of the stationary points,
\citet{xu2021doubly} also make some assumptions about the Fisher information.
In this work,
we do not use any second order information.
How this work achieves optimality (i.e., vanilla actor critic with decaying KL regularization) is fundamentally different from natural actor critic.

\citet{liu2020improved} improve the results of \citet{agarwal2019optimality} regarding the optimality of policy gradient methods from exact gradient to stochastic and approximate gradient.
\citet{liu2020improved},
however,
work on on-policy settings and require nested loops.
By contrast,
we work on off-policy settings and consider three-timescale updates.

\citet{degris2012off} also study the convergence of an off-policy actor critic without using density ratio to correct the state distribution mismatch.
As noted in the Errata of \citet{degris2012off},
their results also exclusively apply to tabular settings.
Additionally,
\citet{degris2012off} establish asymptotic convergence to only some locally asymptotically stable points of an ODE without any convergence rate. 
And the optimality of those locally asymptotically stable points remains unclear.
Further,
\citet{degris2012off} assume the transitions are identically and independently sampled.
By contrast,
our transitions are obtained by following a time-inhomogeneous behavior policy.

In this paper,
we focus on the tabular setting as a starting point for this line of research. 
When linear function approximation is used for the critic,
compatible features \citep{sutton2000policy,konda2002thesis,zhang2019provably} can be used to eliminate the bias resulting from the limit of the representation capacity.
With the help of compatible features,
\citet{liu2020improved} show the optimality of their on-policy actor critic and
\citet{xu2021doubly} show the optimality of their off-policy actor critic.
We leave the study of linear function approximation in our settings with compatible features for future work.

\section{Experiments}

In this section,
we provide some empirical results in complement to our theoretical analysis. 
The implementation is made publicly available to facilitate future research.\footnote{\url{https://github.com/ShangtongZhang/DeepRL}}
In particular,
we are interested in the following three questions:
\begin{enumerate}[(i).]
  \item Can the claimed convergence and optimality of Algorithm~\ref{alg opac} in Theorem~\ref{thm:optimality-AC} be observed in computational experiments?
  \item Can the claimed convergence of Algorithm~\ref{alg sac} in Theorem~\ref{thm convergence actor sac} and its conjectured optimality from \eqref{eq conjecture} and Corollary~\ref{cor optimality of actor sac} be observed in computational experiments? 
  \item How is the KL-based regularization (cf. Algorithm~\ref{alg opac}) qualitatively different from the entropy-based regularization (cf. Algorithm~\ref{alg sac})?
\end{enumerate}
\begin{figure}[h]
  \centering
  \includegraphics[width=0.6\textwidth]{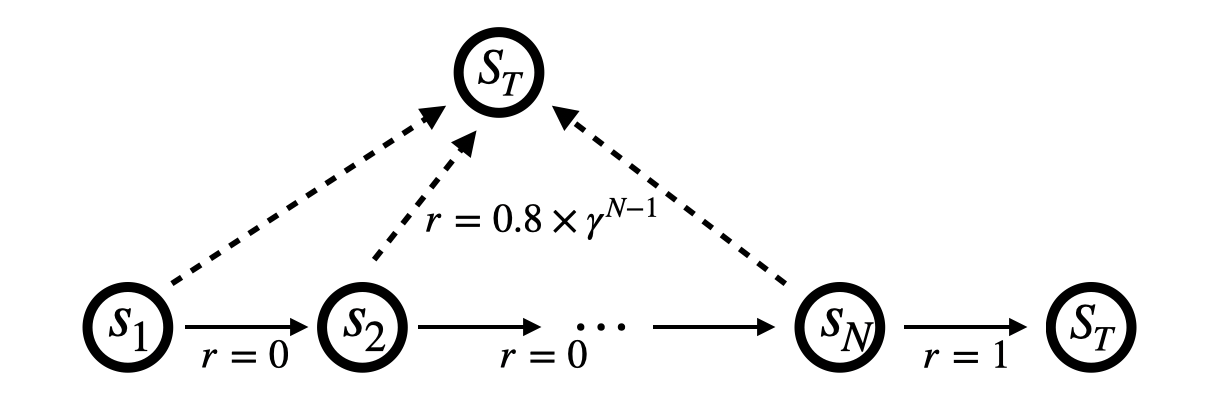}
  \caption{\label{fig chain} The chain domain from \citet{laroche2021dr} with $\gamma = 0.99$.}
\end{figure}

\begin{figure}[t]
  \centering
  \includegraphics[width=\textwidth]{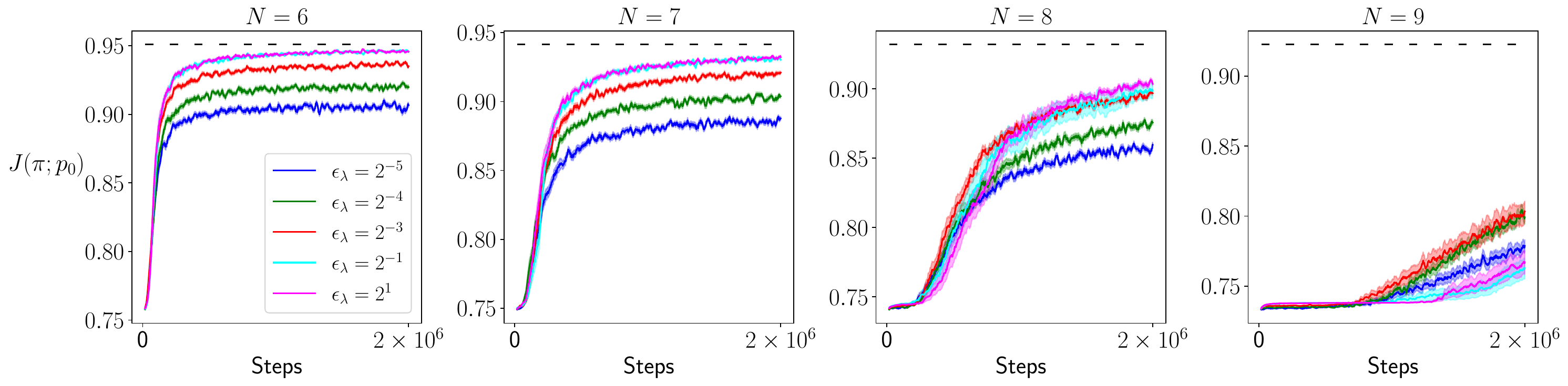}
  \caption{\label{fig opac} Evaluation performance against training steps of Algorithm~\ref{alg opac}}
\end{figure}
\begin{figure}[t]
  \centering
  \includegraphics[width=\textwidth]{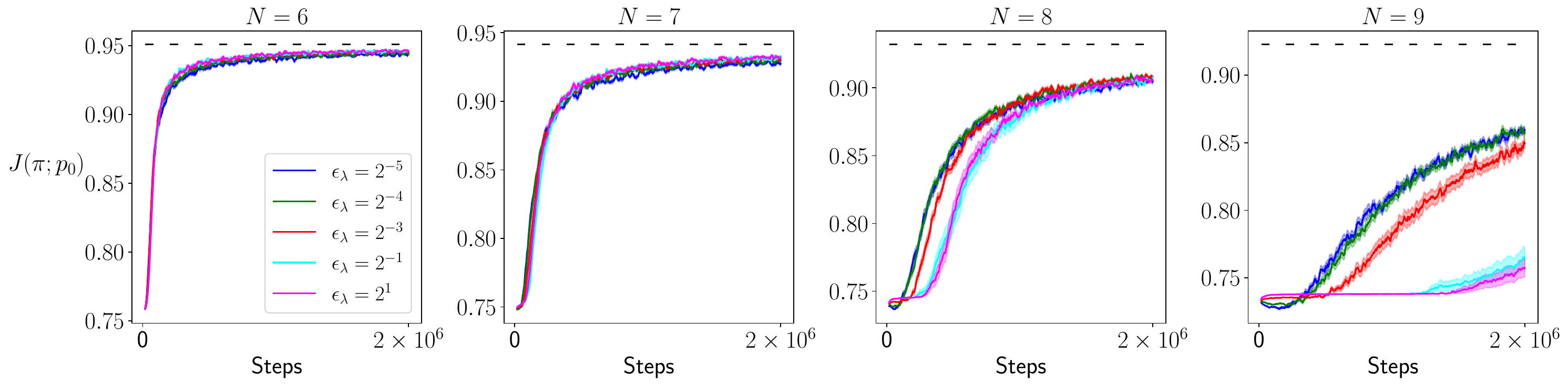}
  \caption{\label{fig sac} Evaluation performance against training steps of Algorithm~\ref{alg sac}}
\end{figure}
We use the chain domain from \citet{laroche2021dr} as our testbed.
As described in Figure~\ref{fig chain},
there are $N$ non-terminal states in the chain and the agent is always initialized at state $s_1$.
There are two actions available in each state.
The \texttt{solid} action leads the agent from $s_i$ to $s_{i+1}$ and yields a reward of $0$ for all $i < N$.
At $s_N$,
the \texttt{solid} action instead leads to the terminal state and yields a reward of $1$.
The \texttt{dotted} action always leads to the terminal state directly and yields a reward $0.8 \times \gamma^{N-1}$.
Trivially,
the optimal policy is to always choose the \texttt{solid} action,
which will yield an episodic return of $\gamma^{N-1}$.
As noted by \citet{laroche2021dr},
the challenge of this chain domain is to overcome the immediate rewards pushing the agent towards suboptimal policies.
We remark that though this chain has a finite horizon,
we can indeed reformalize it into an infinite-horizon chain with transition-dependent discounting.
We refer the reader to \citet{white2017unifying} for more details about this technique and we believe our theoretical results can be easily extended to transition-dependent discounting.

We run Algorithms~\ref{alg opac} and~\ref{alg sac} in the chain domain.
According to \eqref{eq mu example},
we use the behavior policy
\begin{align}
  \mu_\theta(\cdot|s) \doteq& 0.1 \times \frac{1}{2} + 0.9 \times \frac{\exp(0.1 \times \theta_{s,\cdot})}{\exp(0.1 \times \theta_{s,\texttt{solid}}) + \exp(0.1 \times \theta_{s,\texttt{dotted}})}.
\end{align}
According to \eqref{eq optimal lr} and Assumption~\ref{assu three timescales},
we set $\qty{\alpha_t, \beta_t, \lambda_t}$ as
\begin{align}
  \alpha_t &\doteq \frac{100 + 1}{(t + 10^5)^{0.5+0.001}}, \\
  \beta_t &\doteq \frac{100}{(t+10^5)^{0.75 + 0.001}}, \\
  \lambda_t &\doteq \frac{0.025}{(t + 10^5)^{\epsilon_\lambda}},
\end{align}
where we test a range of $\epsilon_\lambda$ from $\qty{2^{-5}, 2^{-4}, 2^{-3}, 2^{-1}, 2}$\footnote{We omit $\epsilon_\lambda = 2^{-2}$ and $\epsilon_\lambda = 1$ to improve the readability of the figures. 
The corresponding curves are similar to $\epsilon_\lambda = 2^{-1}$ and $\epsilon_\lambda = 2$.}.
We run both Algorithms~\ref{alg opac} and~\ref{alg sac} for $2\times 10^6$ steps 
and evaluate the target policy every $2 \times 10^3$ steps,
where 
we execute it for 10 episodes and take the mean episodic return.
The evaluation performance is reported in Figures~\ref{fig opac} and~\ref{fig sac} respectively.
Curves are averaged over 30 independent runs with shaded regions indicating standard errors.
The black dotted lines are the performance of the optimal policy.

As suggested by Figure~\ref{fig opac} with $N \in \qty{6, 7}$,
when $\epsilon_\lambda \in \qty{2^{-1}, 2}$,
the target policy found by Algorithm~\ref{alg opac} is indeed very close to the optimal policy at the end of training,
which gives an affirmative answer to the question (i).
It is important to note that neither $\epsilon_\lambda = 2^{-1}$ nor $\epsilon_\lambda = 2$
is recommended by \eqref{eq optimal reg}.
This is expected as Assumption~\ref{assu three timescales} is only sufficient and
the convergence rate in Theorem~\ref{thm:optimality-AC} can possibly be significantly improved.
Further, with the increase of $N$,
the suboptimality of the target policy at the end of training also increases.
This is expected as increasing $N$ makes the problem more challenging.
We,
however,
remark that though with $N \in \qty{8, 9}$,
the target policy is not close to the optimal policy at the end of training,
all curves are monotonically improving as time progresses.
Similarly,
the results in Figure~\ref{fig sac} give an affirmative answer to the question (ii).
Comparing Figures~\ref{fig opac} and~\ref{fig sac},
it is easy to see that Algorithm~\ref{alg opac} is much more sensitive to $\epsilon_\lambda$ than Algorithm~\ref{alg sac}.
As shown by Figure~\ref{fig opac},
the selection of $\epsilon_\lambda$ significantly affects the rate that the suboptimality diminishes in Algorithm~\ref{alg opac}.
By contrast, 
Figure~\ref{fig sac} suggests that the rate that the suboptimality diminishes is barely affected by $\epsilon_\lambda$ in Algorithm~\ref{alg sac}. 
This comparison gives an intuitive answer the question (iii).
This difference is because the KL regularization is much more aggressive than the entropy regularization.
To be more specific,
the entropy of the policy is always bounded but the KL divergence used here can be unbounded when the policy becomes deterministic.

\section{Conclusion}
In this paper,
we demonstrate the optimality of the off-policy actor critic algorithm even without using a density ratio to correct the state distribution mismatch.
This result is significant in two aspects.
First,
it advances the understanding of the optimality of policy gradient methods in the tabular setting from \citet{agarwal2019optimality,mei2020global,laroche2021dr}.
Second,
it provides,
to certain extent,
a theoretical justification
for the practice of ignoring state distribution mismatch in large scale RL experiments \citep{wang2016sample,espeholt2018impala,vinyals2019grandmaster,schmitt2020off,zahavy2020self}.
One important ingredient of our results is the finite sample analysis of a generic stochastic approximation algorithm with time-inhomogeneous update operators on time-inhomogeneous Markov chains,
which we believe can be used to analyze more RL algorithms and has interest beyond RL.

\acks{Part of this work was done during SZ's internship at Microsoft Research Montreal and SZ's DPhil at the University of Oxford. SZ is also funded by the Engineering and Physical Sciences Research Council (EPSRC) during his DPhil.}

\appendix
\section{Proofs of Section~\ref{sec sa}}
\subsection{Proof of Lemma~\ref{lem uniform mixing}}

\uniformmixing*
\label{sec proof lem uniform mixing}
\begin{proof}
  Theorem 4.9 of \citet{levin2017markov} confirms the geometric mixing for a single ergodic chain.
  Here we adapt its proof to show the uniform mixing.

  For any $P \in \bar \Lambda_P$, 
  define the indicator matrix $\mathbb{I}_P \in \qty{0, 1}^{\ny \times \ny}$ such that 
  \begin{align}
      \mathbb{I}_P[y, y'] = \begin{cases}
          1, & P[y, y'] > 0 \\
          0, & P[y, y'] = 0
      \end{cases}.
  \end{align}
  Consider the stochastic matrix $M_P$ defined as
  \begin{align}
      M_P(i, j) = \frac{\mathbb{I}_P(i, j)}{\sum_{j'} \mathbb{I}_P(i, j')}.
  \end{align}
  Since the chain induced by $P$ is ergodic,
  it is easy to see the chain induced by $M_P$ is also ergodic.
  This is because (1)
  a finite chain is ergodic if and only if it is irreducible and aperiodic;
  (2) the connectivity of the chain induced by $M_P$ is the same as that by $P$;
  and (3) irreducibility and aperiodicity depend only on connectivity, not on the specific probability of each transition. 

  The proof of Proposition~1.7 of \citet{levin2017markov} then asserts that there exists a constant $t_P > 0$ such that for all $t \geq t_P$,
  \begin{align}
      M^t_P(i, j) > 0
  \end{align}
  holds for any $i, j$.
  Hence $P^t(i, j) > 0$ also holds because $M_P$ and $P$ share the same connectivity, so do their powers.
  Formally, it can be proved via induction that 
  \begin{align}
    \label{eq connectivity induction}
    \forall t \geq 1, i, j, \, M_P^t(i, j) > 0 \implies P^t(i, j) > 0.
  \end{align}
  First, \eqref{eq connectivity induction} obviously holds for $t=1$.
  Suppose \eqref{eq connectivity induction} holds for $t \leq k$.
  If
  \begin{align}
    M_P^{k+1}(i, j) = \sum_l M_P^k(i, l) M_P(l, j) > 0,
  \end{align}
  there must exist at least one $l$ such that $M_P^k(i, l) M_P(l, j) > 0$,
  i.e.,
  \begin{align}
    M_P^k(i, l) > 0, M_P(l, j) > 0.
  \end{align}
  Using the induction hypothesis yields
  \begin{align}
    P^k(i, l) > 0, P(l, j) > 0,
  \end{align}
  from which $P^{k+1}(i, j) > 0$ follows easily.
  This completes the induction.
  Since $\fY$ is finite,
  the set $\qty{\mathbb{I}_P | P \in \bar \Lambda_P}$ is also finite (at most $2^{\ny \times \ny}$ elements),
  and so are the sets $\qty{M_P | P \in \bar \Lambda_P}$ and $\qty{t_P | P \in \bar \Lambda_P}$.
  Let 
  \begin{align}
      t_* \doteq \max_{P \in \bar \Lambda_P}\qty{t_P},
  \end{align}
  we then have for any $P \in \bar \Lambda_P$,
  $P^{t_*} (i, j) > 0$ always holds. 
  Importantly, $t_*$ is independent of $P$.
  Then the extreme value theorem implies that 
  \begin{align}
      \delta \doteq \inf_{P \in \bar \Lambda_P, i, j} P^{t_*}(i, j) > 0.
  \end{align}
  Let $d_P$ be the invariant distribution of the chain induced by $P$
  and take any $\delta'$ such that $0 < \delta' < \delta$,
  then
  \begin{align}
      P^{t_*}(i, j) > \delta' \geq \delta' d_P(j) 
  \end{align}
  holds for any $P \in \bar \Lambda_P, i, j$.

  For any $P \in \bar \Lambda_P$,
  let $\Pi$ be a matrix, each row of which is $d_P^\top$, 
  and define
  \begin{align}
      \zeta \doteq 1 - \delta'.
  \end{align}
  We now verify that the matrix
  \begin{align}
      Q \doteq \frac{P^{t_*} + \zeta \Pi - \Pi}{\zeta}
  \end{align}
  is a stochastic matrix.
  First, its row sums are $1$:
  \begin{align}
    (Q\tb{1})(i) \doteq \frac{1 + \zeta - 1}{\zeta} = 1.
  \end{align}
  Second,
  its elements are nonnegative:
  \begin{align}
   Q(i, j) = \frac{P^{t_*}(i, j) + \zeta d_P(j) - d_P(j)}{\zeta} \geq \frac{\delta' d_P(j) + \zeta d_P(j) - d_P(j)}{\zeta} = 0.
  \end{align}
  Rearranging terms yields
  \begin{align}
      \label{eq tmp 1}
      P^{t_*} = (1 - \zeta)\Pi + \zeta Q.
  \end{align}
  We now use induction to show that for any $k \geq 1$,
  \begin{align}
    \label{eq tmp 2}
      P^{t_*k} = (1 - \zeta^k) \Pi + \zeta^k Q^k.
  \end{align}
  For $k=1$, we know \eqref{eq tmp 2} holds from \eqref{eq tmp 1}.
  Suppose \eqref{eq tmp 2} holds for $k=n$,
  then 
  \begin{align}
      P^{t_*(n+1)} &= P^{t_*n}P^{t_*} \\
      &= \left((1 - \zeta^n) \Pi + \zeta^n Q^n \right) P^{t_*} \\
      &= (1 - \zeta^n) \Pi P^{t_*} + \zeta^n Q^n \left((1 - \zeta)\Pi + \zeta Q\right) \\
      &= (1 - \zeta^n) \Pi P^{t_*} + (1 - \zeta)\zeta^n Q^n \Pi +  \zeta^{n+1} Q^{n+1} \\
      &= (1 - \zeta^n) \Pi + (1 - \zeta)\zeta^n Q^n \Pi +  \zeta^{n+1} Q^{n+1} \\
      \intertext{\hfill (Property of invariant distribution)}
      &= (1 - \zeta^n) \Pi + (1 - \zeta)\zeta^n \Pi +  \zeta^{n+1} Q^{n+1} \\
      \intertext{\hfill ($Q\Pi = Q1 d_P^\top = 1 d_P^\top = \Pi$ for any stochastic matrix $Q$)}
      &= (1 - \zeta^{n+1}) \Pi + \zeta^{n+1} Q^{n+1},
  \end{align}
  which completes the induction.
  Consequently,
  for any $l \in \qty{0, 1, \dots, t_* - 1}$,
  multiplying by $P^l$ both sides of \eqref{eq tmp 2} yields
  \begin{align}
      P^{t_*k + l} &= (1 - \zeta^k) \Pi P^l + \zeta^k Q^kP^l \\
      &=(1 - \zeta^k) \Pi + \zeta^k Q^kP^l.
  \end{align}
  Rearranging terms yields
  \begin{align}
      P^{t_*k + l} - \Pi = \zeta^k(Q^kP^l - \Pi),
  \end{align}
  implying for any $i$,
  \begin{align}
      \sum_{j} \abs{P^{t_*k+l}(i, j) - d_P(j)} &= \zeta^k \sum_j \abs{(Q^kP^l)(i, j) - d_P(j)} \\
      &\leq 2 \zeta^k 
      \qq{\hfill (Boundedness of total variation)} \\
      &= 2 \zeta^{-\frac{l}{t_*}} \left(\zeta^{\frac{1}{t_*}}\right)^{t_* k + l} \\
      &\leq 2 \zeta^{-\frac{t_*-1}{t_*}}\left(\zeta^{\frac{1}{t_*}}\right)^{t_* k + l}.
  \end{align}
  Let 
  \begin{align}
    C_0' \doteq 2 \zeta^{-\frac{t_*-1}{t_*}}, \tau \doteq \zeta^{\frac{1}{t_*}}.
  \end{align}
  It is easy to see $C_0' > 0, \tau \in (0, 1)$ and both $C_0$ and $\tau$ are independent of $P$.
  Consequently,
  for any $n \geq t_*$,
  we have
  \begin{align}
    \sum_j \abs{P^n(i, j) - d_P(j)} \leq C_0' \tau^n.
  \end{align}
  By the boundedness of total variation,
  for $n \in \qty{0, 1, \dots, t_* - 1}$,
  we have
  \begin{align}
    \sum_j \abs{P^n(i, j) - d_P(j)} \leq 2 \leq \frac{2}{\tau^{t_*}} \tau^n.
  \end{align}
  Setting $C_0 \doteq \max\qty{C_0', \frac{2}{\tau^{t_*}}}$ completes the proof.
\end{proof}

\subsection{Proof of Theorem \ref{thm sa convergence}}
\label{sec proof thm sa convergence}

\thmsaconvergence*
\begin{proof}
Since the theorem is a generalization of the results in \citet{chen2021lyapunov},
we follow their framework to complete the proof.
In our setting,
the dynamics of the Markov chain changes every time step according to a secondary random sequence $\qty{\theta_t}$.
Consequently,
we have many new error terms which are not controlled by \citet{chen2021lyapunov} 
and that we handle using techniques from \citet{zou2019finite}.

Following \citet{chen2021lyapunov},
we use a Lyapunov method for the proof
with the generalized Moreau envelope of $\frac{1}{2}\norm{\cdot}_c^2$ as the Lyapunov function.
In particular,
we consider the Lyapunov function
\begin{align}
  M(w) \doteq \inf_{u\in\R^K}\qty{\frac{1}{2} \norm{u}_c^2 + \frac{1}{2 \xi} \norm{w - u}_s^2},
\end{align}
where
$\xi > 0$ is a constant to be tuned,
$\norm{\cdot}_c$ is the norm w.r.t. which $\bar F_\theta$ is contractive (cf. Assumption~\ref{assu uniform contraction}),
and $\norm{\cdot}_s$ is an arbitrary norm such that $\frac{1}{2}\norm{\cdot}_s^2$ is $L$-smooth (Lemma~\ref{lem smooth definition}).
It can, e.g., be an $\ell_p$ norm with $p \geq 2$ (Example 5.11 of \citet{beck2017first}).
Due to the equivalence between norms,
there exist positive constants $l_{cs}$ and $u_{cs}$ such that 
\begin{align}
  l_{cs} \norm{w}_s \leq \norm{w}_c \leq u_{cs} \norm{w}_s
\end{align}
holds for any $w$.
The following lemma proved by \citet{chen2021lyapunov} describes some properties of $M$.
\begin{lemma}
    (Proposition A.1 of \citet{chen2021lyapunov})
    \label{lem property of M}
    \begin{enumerate}[(i).]
        \item $M(w)$ is convex, and $\frac{L}{\xi}$-smooth w.r.t. $\norm{\cdot}_s$.
        \item There exists a norm $\norm{\cdot}_m$ such that $M(w) = \frac{1}{2}\norm{w}_m^2$.
        \item Define
        \begin{align}
            &l_{cm} = \sqrt{(1 + \xi l_{cs}^2)} \\
            &u_{cm} = \sqrt{(1 + \xi u_{cs}^2)},
        \end{align}
        then $\forall w$,
        \begin{align}
            l_{cm}\norm{w}_m \leq \norm{w}_c \leq u_{cm} \norm{w}_m.
        \end{align}
    \end{enumerate}
\end{lemma}
Lemma~\ref{lem property of M} (i) and Lemma~\ref{lem smooth definition} imply that for any $x, x'$,
\begin{align}
    M(x') \leq M(x) + \indot{\nabla M(x)}{x' - x} + \frac{L}{2\xi} \norm{x - x'}_s^2.
\end{align}
Using $x' = w_{t+1} - w^*_{\theta_{t+1}}$ and $x = w_t - w^*_{\theta_t}$ in the above inequality and the update equation~\eqref{eq sa iterates}:
\begin{align}
    w_{t+1} \doteq w_t + \alpha_t (F_{\theta_t}(w_t, Y_t) - w_t + \epsilon_t),
\end{align} Lemma~\ref{lem property of M} (ii) yields
\begin{align}
  \label{eq expansion of M}
  &\frac{1}{2} \norm{w_{t+1} - w^*_{\theta_{t+1}}}_m^2 \\
  \leq& \frac{1}{2}\norm{w_t - w^*_{\theta_t}}_m^2 + \indot{\nabla M(w_t - w^*_{\theta_t})}{w_{t+1} - w_t + w^*_{\theta_t} - w^*_{\theta_{t+1}}} \\
  &+ \frac{L}{2\xi} \norm{w_{t+1} - w_t + w^*_{\theta_t} - w^*_{\theta_{t+1}}}^2_s \\
  =&\frac{1}{2}\norm{w_t - w^*_{\theta_t}}_m^2 \\
  &+ \underbrace{\indot{\nabla M(w_t - w^*_{\theta_t})}{w^*_{\theta_t} - w^*_{\theta_{t+1}}}}_{T_1} \\
  &+ \alpha_t \underbrace{\indot{\nabla M(w_t - w^*_{\theta_t})}{\bar F_{\theta_t}(w_t) - w_t}}_{T_2} \\
  &+ \alpha_t \underbrace{\indot{\nabla M(w_t - w^*_{\theta_t})}{F_{\theta_t}(w_t, Y_t) - \bar F_{\theta_t}(w_t)} }_{T_3} \\
  &+ \alpha_t \underbrace{\indot{\nabla M(w_t - w^*_{\theta_t})}{\epsilon_t} }_{T_4} \\
  &+ \alpha_t^2\underbrace{\frac{L }{\xi}\norm{F_{\theta_t}(w_t, Y_t) - w_t + \epsilon_t}_s^2}_{T_5} \\
  &+ \underbrace{\frac{L}{\xi} \norm{w^*_{\theta_t} - w^*_{\theta_{t+1}}}_s^2}_{T_6}.
\end{align}
We now bound $T_1$ - $T_6$ one by one. 
$T_1$ and $T_6$ are errors resulting from changing dynamics and are not controlled in \citet{chen2021lyapunov}.
$T_2$, $T_4$, and $T_5$ can be bounded similarly to \citet{chen2021lyapunov}.
To bound $T_3$,
we further decompose it as
\begin{align}
    T_3 =& \indot{\nabla M(w_t - w^*_{\theta_t})}{F_{\theta_t}(w_t, Y_t) - \bar F_{\theta_t}(w_t)} \\
    =&\underbrace{\indot{\nabla M(w_t - w^*_{\theta_t}) - \nabla M(w_{t-\tau_{\alpha_t}} - w^*_{\theta_{t-\tau_{\alpha_t}}})}{F_{\theta_t}(w_t, Y_t) - \bar F_{\theta_t}(w_t)}}_{T_{31}} \\
    &+\underbrace{\indot{\nabla M(w_{t-\tau_{\alpha_t}} - w^*_{\theta_{t-\tau_{\alpha_t}}})}{F_{\theta_t}(w_t, Y_t) - F_{\theta_t}(w_{t- \tau_{\alpha_t}}, Y_t) + \bar F_{\theta_t}(w_{t- \tau_{\alpha_t}}) - \bar F_{\theta_t}(w_t)}}_{T_{32}}\\
    &+\underbrace{\indot{\nabla M(w_{t-\tau_{\alpha_t}} - w^*_{\theta_{t-\tau_{\alpha_t}}})}{F_{\theta_t}(w_{t- \tau_{\alpha_t}}, Y_t) - \bar F_{\theta_t}(w_{t- \tau_{\alpha_t}})}}_{T_{33}},
\end{align}
where
\begin{align}
  \label{eq definition of tau alpha t}
    \tau_{\alpha_t} \doteq \min\qty{n \geq 0 \mid C_0 \tau^n \leq \alpha_t},
\end{align}
and $C_0$ and $\tau$ are defined in Lemma~\ref{lem uniform mixing}.
$\tau_{\alpha_t}$ denotes the number of steps the chain needs to mix to an accuracy of $\alpha_t$.
$T_{31}$ and $T_{32}$ can be bounded similarly to \citet{chen2021lyapunov}.
The bound for $T_{33}$ is however significantly different.
We decompose $T_{33}$ as
\begin{align}
  T_{33} =& \indot{\nabla M(w_{t-\tau_{\alpha_t}} - w^*_{\theta_{t-\tau_{\alpha_t}}})}{F_{\theta_t}(w_{t- \tau_{\alpha_t}}, Y_t) - \bar F_{\theta_t}(w_{t- \tau_{\alpha_t}})} \\
  =& \underbrace{\indot{\nabla M(w_{t-\tau_{\alpha_t}} - w^*_{\theta_{t-\tau_{\alpha_t}}})}{F_{\theta_{t - \tau_{\alpha_t}}}(w_{t- \tau_{\alpha_t}}, \tilde Y_t) - \bar F_{\theta_{t-\tau_{\alpha_t}}}(w_{t- \tau_{\alpha_t}})}}_{T_{331}} + \\
  & \underbrace{\indot{\nabla M(w_{t-\tau_{\alpha_t}} - w^*_{\theta_{t-\tau_{\alpha_t}}})}{F_{\theta_{t - \tau_{\alpha_t}}}(w_{t- \tau_{\alpha_t}}, Y_t) -F_{\theta_{t - \tau_{\alpha_t}}}(w_{t- \tau_{\alpha_t}}, \tilde Y_t)}}_{T_{332}} + \\
  & \underbrace{\indot{\nabla M(w_{t-\tau_{\alpha_t}} - w^*_{\theta_{t-\tau_{\alpha_t}}})}{F_{\theta_{t}}(w_{t- \tau_{\alpha_t}}, Y_t) -F_{\theta_{t - \tau_{\alpha_t}}}(w_{t- \tau_{\alpha_t}}, Y_t)}}_{T_{333}} + \\
  & \underbrace{\indot{\nabla M(w_{t-\tau_{\alpha_t}} - w^*_{\theta_{t-\tau_{\alpha_t}}})}{\bar F_{\theta_{t-\tau_{\alpha_t}}}(w_{t- \tau_{\alpha_t}}) - \bar F_{\theta_{t}}(w_{t- \tau_{\alpha_t}})}}_{T_{334}}.
\end{align}
Here $\qty{\tilde Y_t}$ is an auxiliary chain inspired from \citet{zou2019finite}.
Before time $t - \tau_{\alpha_t} - 1$,
$\qty{\tilde Y_t}$ is exactly the same as $\qty{Y_t}$.
After time $t - \tau_{\alpha_t} - 1$,
$\tilde Y_t$ evolves according to the \emph{fixed} kernel $P_{\theta_{t-\tau_{\alpha_t}}}$
while $Y_t$ evolves according the changing kernel $P_{\theta_{t - \tau_{\alpha_t}}}, P_{\theta_{k - \tau_{\alpha_t} + 1}}, \dots$.
\begin{align}
  \qty{\tilde Y_t}&: \dots \to Y_{t-\tau_{\alpha_t}-1} \underbrace{\to}_{P_{\theta_{t-\tau_{\alpha_t}}}} Y_{t-\tau_{\alpha_t}} \underbrace{\to}_{P_{\theta_{t-\tau_{\alpha_t}}}} \tilde Y_{t-\tau_{\alpha_t}+1} \underbrace{\to}_{P_{\theta_{t-\tau_{\alpha_t}}}} \tilde Y_{t-\tau_{\alpha_t}+2} \to \dots \\
  \qty{Y_t}&: \dots \to Y_{t-\tau_{\alpha_t}-1} \underbrace{\to}_{P_{\theta_{t-\tau_{\alpha_t}}}} Y_{t-\tau_{\alpha_t}} \underbrace{\to}_{P_{\theta_{t-\tau_{\alpha_t}+1}}} Y_{t-\tau_{\alpha_t}+1} \underbrace{\to}_{P_{\theta_{t-\tau_{\alpha_t}+2}}} Y_{t-\tau_{\alpha_t}+2} \to \dots.
\end{align}

We are now ready to present bounds for each of the above terms.
To begin,
we define some shorthand and study their properties:
\begin{align}
  \alpha_{t_1, t_2} &\doteq \sum_{t=t_1}^{t_2} \alpha_t, \quad \beta_{t_1, t_2} \doteq \sum_{t=t_1}^{t_2} \beta_t  \\
  \label{eq shorthand a and b}
  A &\doteq U_\epsilon + L_F + 1, \quad B \doteq U_F + U_\epsilon', \quad C \doteq AU_w + B + A + A(1 + U_F' + U_F'').
\end{align}
\begin{restatable}{lemma}{lemlearningrates}
  \label{lem learning rates}
  For sufficiently large $t_0$,
  \begin{align}
    \tau_{\alpha_t} &= \fO(\log (t+t_0)), \quad \alpha_{t-\tau_{\alpha_t}, t-1} = \fO\left(\frac{\log (t+t_0)}{(t+t_0)^{\epsilon_\alpha}}\right), \\
    \beta_{t-\tau_{\alpha_t}, t-1} &= \fO\left(\frac{\log (t+t_0)}{(t+t_0)^{\epsilon_\alpha}}\right), \quad \frac{\alpha_t \alpha_{t-\tau_{\alpha_t}, t-1}}{\beta_t} = \fO\left(\frac{\log (t+t_0)}{(t+t_0)^{2\epsilon_\alpha - \epsilon_\beta}}\right).
  \end{align}
\end{restatable}
\noindent
The proof of Lemma~\ref{lem learning rates} is provided in Section~\ref{sec proof lem learning rates}.
Lemma~\ref{lem learning rates} asserts that we can select a $t_0$ sufficiently large such that
\begin{align}
  \alpha_{t-\tau_{\alpha_t}, t-1} \leq \frac{1}{4A}
\end{align}
holds for all $t$.
This condition is crucial for Lemma~\ref{lem bound of xk diff},
which plays an important role in the following bounds.

\begin{restatable}{lemma}{lemboundtone}
  \label{lem bound t1}
  (Bound of $T_1$)
  \begin{align}
      T_1 
      \leq \frac{L_w L_\theta \beta_t}{l_{cm}} \norm{w_t - w^*_{\theta_t}}_m.
  \end{align}
\end{restatable}
\noindent
The proof of Lemma~\ref{lem bound t1} is provided in Section~\ref{sec proof lem bound t1}.

\begin{restatable}{lemma}{lemboundttwo}
  \label{lem bound t2}
(Bound of $T_2$)
\begin{align}
  T_2 \leq -(1 - \kappa\frac{u_{cm}}{l_{cm}}) \norm{w_t - w^*_{\theta_t}}_m^2.
\end{align}
\end{restatable}
\noindent
The proof of Lemma~\ref{lem bound t2} is provided in Section~\ref{sec proof lem bound t2}.

\begin{restatable}{lemma}{lemboundoftthreeone}
  \label{lem bound of t31}
  (Bound of $T_{31}$)
  \begin{align}
      T_{31} \leq \frac{8L(L_wL_\theta + 1)  \alpha_{t-\tau_{\alpha_t}, t-1}}{\xi l_{cs}^2} \left(u_{cm}^2A^2\norm{w_t - w^*_{\theta_t}}_m^2 + C^2\right).
  \end{align}
\end{restatable}
\noindent
The proof of Lemma~\ref{lem bound of t31} is provided in Section~\ref{sec proof lem bound t31}.

\begin{restatable}{lemma}{lemboundtthreetwo}
  \label{lem bound t32}
  (Bound of $T_{32}$)
  \begin{align}
      T_{32} \leq \frac{32 L  \alpha_{t-\tau_{\alpha_t}, t-1}(1 + L_wL_\theta \beta_{t-\tau_{\alpha_t}, t-1})}{\xi l_{cs}^2} \left( u_{cm}^2A^2\norm{w_t - w^*_{\theta_t}}_m^2 + C^2 \right).
  \end{align}
\end{restatable}
\noindent
The proof of Lemma~\ref{lem bound t32} is provided in Section~\ref{sec proof lem bound t32}.

\begin{restatable}{lemma}{lemboundtthreethreeone}
  \label{lem bound t331}
  (Bound of $T_{331}$)
\begin{align}
  \E\left[T_{331}\right]
  \leq & \frac{8L \alpha_t(1 + L_wL_\theta \beta_{t-\tau_{\alpha_t}, t-1})}{A\xi l_{cs}^2}\left(u_{cm}^2 A^2 \E \left[\norm{w_t - w^*_{\theta_t}}_m^2\right] + C^2\right).
\end{align}
\end{restatable}
\noindent
The proof of Lemma~\ref{lem bound t331} is provided in Section~\ref{sec proof lem bound t331}.

\begin{restatable}{lemma}{lemboundtthreethreetwo}
  \label{lem bound t332}
  (Bound of $T_{332}$)
  \begin{align}
      \E\left[T_{332}\right]
      \leq \frac{8 \ny L_P L_\theta \sum_{j=t-\tau_{\alpha_t}}^{t-1}\beta_{t-\tau_{\alpha_t}, j} L(1 + L_wL_\theta \beta_{t-\tau_{\alpha_t}, t-1})}{A \xi l_{cs}^2} \left(u_{cm}^2 A^2 \E\left[ \norm{w_t - w^*_{\theta_t}}_m^2 \right] + C^2\right).
  \end{align}
\end{restatable}
\noindent
The proof of Lemma~\ref{lem bound t332} is provided in Section~\ref{sec proof lem bound t332}.

\begin{restatable}{lemma}{lemboundoftthreethreethree}
  \label{lem bound of t333}
  (Bound of $T_{333}$)
  \begin{align}
  T_{333} \leq \frac{8LL_F' L_\theta \beta_{t-\tau_{\alpha_t}, t-1} (1 + L_wL_\theta \beta_{t-\tau_{\alpha_t}, t-1})}{A^2 \xi l_{cs}^2}\left(u_{cm}^2 A^2 \norm{w_t - w^*_{\theta_t}}_m^2 + C^2\right).
  \end{align}
\end{restatable}
\noindent
The proof of Lemma~\ref{lem bound of t333} is provided in Section~\ref{sec proof lem bound t333}.

\begin{restatable}{lemma}{lemboundtthreethreefour}
  \label{lem bound t334}
  (Bound of $T_{334}$)
  \begin{align}
  T_{334} \leq \frac{8LL_F'' L_\theta \beta_{t-\tau_{\alpha_t}, t-1}(1 + L_wL_\theta \beta_{t-\tau_{\alpha_t}, t-1})}{A^2 \xi l_{cs}^2}\left(u_{cm}^2 A^2 \norm{w_t - w^*_{\theta_t}}_m^2 + C^2\right).
  \end{align}
\end{restatable}
\noindent
The proof of Lemma~\ref{lem bound t334} is provided in Section~\ref{sec proof lem bound t334}.

\begin{restatable}{lemma}{lemboundtfour}
  \label{lem bound t4}
  (Bound of $T_4$)
  \begin{align}
    \E\left[T_4\right] = 0.
  \end{align}
\end{restatable}
\noindent
The proof of Lemma~\ref{lem bound t4} is provided in Section~\ref{sec proof lem bound t4}.

\begin{restatable}{lemma}{lemboundtfive}
  \label{lem bound t5}
  (Bound of $T_5$)
  \begin{align}
      T_5 
      \leq \frac{2L}{\xi l_{cs}^2}\left(A^2 u_{cm}^2 \norm{w_t - w^*_{\theta_t}}_m^2 + C^2\right).
  \end{align}
\end{restatable}
\noindent
The proof of Lemma~\ref{lem bound t5} is provided in Section~\ref{sec proof lem bound t5}.

\begin{restatable}{lemma}{lemboundtsix}
  \label{lem bound t6}
  (Bound of $T_6$)
  \begin{align}
      T_6 = \frac{L}{\xi} \norm{w^*_{\theta_t} - w^*_{\theta_{t+1}}}_s^2 \leq \frac{L L_w^2 L_\theta^2 \beta_t^2}{\xi l_{cs}^2}.
  \end{align}
\end{restatable}
\noindent
Lemma~\ref{lem bound t6} follows immediately from Assumptions~\ref{assu regularization} and \ref{assu twotimescale}.

We now assemble the bounds in Lemmas~\ref{lem bound t1} - \ref{lem bound t6} back into \eqref{eq expansion of M}.
By the definition of $u_{cm}$ and $l_{cm}$ in Lemma~\ref{lem property of M},
we have
\begin{align}
  \lim_{\xi \to 0} \frac{u_{cm}}{l_{cm}} = 1.
\end{align}
Since $\kappa < 1$,
we can select a sufficiently small $\xi > 0$ such that
\begin{align}
  \psi_1 \doteq \frac{2}{9} (1 - \kappa \frac{u_{cm}}{l_{cm}})
\end{align}
satisfies $\psi_1 \in (0, 1)$,
implying 
\begin{align}
  T_2 &\leq -\frac{9}{2} \psi_1 \norm{w_t - w^*_{\theta_t}}_m^2, \\
  \label{eq simple t2}
  \alpha_t T_2 &\leq -\frac{9}{2} \alpha_t \psi_1 \norm{w_t - w^*_{\theta_t}}_m^2.
\end{align}
Let $\psi_2$ be a positive constant to be tuned.
For $T_1$, suppose $\psi_2$ is large enough,
then we have
\begin{align}
  \label{eq simple t1}
  T_1 \leq \frac{1}{2} \beta_t \psi_2 \norm{w_t - w^*_{\theta_t}}_m.
\end{align}
For $T_{31}$,
Lemmas~\ref{lem learning rates} and~\ref{lem bound of t31} assert that
we can select sufficiently large $t_0$ and $\psi_2$ such that
\begin{align}
  T_{31} &\leq \frac{1}{2} \psi_1 \norm{w_t - w^*_{\theta_t}}_m^2 + \frac{1}{2}\alpha_{t-\tau_{\alpha_t}, t-1} \psi_2, \\
  \label{eq simple t31}
  \alpha_t T_{31} &\leq \frac{1}{2}\alpha_t \psi_1 \norm{w_t - w^*_{\theta_t}}_m^2 + \frac{1}{2}\alpha_t \alpha_{t-\tau_{\alpha_t}, t-1} \psi_2.
\end{align}
For $T_{32}$,
Lemma~\ref{lem learning rates} implies that for $t_0$ large enough
\begin{align}
  \beta_{t-\tau_{\alpha_t}, t-1} \leq 1.
\end{align}
Hence, Lemma~\ref{lem bound t32} guarantees that we can select sufficiently large $t_0$ and $\psi_2$
such that
\begin{align}
  T_{32} &\leq \frac{1}{2} \psi_1 \norm{w_t - w^*_{\theta_t}}_m^2 + \frac{1}{2} \alpha_{t-\tau_{\alpha_t}, t-1} \psi_2, \\
  \label{eq simple t32}
  \alpha_t T_{32} &\leq \frac{1}{2}\alpha_t \psi_1 \norm{w_t - w^*_{\theta_t}}_m^2 + \frac{1}{2}\alpha_t \alpha_{t-\tau_{\alpha_t}, t-1} \psi_2.
\end{align}
For $T_{331}$,
similarly,
we can select sufficiently large $t_0$
and $\psi_2$ such that
\begin{align}
  \E\left[T_{331}\right] &\leq \frac{1}{2} \psi_1  \E \left[\norm{w_t - w^*_{\theta_t}}_m^2\right] + \frac{1}{2} \alpha_t \psi_2 \\
  &\leq \frac{1}{2} \psi_1  \E \left[\norm{w_t - w^*_{\theta_t}}_m^2\right] + \frac{1}{2} \alpha_{t-\tau_{\alpha_t}, t-1} \psi_2, \\
  \label{eq simple t331}
  \alpha_t \E\left[T_{331}\right]&\leq \frac{1}{2}\alpha_t \psi_1  \E \left[\norm{w_t - w^*_{\theta_t}}_m^2\right] + \frac{1}{2}\alpha_t \alpha_{t-\tau_{\alpha_t}, t-1} \psi_2.
\end{align}
For $T_{332}$,
we have
\begin{align}
  \frac{\sum_{j=t-\tau_{\alpha_t}}^{t-1}\beta_{t-\tau_{\alpha_t}, j}}{\alpha_{t-\tau_{\alpha_t}, t-1}} &\leq \frac{\tau_{\alpha_t} \tau_{\alpha_t} \beta_{t-\tau_{\alpha_t}}}{\tau_{\alpha_t} \alpha_t} = \frac{\tau_{\alpha_t} \beta_{t-\tau_{\alpha_t}}}{\alpha_t} =\fO\left( \frac{\log (t+t_0)\beta_{t-\tau_{\alpha_t}}}{\alpha_t} \right) \\
  &= \fO\left(\frac{\log (t+t_0) \beta_t}{\alpha_t}\right) \qq{(for $t_0$ sufficiently large).}
\end{align}
Since the RHS of the above inequality approaches $0$ when $t_0$ is sufficiently large,
we can select sufficiently large $t_0$ such that
\begin{align}
  \sum_{j=t-\tau_{\alpha_t}}^{t-1}\beta_{t-\tau_{\alpha_t}, j} \leq \alpha_{t-\tau_{\alpha_t}, t-1}.
\end{align}
Then it is easy to see for sufficiently large $t_0$ and $\psi_2$,
\begin{align}
  \E\left[T_{332}\right]&\leq \frac{1}{2} \psi_1  \E \left[\norm{w_t - w^*_{\theta_t}}_m^2\right] + \frac{1}{2} \alpha_{t-\tau_{\alpha_t}, t-1} \psi_2, \\
  \label{eq simple t332}
  \alpha_t\E\left[T_{332}\right]&\leq \frac{1}{2}\alpha_t \psi_1  \E \left[\norm{w_t - w^*_{\theta_t}}_m^2\right] + \frac{1}{2}\alpha_t \alpha_{t-\tau_{\alpha_t}, t-1} \psi_2.
\end{align}
Similarly, for sufficiently large $t_0$ and $\psi_2$,
\begin{align}
  \label{eq simple t333}
  \alpha_t T_{333}&\leq \frac{1}{2}\alpha_t \psi_1 \norm{w_t - w^*_{\theta_t}}_m^2 + \frac{1}{2}\alpha_t \alpha_{t-\tau_{\alpha_t}, t-1} \psi_2, \\
  \label{eq simple t334}
  \alpha_t T_{334}&\leq \frac{1}{2}\alpha_t \psi_1 \norm{w_t - w^*_{\theta_t}}_m^2 + \frac{1}{2}\alpha_t \alpha_{t-\tau_{\alpha_t}, t-1} \psi_2.
\end{align}
For $T_5$,
it is easy to see for sufficiently large $t_0$ and $\psi_2$,
\begin{align}
  \label{eq simple t5}
  \alpha_t^2 T_5 &\leq \frac{1}{2}\alpha_t \psi_1 \norm{w_t - w^*_{\theta_t}}_m^2 + \frac{1}{2}\alpha_t^2 \psi_2 \\
  &\leq \frac{1}{2}\alpha_t \psi_1 \norm{w_t - w^*_{\theta_t}}_m^2 + \frac{1}{2}\alpha_t \alpha_{t-\tau_{\alpha_t}, t-1} \psi_2.
\end{align}
For $T_6$,
since $\beta_t < \alpha_t$,
we can similarly select sufficiently large $t_0$ and $\psi_2$ such that
\begin{align}
  \label{eq simple t6}
  T_6 &\leq \frac{1}{2}\alpha_t\alpha_{t-\tau_{\alpha_t}, t-1} \psi_2.
\end{align}
Putting \eqref{eq simple t2}, \eqref{eq simple t1}, \eqref{eq simple t31}, \eqref{eq simple t32}, \eqref{eq simple t331}, \eqref{eq simple t332},  \eqref{eq simple t333}, \eqref{eq simple t334}, \eqref{eq simple t5}, and \eqref{eq simple t6} back to \eqref{eq expansion of M} yields
\begin{align}
  \label{eq sa bound recursive}
  &\E\left[\norm{w_{t+1} - w^*_{\theta_{t+1}}}_m^2\right] \\ 
  \leq& (1 - \psi_1\alpha_t) \E\left[\norm{w_{t} - w^*_{\theta_{t}}}_m^2\right] + \beta_t\psi_2\E\left[\norm{w_{t} - w^*_{\theta_{t}}}_m\right] + 8 \alpha_t \alpha_{t-\tau_{\alpha_t}, t-1} \psi_2 \\
  \leq& (1 - \psi_1\alpha_t) \E\left[\norm{w_{t} - w^*_{\theta_{t}}}_m^2\right] + \beta_t\psi_2 \sqrt{\E\left[\norm{w_{t} - w^*_{\theta_{t}}}_m^2\right]} + 8 \alpha_t \alpha_{t-\tau_{\alpha_t}, t-1} \psi_2
  \intertext{\hfill (Jensen's inequality).}
\end{align}
\eqref{eq sa bound recursive} applies only for $t$ such that $t - \tau_{\alpha_t} \geq 0$.
According to Lemma~\ref{lem learning rates},
we can select a sufficiently large $t_0$ such that
for all $t \geq t_0$, 
we have $t - \tau_{\alpha_t} \geq 0$.
We now bound $\E\left[\norm{w_{t} - w^*_{\theta_{t}}}_m^2\right]$ for both $t \leq t_0$ and $t \geq t_0$.
\begin{restatable}{lemma}{lemsaerrorboundone}
  \label{lem sa error bound 1}
  There exists a constant $C_{t_0, w_0}$ such that 
  for all $t \leq t_0$,
  \begin{align}
    \E\left[\norm{w_{t} - w^*_{\theta_{t}}}_m^2\right] \leq C_{{t_0}, w_0}.
  \end{align}
\end{restatable}
\noindent
The proof of Lemma~\ref{lem sa error bound 1} is provided in Section~\ref{sec proof lem sa error bound 1}.
We now proceed to the case of $t \geq t_0$.
When $t_0$ is sufficiently large,
Lemma~\ref{lem learning rates} asserts that there exists a constant $\psi_3$ such that
\begin{align}
  8 \alpha_t \alpha_{t-\tau_{\alpha_t}, t-1} \psi_2 \leq \psi_3 \frac{\log (t+t_0)}{(t+t_0)^{2\epsilon_\alpha}}.
\end{align}
Then using 
\begin{align}
  z_t \doteq \sqrt{\E\left[\norm{w_{t} - w^*_{\theta_{t}}}_m^2\right]}
\end{align}
as a shorthand,
we get from \eqref{eq sa bound recursive} that
\begin{align}
  z_{t+1}^2
  \leq& (1-\frac{\alpha\psi_1}{(t+t_0)^{\epsilon_\alpha}})z_t^2 + \frac{\beta \psi_2 {z_t}}{(t+t_0)^{\epsilon_\beta}} +  \psi_3 \frac{\log (t+t_0)}{(t+t_0)^{2\epsilon_\alpha}}.
\end{align}
We now use an induction to show that $\forall t\geq t_0$,
\begin{align}
  \label{eq final sa rate}
  z_t \leq \frac{C_0}{(t+t_0)^\epsilon},
\end{align}
where $C_0 > 1$ and $\epsilon \in (0, 1)$ are constants to be tuned.
Since Lemma~\ref{lem sa error bound 1} asserts that $z_{t_0} \leq C_{t_0, w_0}$,
we can select 
\begin{align}
  C_0 \geq C_{t_0, w_0} (2t_0)^\epsilon
\end{align}
such that 
\eqref{eq final sa rate} holds for $t=t_0$.
Now assume that \eqref{eq final sa rate} holds for $t=n$,
then for $t=n+1$,
we have
\begin{align}
  &z_{n+1}^2 \\
  \leq& (1-\frac{\alpha\psi_1}{(n+t_0)^{\epsilon_\alpha}})z_n^2 + \frac{\beta \psi_2 {z_n}}{(n+t_0)^{\epsilon_\beta}} +  \psi_3 \frac{\log (n+t_0)}{(n+t_0)^{2\epsilon_\alpha}} \\
  \stackrel{(i)}{\leq}& (1-\frac{\alpha\psi_1}{(n+t_0)^{\epsilon_\alpha}})\frac{C_0^2}{(n+t_0)^{2\epsilon}} + \frac{\beta \psi_2 C_0}{(n+t_0)^{\epsilon_\beta + \epsilon}} +  \psi_3 \frac{\log (n+t_0)}{(n+t_0)^{2\epsilon_\alpha}} \\
  =& \frac{C_0^2}{(n+t_0)^{2\epsilon}} - \frac{\alpha\psi_1C_0^2}{(n+t_0)^{\epsilon_\alpha+2\epsilon}} + \frac{\beta \psi_2 C_0}{(n+t_0)^{\epsilon_\beta + \epsilon}} +  \psi_3 \frac{\log (n+t_0)}{(n+t_0)^{2\epsilon_\alpha}} \\
  \stackrel{(ii)}{\leq}& \frac{C_0^2}{(n+1+t_0)^{2\epsilon}} + \frac{2C_0^2}{(n+t_0)^{2\epsilon + 1}} - \frac{\alpha\psi_1C_0^2}{(n+t_0)^{\epsilon_\alpha+2\epsilon}} + \frac{\beta \psi_2 C_0}{(n+t_0)^{\epsilon_\beta + \epsilon}} +  \psi_3 \frac{\log (n+t_0)}{(n+t_0)^{2\epsilon_\alpha}} \\
  \stackrel{}{\leq}& \frac{C_0^2}{(n+1+t_0)^{2\epsilon}} + \underbrace{\left(\frac{2}{(n+t_0)^{2\epsilon + 1}} - \frac{\alpha\psi_1}{(n+t_0)^{\epsilon_\alpha+2\epsilon}} + \frac{\beta \psi_2 }{(n+t_0)^{\epsilon_\beta + \epsilon}} +  \psi_3 \frac{\log (n+t_0)}{(n+t_0)^{2\epsilon_\alpha}}\right)}_{z_n'}C_0^2.
\end{align}
Here (i) results from the inductive hypothesis
and (ii) results from the fact that
\begin{align}
  x^{-2\epsilon} \leq (x+1)^{-2\epsilon} + \frac{2}{x^{2\epsilon+1}}.
\end{align}
To see the above inequality,
consider
\begin{align}
  f(x) = x^{-2\epsilon},
\end{align}
which is convex on $(0, +\infty)$,
implying
\begin{align}
  f(x) - f(x+1) \leq f'(x) \left(x - (x+1)\right).
\end{align}
To complete the induction,
it is sufficient to ensure that $\forall n$,
\begin{align}
  z_n' \leq 0.
\end{align}
One way to achieve this is to select $\epsilon$ such that
\begin{align}
  \begin{cases}
    \epsilon_\alpha + 2\epsilon < 2\epsilon + 1 \\
    \epsilon_\alpha + 2\epsilon < \epsilon_\beta + \epsilon \\
    \epsilon_\alpha + 2\epsilon < 2\epsilon_\alpha \\
  \end{cases}
  \iff
  \begin{cases}
    \epsilon_\alpha < 1 \\
    \epsilon < \epsilon_\beta - \epsilon_\alpha \\
    \epsilon < \frac{\epsilon_\alpha}{2}
  \end{cases}
\end{align}
and pick $t_0$ sufficiently large (depending on the chosen $\epsilon$).

With the induction completed,
\eqref{eq final sa rate} implies that $\forall t\geq t_0$,
\begin{align}
  \label{eq tmp 4}
  \E\left[\norm{w_{t} - w^*_{\theta_{t}}}_m^2\right] \leq \frac{C_0^2}{(t+t_0)^{2\epsilon}}.
\end{align}
Combining \eqref{eq tmp 4} and Lemma~\ref{lem sa error bound 1},
we conclude that
for any
\begin{align}
  \epsilon_w \in (0, \min\qty{2(\epsilon_\beta - \epsilon_\alpha), \epsilon_\alpha}),
\end{align}
if $t_0$ is sufficiently large,
then $\forall t$,
\begin{align}
  \E\left[\norm{w_{t} - w^*_{\theta_{t}}}_c^2\right] =\fO\left(\frac{1}{(t+t_0)^{\epsilon_w}}\right),
\end{align}
which completes the proof.
\end{proof}

\section{Proofs of Section~\ref{sec opac}}
\subsection{Proof of Lemma~\ref{lem uniform contraction opac}}

\lemuniformcontractionopac*
\label{sec proof lem uniform contraction opac}
\begin{proof}
  Assumption~\ref{assu mu uniform ergodicity} implies that for any $\mu \in \bar \Lambda_\mu$,
  we have 
  \begin{align}
    d_\mu(s, a) > 0.
  \end{align}
  Then 
  by the continuity of invariant distribution (Lemma~\ref{lem continuity of ergodic distribution}) and the extreme value theorem,
  we have
  \begin{align}
      d_{\mu, min} \doteq \inf_{\mu \in \bar \Lambda_\mu, s, a} d_\mu(s, a) > 0.
  \end{align}
  Let
  \begin{align}
    \label{eq definition of a theta}
      A_\theta \doteq I - D_{\mu_\theta}(I - \gamma P_{\pi_\theta}),
  \end{align}
  then
  \begin{align}
    \bar F_\theta(q) - \bar F_\theta(q') = A_\theta(q - q').
  \end{align}
  The matrix $A_\theta$ has the following properties
  \begin{enumerate}[(i).]
    \item Each element of $A_\theta$ is always nonnegative
    \item The column sum of $A_\theta$ is always smaller than 2
    \item The row sum of $A_\theta$ is always smaller than $\kappa_0 \doteq 1 - (1-\gamma)d_{\mu, min}$ and greater than 0.
  \end{enumerate}
  To see (i),
  for any diagonal entry, we have
  \begin{align}
    A_\theta(i, i) = 1 - d_{\mu_\theta}(i) + \gamma d_{\mu_\theta}(i) P_{\pi_\theta}(i, i) \geq 0;
  \end{align}
  for any off-diagonal entry, we have
  \begin{align}
    A_\theta(i, j) = \gamma (D_{\mu_\theta}P_{\pi_\theta})(i, j) \geq 0.
  \end{align}
  To see (ii), we have
  \begin{align}
    \tb{1}^\top A_\theta = \tb{1}^\top - d_{\mu_\theta}^\top + \gamma d_{\mu_\theta}^\top P_{\pi_\theta}. 
  \end{align}
  Then (ii) follows immediately from the fact that $d_{\mu_\theta}^\top P_{\pi_\theta}$ is a valid probability distribution.
  To see (iii), we have
  \begin{align}
    A_\theta \tb{1} = \tb{1} - d_{\mu_\theta} + \gamma d_{\mu_\theta} = \tb{1} - (1 - \gamma) d_{\mu_\theta}.
  \end{align}
  Then for each $i$,
  $\left(A_\theta \tb{1}\right)(i) > 0$ and 
  \begin{align}
    \left(A_\theta \tb{1}\right)(i) = 1 - (1-\gamma) d_{\mu_\theta}(i) \leq 1 - (1-\gamma) d_{\mu, min} = \kappa_0 < 1.
  \end{align}
  With those three properties,
  for any $\ell_p$ norm with $p > 1$, we have
  \begin{align}
      &\norm{A_\theta x}_p^p \\
      =& \sum_i \left|\sum_j A_\theta(i, j)x_j\right|^p \\
      =&\sum_i \left(\sum_k A_\theta(i, k)\right)^p \left|\sum_j \frac{A_\theta(i, j)}{\sum_k A_\theta(i, k)}x_j\right|^p \qq{(Row sum of $A_\theta$ is strictly positive)}\\
      \leq& \sum_i \left(\sum_k A_\theta(i, k)\right)^p \sum_j \frac{A_\theta(i, j)}{\sum_k A_\theta(i, k)}\left|x_j\right|^p \\
      \intertext{\hfill (Jensen's inequality and convexity of $\abs{\cdot}^p$)} \\
      =& \sum_i \left(\sum_k A_\theta(i, k)\right)^{p-1} \sum_j {A_\theta(i, j)}\left|x_j\right|^p \\
      \leq& \sum_i \kappa_0^{p-1} \sum_j A_\theta(i, j) \left|x_j\right|^p \qq{(Row sum of $A_\theta$ is smaller than $\kappa_0$)} \\
      =& \kappa_0^{p-1} \sum_j \left|x_j\right|^p \sum_i A_\theta(i, j) \\
      \leq& 2 \kappa_0^{p-1} \sum_j \left|x_j\right|^p,
  \end{align}
  implying
  \begin{align}
      \norm{A_\theta x}_p \leq (2\kappa_0^{p-1})^{\frac{1}{p}} \norm{x}_p.
  \end{align}
  Since $\kappa_0 < 1$,
  for
  sufficiently large $p$, we have
  \begin{align}
    2 \kappa_0^{p-1} < 1,
  \end{align}
  implying
  \begin{align}
    \kappa \doteq (2\kappa_0^{p-1})^{\frac{1}{p}} < 1.
  \end{align} 
  Consequently,
  \begin{align}
    \norm{\bar F_\theta(q) - \bar F_\theta(q')}_p = \norm{A_\theta (q - q')}_p \leq \kappa \norm{q - q'}_p,
  \end{align}
  i.e.,
  $\bar F_\theta$ is a $\kappa$-contraction w.r.t. $\norm{\cdot}_p$ for all $\theta$.
  Further,
  \begin{align}
    \bar F_\theta(q) &= q, \\
    \iff D_{\mu_\theta}(r + \gamma P_{\pi_\theta} q - q) &= 0, \\
    \iff r + \gamma P_{\pi_\theta} q - q &=0, \\
    \iff q &= q_{\pi_\theta},
  \end{align}
  which completes the proof.
\end{proof}

\subsection{Proof of Proposition~\ref{prop critic convergence}}

\propcriticconvergence*
\label{sec proof prop critic convergence}
\begin{proof}
  As previously described, 
  the iterates $\qty{q_t}$ in Algorithm~\ref{alg opac} evolve according to \eqref{eq critic update with F}.
  We, therefore, proceed by verifying Assumptions~\ref{assu makovian} - \ref{assu twotimescale} in order to invoke Theorem~\ref{thm sa convergence}.

  To start with,
  define
  \begin{align}
    \label{eq definition of y}
    \fY &\doteq \qty{(s, a, s') \mid s \in \fS, a \in \fA, s' \in \fS, p(s'|s, a) > 0}, \\
    Y_t &\doteq (S_t, A_t, S_{t+1}), \\
    P_{\theta}((s_1, a_1, s_1'), (s_2, a_2, s_2')) &\doteq \begin{cases}
      0 & s_1' \neq s_2 \\
      \mu_\theta(a_2|s_2) p(s_2'|s_2, a_2) & s_1' = s_2.
    \end{cases}
  \end{align} 
  According to the action selection rule for $A_{t}$ specified in Algorithm~\ref{alg opac},
  we have
  \begin{align}
    \Pr(Y_{t+1} = y) = P_{\theta_{t+1}}(Y_t, y),
  \end{align}
  Assumption~\ref{assu makovian} is then fulfilled.

  Assumption~\ref{assu uniform ergodicity} is immediately implied by Assumption~\ref{assu mu uniform ergodicity}.
  In particular, 
  for any $\theta$,
  the invariant distribution of 
  the chain induced by $P_\theta$ is $d_{\mu_\theta}(s)\mu_\theta(a|s)p(s'|s, a)$.

  Assumption~\ref{assu uniform contraction} is verified by Lemma~\ref{lem uniform contraction opac}.

  We now verify Assumption~\ref{assu regularization}.
  In particular,
  the norm $\norm{\cdot}_c$ in Section~\ref{sec sa} is now realized as the $\ell_p$ norm specified by Lemma~\ref{lem uniform contraction opac}.
  We will repeatedly use the equivalence between 
  $\norm{\cdot}_\infty$, $\norm{\cdot}$, and $\norm{\cdot}_p$,
  i.e.,
  there exist positive constants $l_{\infty, p}, u_{\infty, p}$, $l_{2, p}, u_{2, p}$ such that $\forall x$,
  \begin{align}
    l_{\infty, p} \norm{x}_\infty &\leq \norm{x}_p \leq u_{\infty, p} \norm{x}_\infty \\
    l_{2, p} \norm{x} &\leq \norm{x}_p \leq u_{2, p} \norm{x}.
  \end{align}

  To verify Assumption~\ref{assu regularization} (i), for any $y = (s_0, a_0, s_1)$,
  we have,
  \begin{align}
    \left(F_\theta(q, y) - F_\theta(q', y)\right)(s, a) = \begin{cases}
      q(s, a) - q'(s, a), & (s, a) \neq (s_0, a_0) \\
      \gamma \sum_{a_1}\pi_\theta(a_1| s_1) \left(q(s_1, a_1) - q'(s_1, a_1)\right), &(s, a) = (s_0, a_0)
    \end{cases}.
  \end{align}
  Hence
  \begin{align}
    \norm{F_\theta(q, y) - F_\theta(q', y)}_\infty \leq \norm{q - q'}_\infty,
  \end{align}
  implying
  \begin{align}
    \norm{F_\theta(q, y) - F_\theta(q', y)}_p \leq \frac{u_{\infty, p}}{l_{\infty, p}} \norm{q - q'}_p.
  \end{align}
  Assumption~\ref{assu regularization} (i) is then fulfilled.

  To verify Assumption~\ref{assu regularization} (ii),
  for any $y = (s_0, a_0, s_1)$,
  we have
  \begin{align}
    \left(F_\theta(q, y) - F_{\theta'}(q, y)\right)(s, a) = \begin{cases}
      0, & (s, a) \neq (s_0, a_0) \\
      \gamma \sum_{a_1} \left(\pi_\theta(a_1| s_1) - \pi_{\theta'}(a_1|s_1) \right) q(s_1, a_1), &(s, a) = (s_0, a_0)
    \end{cases}.
  \end{align}
  Hence
  \begin{align}
    \norm{F_\theta(q, y) - F_{\theta'}(q, y)}_\infty \leq \gamma \na L_\pi \norm{\theta - \theta'}\norm{q}_\infty \qq{(using \eqref{eq pi lipschitz}),}
  \end{align}
  implying
  \begin{align}
    \norm{F_\theta(q, y) - F_{\theta'}(q, y)}_p \leq \frac{u_{\infty, p}\gamma \na L_\pi}{l_{\infty, p} l_{2, p}} \norm{\theta - \theta'}_p\norm{q}_p.
  \end{align}
  Assumption~\ref{assu regularization} (ii) is then fulfilled.

  To verify Assumption~\ref{assu regularization} (iii), 
  for any $y = (s_0, a_0, s_1)$,
  we have
  \begin{align}
    \left(F_\theta(0, y)\right)(s, a) = \begin{cases}
      0, & (s, a) \neq (s_0, a_0) \\
      r(s_0, a_0), &(s, a) = (s_0, a_0)
    \end{cases}.
  \end{align}
  Then 
  \begin{align}
    \norm{F_\theta(0, y)}_p \leq u_{\infty, p} \norm{F_\theta(0, y)}_\infty \leq  u_{\infty, p}r_{max}.
  \end{align}
  Assumption~\ref{assu regularization} (iii) is the fulfilled.

  To verify Assumption~\ref{assu regularization} (iv), 
  we have
  \begin{align}
    \bar F_\theta(q) - \bar F_{\theta'}(q)
    = \left(D_{\mu_\theta} - D_{\mu_{\theta'}}\right) r + \gamma \left(D_{\mu_\theta}P_{\pi_\theta} - D_{\mu_{\theta'} } P_{\pi_{\theta'}} \right) q- \left(D_{\mu_\theta} - D_{\mu_{\theta'}}\right) q .
  \end{align}
  Since $D_{\mu_\theta}$ is Lipschitz continuous in $\theta$ (Lemma~\ref{lem continuity of ergodic distribution}) and $\norm{D_{\mu_\theta}}$ is bounded from above,
  and $P_{\pi_\theta}$ is Lipschitz continuous in $\theta$ (see \eqref{eq pi lipschitz}) and $\norm{P_{\pi_\theta}}$ is bounded from the above,
  Lemma~\ref{lem product of lipschitz functions} confirms the Lipschitz continuity of $\bar F_\theta$,
  which completes the verification of Assumption~\ref{assu regularization} (iv).

  To verify Assumption~\ref{assu regularization} (v),
  recall that Lemma~\ref{lem uniform contraction opac} asserts that the fixed point of $\bar F_\theta$ is $q_{\pi_\theta}$.
  We have
  \begin{align}
    \label{eq lipschitz continuity of q}
    q_{\pi_\theta} - q_{\pi_{\theta'}} = \left((I - \gamma P_{\pi_\theta})^{-1} - (I - \gamma P_{\pi_{\theta'}})^{-1}\right) r.
  \end{align}
  Using Lemma~\ref{lem bound of matrix inverse diff} yields
  \begin{align}
    \norm{q_{\pi_\theta} - q_{\pi_{\theta'}}}_p \leq \norm{(I - \gamma P_{\pi_\theta})^{-1}}_p  \norm{\gamma P_{\pi_\theta} - \gamma P_{\pi_{\theta'}}}_p \norm{(I - \gamma P_{\pi_{\theta'}})^{-1}}_p \norm{r}_p.
  \end{align}
  Notice that 
  (1) for any policy $\pi$, $(I - \gamma P_\pi)^{-1}$ is always well-defined;
  (2) $(I - \gamma P_\pi)^{-1}$ is continuous in $\pi$ (this can be seen by writing the inverse explicitly with the adjugate matrix); (3) the space of all policies is compact,
  by the extreme value theorem we conclude that 
  \begin{align}
    \sup_{\theta} \norm{(I - \gamma P_{\pi_\theta})^{-1}}_p < \infty,
  \end{align}
  which together with \eqref{eq pi lipschitz} completes the verification of Assumption~\ref{assu regularization} (v).

  Assumption~\ref{assu regularization} (vi) follows immediately from the fact that
  \begin{align}
    \abs{q_{\pi_\theta}(s, a)} \leq \frac{r_{max}}{1 - \gamma}.
  \end{align}

  Assumption~\ref{assu regularization} (vii) follows immediately from Assumption~\ref{assu lipschitz mu}.

  Assumption~\ref{assu mds} is automatically fulfilled since in our setting we have $\epsilon_t \equiv 0$.

  Assumption~\ref{assu twotimescale} is identical to Assumption~\ref{assu three timescales} except for \eqref{eq konda actor update}.
  To verify~\eqref{eq konda actor update},
  we first establish the boundedness of $\qty{q_t}$.
  It can be easily seen that
  \begin{align}
    \abs{q_{t+1}(S_t, A_t)} \leq& (1 - \alpha_t) \abs{q_t(S_t, A_t)} + \alpha_t (r_{\max} + \gamma \norm{q_t}_\infty) \\
    \leq& (1 - \alpha_t) \norm{q_t}_\infty + \alpha_t (r_{\max} + \gamma \norm{q_t}_\infty).
  \end{align}
  Suppose $t_0$ is sufficiently large such that $\alpha_t < 1$, 
  then we have
  \begin{align}
    &\norm{q_t}_\infty \leq \frac{r_{max}}{1 - \gamma} \implies \abs{q_{t+1}(S_t, A_t)} \leq \frac{r_{max}}{1 - \gamma} \implies \norm{q_{t+1}}_\infty \leq \frac{r_{max}}{1 - \gamma}, \\
    &\norm{q_t}_\infty \geq \frac{r_{max}}{1 - \gamma} \implies \abs{q_{t+1}(S_t, A_t)} \leq \norm{q_t}_\infty \implies \norm{q_{t+1}}_\infty \leq \norm{q_t}_\infty.
  \end{align}
  It is then trivial to see that for any $t$
  \begin{align}
    \label{eq boundedness q}
    \norm{q_t}_\infty \leq C_q \doteq \max\qty{\frac{r_{max}}{1 - \gamma}, \norm{q_0}_\infty}.
  \end{align}
  According to the updates of $\qty{\theta_t}$ in Algorithm~\ref{alg opac},
  we have
  \begin{align}
    &\norm{\theta_{t+1} - \theta_t} \\
    =& \beta_t \norm{\rho_t  \nabla_\theta \log \pi_{\theta_t}(A_t | S_t) q_t(S_t, A_t) - \lambda_t \nabla_{\theta} \kl{\fU_\fA}{\pi_{\theta_t}(\cdot | S_t)}} \\
    \leq& \beta_t \left( \norm{\rho_t} \norm{\nabla_\theta \log \pi_{\theta_t}(A_t | S_t)} C_q + \lambda_t \norm{\nabla_{\theta} \kl{\fU_\fA}{\pi_{\theta_t}(\cdot | S_t)}} \right).
  \end{align}
  Assumption~\ref{assu mu uniform ergodicity} and the extreme value theorem ensures that
  \begin{align}
    \inf_{\theta, s, a} \mu_{\theta}(a|s) > 0.
  \end{align}
  Hence
  \begin{align}
    \rho_{max} \doteq \sup_{\theta, s, a} \frac{\pi_\theta(a|s)}{\mu_\theta(a|s)} < \infty,
  \end{align}
  implying $\forall t,$
  \begin{align}
    \norm{\rho_t} < \infty.
  \end{align}
  Assumption~\ref{assu three timescales} ensures 
  \begin{align}
    \lambda_t \leq \lambda.
  \end{align}
  Lemma~\ref{lem softmax policy gradient} ensures the boundedness of $\norm{\nabla_\theta \log \pi_{\theta_t}(A_t | S_t)}$ and $\norm{\nabla_{\theta} \kl{\fU_\fA}{\pi_{\theta_t}(\cdot | S_t)}}$,
  from which it is easy to see that there exists a constant $L_\theta$ 
  such that 
  \begin{align}
    \label{eq ltheta}
    \norm{\theta_{t+1} - \theta_t}_p \leq \beta_t L_\theta,
  \end{align}
  completing the verification of Assumption~\ref{assu twotimescale}.

  With Assumptions \ref{assu makovian} - \ref{assu twotimescale} satisfied,
  invoking Theorem~\ref{thm sa convergence} completes the proof.
\end{proof}

\subsection{Proof of Theorem~\ref{thm:optimality-AC}}
\optimalityAC*
\label{sec proof thm:optimality-AC}
\noindent \textbf{Proof Sketch} 
We start with a proof sketch and then proceed to the full proof.
We first define a KL regularized objective 
\begin{align}
  J_\eta(\pi; p_0) \doteq J(\pi; p_0) - \eta \E_{s \sim \fU_\fS} \left[\kl{\fU_\fA}{\pi(\cdot |s)}\right],
\end{align}
where $\fU_\mathcal{X}$ denotes the uniform distribution on the set $\mathcal{X}$.
Key to our proof is the following lemma:
\begin{lemma}
  (Theorem 5.2 of \citet{agarwal2019optimality})
  For any state distribution $d$ and $d'$,
  if 
  \begin{align}
    \norm{\nabla J_\eta(\pi_\theta; d')} \leq \frac{\eta}{2\nsa},
  \end{align}
  then
  \begin{align}
    J(\pi_\theta; d) \geq J(\pi_*; d) - \frac{2\eta}{1 - \gamma} \max_s {\frac{d_{\pi_*, \gamma, d}(s)}{d'(s)}},
  \end{align}
  where $\pi_*$ can be any optimal policy in \eqref{eq optimal policy}. 
\end{lemma}
The above lemma establishes the suboptimality of the stationary points of the KL regularized objective.
If we can find those stationary points and decay the weight of the KL regularization ($\eta$) properly,
optimality is then expected.

There are, however, two caveats.
First, for the above lemma to be nontrivial,
we have to ensure $\forall s, d'(s) > 0$. 
Consequently, 
we cannot simply set $d=d'=p_0$ because we do not make any assumption about $p_0$.
Instead, we consider an artificial state distribution $p_0'$ such that $\forall s, p_0'(s) > 0$ and set $d=p_0, d'=p_0'$.
The second caveat is the following.
To use the above lemma,
we now have to optimize $J_\eta(\pi_\theta; p_0')$ to find its stationary points.
This objective involves state distributions $p_0'$ and $\fU_\fS$.
We, however, only have access to samples from 
\begin{align}
  d_t(s) \doteq \Pr(S_t = s | p_0, \mu_{\theta_0}, \dots, \mu_{\theta_{t-1}}).
\end{align}
We, therefore, would need to reweight them using 
\begin{align}
  \frac{d_{\pi_{\theta}, \gamma, p_0'}(s)}{d_t(s)} \qq{and} \frac{\fU_\fS(s)}{d_t(s)}.
\end{align}
Obviously we do not know those quantities but fortunately we can bound them.
As a consequence,
the reweightings can be properly accounted for even without knowing them exactly (see in particular Lemma~\ref{lem bound of m11}).
With those two caveats addressed,
we are now ready to present the full proof.

\begin{proof}
  This proof borrows ideas from \citet{wu2020finite} but is much more convoluted since we have the additional decaying KL regularization and our algorithm is off-policy without using density ratio for correcting the state distribution mismatch.
  Define the KL regularized objective
  \begin{align}
    \label{eq definition of new objective}
    J_\eta(\pi; p_0) \doteq J(\pi; p_0) - \eta \E_{s \sim \fU_\fS} \left[\kl{\fU_\fA}{\pi(\cdot |s)}\right],
  \end{align}
  where $\fU_\mathcal{X}$ denotes the uniform distribution on the set $\mathcal{X}$.
  Let $p_0'$ be an arbitrary distribution on $\fS$ such that $p_0'(s) > 0$ holds for all $s \in \fS$.
  In the rest of this proof,
  we use as shorthand
  \begin{align}
    \label{eq actor convergence shorthard}
    J(\theta) &\doteq J(\pi_\theta; p_0'), \\
    J_\eta(\theta) &\doteq J_\eta(\pi_\theta; p_0'), \\
    d_{\pi, \gamma}(s) &\doteq d_{\pi, \gamma, p_0'}(s),
  \end{align}
  i.e.,
  we work on the initial distribution $p_0'$ (instead of $p_0$) by default.
  Note that the sampling is still done with respect to $p_0$, $p_0'$ is simply an auxiliary distribution used for the proof. 
  Similarly, the KL regularized objective is built with a uniform distribution that does not correspond to what the algorithm implements. 
  This too is a proof artefact. 
  Both mismatches are accounted for, in particular in Lemma~\ref{lem bound of m11}.

  According to Lemma 7 of \citet{mei2020global}, $J(\theta)$ is $L_J$-smoothness for some positive constant $L_J$ w.r.t $\norm{\cdot}$.
  Consequently,
  the Hessian of $J(\theta)$ is bounded from above by $L_J$.
  From Lemma~\ref{lem softmax policy gradient},
  it is easy to see the Hessian of $\E_{s \sim \fU_\fS} \left[\kl{\fU_\fA}{\pi_\theta(\cdot |s)}\right]$ is also bounded from above by some positive constant $L_{\text{KL}}$.
  Consequently,
  the Hessian of $J_\eta(\theta)$ is bounded from above by $L_J + \eta L_{\text{KL}}$,
  i.e.,
  $J_\eta(\theta)$ is $(L_J + \eta L_\text{KL})$-smooth.
  With $\eta = \lambda_t$,
  Lemma~\ref{lem smooth definition} then implies
  \begin{align}
    \label{eq j lambda smooth inequality}
    J_{\lambda_t}(\theta_{t+1}) \geq& J_{\lambda_t}(\theta_t) + \indot{\nabla J_{\lambda_t}(\theta_t)}{\theta_{t+1} - \theta_t} - (L_J + \lambda_t L_{\text{KL}}) \norm{\theta_{t+1}-\theta_t}^2 \\
    \geq& J_{\lambda_t}(\theta_t) + \underbrace{\indot{\nabla J_{\lambda_t}(\theta_t)}{\theta_{t+1} - \theta_t}}_{M_1} - L_J' \underbrace{\norm{\theta_{t+1}-\theta_t}^2}_{M_2},
  \end{align}
  where $L_J' \doteq L_J + \lambda L_\text{KL}$.
  Using \eqref{eq ltheta} to bound $M_2$ yields
  \begin{align}
    M_2 \leq \frac{1}{l_{2,p}} \beta_t^2 L_\theta^2.
  \end{align}
  To bound $M_1$, 
  let 
  \begin{align}
    \label{eq actor yt}
    Y_t \doteq (S_t, A_t).
  \end{align}
  Here different from \eqref{eq definition of y},
  we redefine $Y_t$ to consider only state action pairs to ease presentation.
  For $y = (s, a)$, we define
  \begin{align}
    \label{eq actor helper function}
    \Lambda(\theta, y, \eta) &\doteq \frac{\pi_{\theta}(a|s)}{\mu_{\theta}(a|s)} \nabla \log \pi_\theta(a|s) q_{\pi_\theta}(s, a) + \frac{\eta}{\na} \sum_{\bar a} \nabla \log \pi_\theta(\bar a | s) \\
    \bar \Lambda(\theta, \eta) &\doteq \sum_{s, a} d_{\mu_\theta}(s) \mu_\theta(a|s) \Lambda(\theta, y, \eta).
  \end{align}
  Then we have
  \begin{align}
    M_1 = &\indot{\nabla J_{\lambda_t}(\theta_t)}{\theta_{t+1} - \theta_t} \\
    =&\beta_t \indot{\nabla J_{\lambda_t}(\theta_t)}{ \rho_t  \nabla \log \pi_{\theta_t}(A_t | S_t) q_t(S_t, A_t) - \lambda_t \nabla \kl{\fU_\fA}{\pi_{\theta_t}(\cdot | S_t)} } \\
    =&\beta_t \indot{\nabla J_{\lambda_t}(\theta_t)}{ \rho_t  \nabla \log \pi_{\theta_t}(A_t | S_t) q_t(S_t, A_t) + \frac{\lambda_t}{\na} \sum_a \nabla \log \pi_{\theta_t}(a|S_t) } \\
    =&\beta_t \indot{\nabla J_{\lambda_t}(\theta_t)}{ \rho_t  \nabla \log \pi_{\theta_t}(A_t | S_t) q_{\pi_{\theta_t}}(S_t, A_t) + \frac{\lambda_t}{\na} \sum_a \nabla \log \pi_{\theta_t}(a|S_t) } \\
    &+ \beta_t \indot{\nabla J_{\lambda_t}(\theta_t)}{ \rho_t  \nabla \log \pi_{\theta_t}(A_t | S_t) \left(q_t(S_t, A_t) - q_{\pi_{\theta_t}}(S_t, A_t) \right) } \\
    =&\beta_t \underbrace{\indot{\nabla J_{\lambda_t}(\theta_t)}{\bar \Lambda(\theta_t, \lambda_t)}}_{M_{11}} + \beta_t \underbrace{\indot{\nabla J_{\lambda_t}(\theta_t)}{\Lambda(\theta_t, Y_t, \lambda_t) - \bar \Lambda(\theta_t, \lambda_t)}}_{M_{12}} \\
    &+ \beta_t \underbrace{\indot{\nabla J_{\lambda_t}(\theta_t)}{ \rho_t  \nabla \log \pi_{\theta_t}(A_t | S_t) \left(q_t(S_t, A_t) - q_{\pi_{\theta_t}}(S_t, A_t) \right) }}_{M_{13}}.
  \end{align}
  To bound $M_{12}$,
  define 
  \begin{align}
    \label{eq actor helper function 2}
    \Lambda'(\theta, y, \eta) \doteq \indot{\nabla J_{\eta}(\theta)}{\Lambda(\theta, y, \eta) - \bar \Lambda(\theta, \eta)}.
  \end{align}
  Assumption~\ref{assu mu uniform ergodicity} and Lemma~\ref{lem uniform mixing} assert that 
  there exist constants $C_0 > 0$ and $\tau \in (0, 1)$,
  independent of $\theta$,
  such that for any $n > 0$,
  \begin{align}
    \sup_{s, a, \theta}\sum_{s', a'} \abs{P_{\mu_\theta}^n((s, a), (s', a')) - d_{\mu_\theta}(s')\mu_{\theta}(a'|s')} \leq C_0 \tau^n,
  \end{align} 
  which allows us to define 
  \begin{align}
    \label{eq definition of beta t}
    \tau_{\beta_t} \doteq \min\qty{n \mid \sup_{s, a, \theta}\sum_{s', a'} \abs{P_{\mu_\theta}^n((s, a), (s', a')) - d_{\mu_\theta}(s')\mu_{\theta}(a'|s')}  \leq \beta_t}.
  \end{align}
  We then decompose $M_{12}$ as
  \begin{align}
    M_{12} =& \Lambda'(\theta_t, Y_t, \lambda_t) \\
    =& \underbrace{\Lambda'(\theta_t, Y_t, \lambda_t) - \Lambda'(\theta_{t-\tau_{\beta_t}}, Y_t, \lambda_t)}_{M_{121}} \\
    &+ \underbrace{\Lambda'(\theta_{t-\tau_{\beta_t}}, Y_t, \lambda_t) - \Lambda'(\theta_{t-\tau_{\beta_t}}, \tilde Y_t, \lambda_t)}_{M_{122}} \\
    &+ \underbrace{\Lambda'(\theta_{t-\tau_{\beta_t}}, \tilde Y_t, \lambda_t)}_{M_{123}}.
  \end{align}
  Here $\tilde Y_t$ is an auxiliary chain akin to \citet{zou2019finite} and the one used in the proof of Theorem \ref{thm sa convergence} in \ref{sec proof thm sa convergence} (for $\beta_t$ instead of $\alpha_t$). 
  Before time $t - \tau_{\beta_t} - 1$,
  $\qty{\tilde Y_t}$ is exactly the same as $\qty{Y_t}$.
  After time $t - \tau_{\beta_t} - 1$,
  $\qty{\tilde Y_t}$ evolves according to the fixed behavior policy $\mu_{\theta_{t-\tau_{\beta_t}}}$
  while $\qty{Y_t}$ evolves according to the changing behavior policy $\mu_{\theta_{t-\tau_{\beta_t}}}$, $\mu_{\theta_{t-\tau_{\beta_t} + 1}}$, \dots.
  \begin{align}
    \label{eq actor tilde yt}
    \qty{\tilde Y_t}&: \dots \to Y_{t-\tau_{\beta_t}-1} \underbrace{\to}_{\mu_{\theta_{t-\tau_{\beta_t}}}} Y_{t-\tau_{\beta_t}} \underbrace{\to}_{\mu_{\theta_{t-\tau_{\beta_t}}}} \tilde Y_{t-\tau_{\beta_t}+1} \underbrace{\to}_{\mu_{\theta_{t-\tau_{\beta_t}}}} \tilde Y_{t-\tau_{\beta_t}+2} \to \dots \\
    \qty{Y_t}&: \dots \to Y_{t-\tau_{\beta_t}-1} \underbrace{\to}_{\mu_{\theta_{t-\tau_{\beta_t}}}} Y_{t-\tau_{\beta_t}} \underbrace{\to}_{\mu_{\theta_{t-\tau_{\beta_t}+1}}} Y_{t-\tau_{\beta_t}+1} \underbrace{\to}_{\mu_{\theta_{t-\tau_{\beta_t}+2}}} Y_{t-\tau_{\beta_t}+2} \to \dots.
\end{align}
Let us proceed to bounding each term defined above:

\begin{restatable}{lemma}{lemboundofmoneone}
  \label{lem bound of m11}
(Bound of $M_{11}$)
There exists a constant $\chi_{11} > 0$ such that,
\begin{align}
  M_{11} \geq \chi_{11} \norm{\nabla J_{\lambda_t}(\theta_t)}^2.
\end{align}
\end{restatable}
\noindent
The proof of Lemma~\ref{lem bound of m11} is provided in Section~\ref{sec proof lem bound of m11}.

\begin{restatable}{lemma}{lemboundofmonetwoone}
  \label{lem bound of m121}
  (Bound of $M_{121}$)
  There exist constants $L_{\Lambda'}^* > 0$ such that
  \begin{align}
    \norm{M_{121}} \leq \frac{L_{\Lambda'}^* L_\theta}{l_{2, p}} \beta_{t-\tau_{\beta_t}, t-1}.
  \end{align}
\end{restatable}
\noindent
The proof of Lemma~\ref{lem bound of m121} is provided in Section~\ref{sec proof lem bound of m121}

\begin{restatable}{lemma}{lemboundofmonetwotwo}
  \label{lem bound of m122}
  (Bound of $M_{122}$)
  There exists a constant $U_{\Lambda'}^* > 0$ such that
  \begin{align}
      \norm{\E\left[M_{122}\right]} \leq U_{\Lambda'}^* \ns \na L_\mu L_\theta \sum_{j=t-\tau_{\beta_t}}^{t-1} \beta_{t-\tau_{\beta_t}, j}.
  \end{align}
\end{restatable}
\noindent
The proof of Lemma~\ref{lem bound of m122} is provided in Section~\ref{sec proof lem bound of m122}

\begin{restatable}{lemma}{lemboundofmonetwothree}
  \label{lem bound of m123}
  (Bound of $M_{123}$)
  \begin{align}
    \norm{\E\left[M_{123}\right]} \leq U_{\Lambda'}^* \beta_t.
  \end{align}
\end{restatable}
\noindent
The proof of Lemma~\ref{lem bound of m123} is provided in Section~\ref{sec proof lem bound of m123}.

\begin{restatable}{lemma}{lemboundofmonethree}
  \label{lem bound of m13}
  (Bound of $M_{13}$)
  There exists a constant $\rho_{max} > 0$ such that
  \begin{align}
    \norm{\E\left[M_{13}\right]}\leq& 2\rho_{max} \sqrt{\nsa}\sqrt{\E\left[\norm{q_t - q_{\pi_{\theta_t}}}_\infty^2 \right]} \sqrt{\E\left[\norm{\nabla J_{\lambda_t}(\theta_t)}^2\right]}.
  \end{align}
\end{restatable}
\noindent
The proof of Lemma~\ref{lem bound of m13} is provided in Section~\ref{sec proof lem bound of m13}.

We now assemble the bounds of $M_{11}, M_{121}, M_{122}, M_{123}, M_{12}$ and $M_2$ back to \eqref{eq j lambda smooth inequality}.
Similar to Lemma~\ref{lem learning rates},
it is easy to see for sufficiently large $t_0$,
\begin{align}
  \label{eq asymptotic learning rates beta}
  \tau_{\beta_t} &= \fO\left(\log (t+t_0)\right), \\
  \beta_{t-\tau_{\beta_t}, t-1} &= \fO\left(\frac{\log (t+t_0)}{(t+t_0)^{\epsilon_\beta}}\right), \\
  \sum_{j=t-\tau_{\beta_t}}^{t-1} \beta_{t-\tau_{\beta_t}, j} &= \fO\left(\frac{\log^2(t+t_0)}{(t+t_0)^{\epsilon_\beta}}\right).
\end{align}
Hence if $t_0$ is sufficiently large,
there exist positive constants $\chi_{12}, \chi_{13}, \chi_2$ such that
\begin{align}
  \E\left[M_{121} + M_{122} + M_{123}\right] &\geq -\chi_{12} \frac{\log^2(t+t_0)}{(t+t_0)^{\epsilon_\beta}}, \\
  \E\left[M_{13}\right] &\geq -\chi_{13}\sqrt{\E\left[\norm{q_t - q_{\pi_{\theta_t}}}_p^2 \right]} \sqrt{\E\left[\norm{\nabla J_{\lambda_t}(\theta_t)}^2\right]}, \\
  \E\left[M_2\right] &\leq \beta_t \chi_2 \frac{1}{(t+t_0)^{\epsilon_\beta}},
\end{align}
where the $\ell_p$ norm is defined by Proposition~\ref{prop critic convergence}.
Then, from \eqref{eq j lambda smooth inequality}, we get
\begin{align}
  \label{eq actor perf recursive bound}
  \E\left[J_{\lambda_t}(\theta_{t+1})\right] \geq& \E\left[J_{\lambda_t}(\theta_t)\right] + \beta_t \chi_{11} \E\left[\norm{\nabla J_{\lambda_t}(\theta_t)}^2\right] \\
  &-\beta_t \chi_{12} \frac{\log^2(t+t_0)}{(t+t_0)^{\epsilon_\beta}} \\
  &-\beta_t \chi_{13}\sqrt{\E\left[\norm{q_t - q_{\pi_{\theta_t}}}_p^2 \right]} \sqrt{\E\left[\norm{\nabla J_{\lambda_t}(\theta_t)}^2\right]} \\
  &- \beta_t \chi_2 \frac{1}{(t+t_0)^{\epsilon_\beta}}.
\end{align}
Rearranging terms yields
\begin{align}
  \E\left[\norm{\nabla J_{\lambda_t}(\theta_t)}^2\right] \leq& \frac{1}{\chi_{11}} \frac{1}{\beta_t} \left(\E\left[J_{\lambda_t}(\theta_{t+1})\right] - \E\left[J_{\lambda_t}(\theta_t)\right]\right) \\
  &+ \frac{\chi_{13}}{\chi_{11}} \sqrt{\E\left[\norm{q_t - q_{\pi_{\theta_t}}}_p^2 \right]} \sqrt{\E\left[\norm{\nabla J_{\lambda_t}(\theta_t)}^2\right]} \\
  &+ \frac{\chi_{12} + \chi_2}{\chi_{11}} \frac{\log^2(t+t_0)}{(t+t_0)^{\epsilon_\beta}}.
\end{align}
Defining
\begin{align}
  \chi_3 \doteq \frac{1}{\chi_{11}}, \, \chi_4 \doteq \frac{\chi_{13}}{\chi_{11}}, \, \chi_5 \doteq \frac{\chi_{12} + \chi_2}{\chi_{11}}
\end{align}
and telescoping the above inequality from $\ceil{\frac{t}{2}}$ to $t$ yields
\begin{align}
  \label{eq tmp 6}
  \sum_{k=\ceil{\frac{t}{2}}}^t \E\left[\norm{ \nabla J_{\lambda_k}(\theta_k)}^2\right] \leq& \chi_3 \sum_{k=\ceil{\frac{t}{2}}}^t \frac{1}{\beta_k} \left(\E\left[J_{\lambda_k}(\theta_{k+1})\right] - \E\left[J_{\lambda_k}(\theta_k)\right]\right) \\
  &+ \chi_4 \sum_{k=\ceil{\frac{t}{2}}}^t \sqrt{\E\left[\norm{q_k - q_{\pi_{\theta_k}}}_p^2 \right]} \sqrt{\E\left[\norm{\nabla J_{\lambda_k}(\theta_k)}^2\right]} \\
  &+ \chi_5 \sum_{k=\ceil{\frac{t}{2}}}^t \frac{\log^2(k+t_0)}{(k+t_0)^{\epsilon_\beta}}. 
\end{align}
We now bound the right terms of the above inequality.
\begin{restatable}{lemma}{lemboundjlambda}
  \label{lem bound j lambda} 
  There exists a constant $U_{J, \lambda}$ such that for all $t$,
  \begin{align}
    \abs{\E\left[J_{\lambda_t}(\theta_{t})\right]} \leq U_{J, \lambda}, \, \abs{\E\left[J_{\lambda_t}(\theta_{t+1})\right]} \leq U_{J, \lambda}.
  \end{align}
\end{restatable}
\noindent
The proof of Lemma~\ref{lem bound j lambda} is provided in Section~\ref{sec proof lem bound j lambda}.

\begin{restatable}{lemma}{lemboundjlambdadiff}
  \label{lem bound j lambda diff}
  \begin{align}
  \E\left[\sum_{k=\ceil{\frac{t}{2}}}^t \frac{1}{\beta_k}\left(J_{\lambda_k}(\theta_{k+1}) - J_{\lambda_k}(\theta_k) \right) \right]\leq \frac{2U_{J, \lambda}}{\beta} (t+t_0)^{\epsilon_\beta}
  \end{align}
\end{restatable}
\noindent The proof of Lemma~\ref{lem bound j lambda diff} is provided in Section~\ref{sec proof lem bound j lambda diff}.
Using Lemma~\ref{lem bound j lambda diff},
the Cauchy-Schwarz inequality,
and 
\begin{align}
  \sum_{k=\ceil{\frac{t}{2}}}^t \frac{\log^2(k+t_0)}{(k+t_0)^{\epsilon_\beta}} \leq \log^2(t+t_0) \int_{x=\ceil{\frac{t}{2}}-1}^t \frac{1}{(x+t_0)^{\epsilon_\beta}} dx \leq  \frac{\log^2(t+t_0)}{1-\epsilon_\beta} (t+t_0)^{1-\epsilon_\beta}
\end{align}
to bound the RHS of \eqref{eq tmp 6} yields
\begin{align}
  \sum_{k=\ceil{\frac{t}{2}}}^t \E\left[\norm{ \nabla J_{\lambda_k}(\theta_k)}^2\right] \leq& \frac{2\chi_3 U_{J, \lambda}}{\beta} (t+t_0)^{\epsilon_\beta}  + \chi_5 \frac{\log^2(t+t_0)}{1-\epsilon_\beta} (t+t_0)^{1-\epsilon_\beta}\\
  &+ \chi_4 \sqrt{\sum_{k=\ceil{\frac{t}{2}}}^t \E\left[\norm{q_k - q_{\pi_{\theta_k}}}_p^2 \right]} \sqrt{\sum_{k=\ceil{\frac{t}{2}}}^t \E\left[\norm{\nabla J_{\lambda_k}(\theta_k)}^2\right]}.
\end{align}
Multiplying $\frac{1}{t-\ceil{\frac{t}{2}}+1}$ in both sides yields
\begin{align}
  \underbrace{\frac{\sum_{k=\ceil{\frac{t}{2}}}^t \E\left[\norm{ \nabla J_{\lambda_k}(\theta_k)}^2\right]}{t-\ceil{\frac{t}{2}}+1}}_{z_t} \leq& \frac{2\chi_3 U_{J, \lambda}}{\beta} \frac{(t+t_0)^{\epsilon_\beta}}{t - \ceil{\frac{t}{2}} + 1} + \chi_5 \frac{\log^2(t+t_0)}{1-\epsilon_\beta} \frac{(t+t_0)^{1-\epsilon_\beta}}{t - \ceil{\frac{t}{2}} + 1}\\
  &+ \chi_4 \sqrt{\underbrace{\frac{\sum_{k=\ceil{\frac{t}{2}}}^t \E\left[\norm{q_k - q_{\pi_{\theta_k}}}_p^2 \right]}{t-\ceil{\frac{t}{2}}+1}}_{e_t}} \sqrt{\frac{\sum_{k=\ceil{\frac{t}{2}}}^t \E\left[\norm{\nabla J_{\lambda_k}(\theta_k)}^2\right]}{t-\ceil{\frac{t}{2}}+1}}.
\end{align}
It is then easy to see that there exist positive constants $E_1, E_2, E_3, E_4$ such that
\begin{align}
  z_t &\leq \frac{E_1}{t^{1 - \epsilon_\beta}} + \frac{E_2 \log^2 t}{t^{\epsilon_\beta}} + 2E_3 \sqrt{e_t}\sqrt{z_t} \\
  \implies \left(\sqrt{z_t} - E_3 \sqrt{e_t} \right)^2 &\leq \frac{E_1}{t^{1 - \epsilon_\beta}} + \frac{E_2 \log^2 t}{t^{\epsilon_\beta}} + E_3^2 e_t \\
  \implies \sqrt{z_t} - E_3 \sqrt{e_t} &\leq \sqrt{\frac{E_1}{t^{1 - \epsilon_\beta}} + \frac{E_2 \log^2 t}{t^{\epsilon_\beta}} + E_3^2 e_t} \\
  &\leq \sqrt{\frac{E_1}{t^{1 - \epsilon_\beta}} + \frac{E_2 \log^2 t}{t^{\epsilon_\beta}}}  + E_3 \sqrt{e_t} \\ 
  \implies z_t &\leq \frac{2E_1}{t^{1 - \epsilon_\beta}} + \frac{2E_2 \log^2 t}{t^{\epsilon_\beta}} + 8E_3^2 e_t.
\end{align} 
Proposition~\ref{prop critic convergence} implies that there exists a constant $E_5 > 0$ such that 
\begin{align}
  e_t = \frac{\sum_{k=\ceil{\frac{t}{2}}}^t \frac{E_5}{k^{\epsilon_q}}}{t-\ceil{\frac{t}{2}}+1} \leq \frac{E_5 t^{1-\epsilon_q} }{(1 - \epsilon_q) (t - \ceil{\frac{t}{2}} +1)}.
\end{align}
It is then easy to see
\begin{align}
  e_t = \fO\left(\frac{1}{t^{\epsilon_q}}\right),
\end{align}
implying
\begin{align}
  \label{eq actor convergence stationary points}
  \frac{\sum_{k=\ceil{\frac{t}{2}}}^t \E\left[\norm{ \nabla J_{\lambda_k}(\theta_k)}^2\right] }{t-\ceil{\frac{t}{2}}+1} = \fO\left(\frac{1}{t^{1-\epsilon_\beta}} + \frac{\log^2t}{t^{\epsilon_\beta}} + \frac{1}{t^{\epsilon_q}}\right).
\end{align}
The above inequality establishes the convergence to stationary points,
with which we now study the optimality of the sequence $\qty{\theta_t}$.
We rely on the following lemma.
\begin{lemma}
  \label{lem agarwal}
  (Theorem 5.2 of \citet{agarwal2019optimality})
  For any state distribution $d$ and $d'$,
  if 
  \begin{align}
    \norm{\nabla J_\eta(\theta; d')} \leq \frac{\eta}{2\nsa},
  \end{align}
  then
  \begin{align}
    J(\theta; d) \geq J(\pi_*; d) - \frac{2\eta}{1 - \gamma} \max_s {\frac{d_{\pi_*, \gamma, d}(s)}{d'(s)}},
  \end{align}
  where $\pi_*$ can be any optimal policy in \eqref{eq optimal policy}. 
\end{lemma}
Obviously,
for Lemma~\ref{lem agarwal} to be nontrivial,
we have to ensure $d'(s) > 0$.

Fix any $t > 0$.
Then select a $k$ uniformly randomly from $\qty{\ceil{\frac{t}{2}}, \ceil{\frac{t}{2}} + 1, \dots, t-1, t}$.
Now the random variable $\norm{\nabla J_{\lambda_k}(\theta_k)}$ has randomness from both the random selection of $k$ and the learning of $\theta_k$.
Using Markov's inequality yields
\begin{align}
  \Pr(\norm{\nabla J_{\lambda_k}(\theta_k)} \leq \frac{\lambda_t}{2 \nsa}) =&\Pr(\norm{\nabla J_{\lambda_k}(\theta_k)}^2 \leq \frac{\lambda_t^2}{4 \nsa^2}) \\
  \geq& 1 - \frac{4\nsa^2}{\lambda_t^2} \E\left[\norm{\nabla J_{\lambda_k}(\theta_k)}^2\right] \\
  =& 1 - \frac{4\nsa^2}{\lambda_t^2} \E\left[ \E\left[\norm{\nabla J_{\lambda_k}(\theta_k)}^2 \mid k\right] \right] \\
  =& 1 - \frac{4\nsa^2}{\lambda_t^2} \frac{\sum_{i=\ceil{\frac{t}{2}}}^t \E\left[\norm{ \nabla J_{\lambda_k}(\theta_k)}^2 \mid k =i\right] }{t-\ceil{\frac{t}{2}}+1} \\
  \geq & 1 - \frac{1}{\lambda_t^2}\fO\left(\frac{1}{t^{1-\epsilon_\beta}} + \frac{\log^2 t}{t^{\epsilon_\beta}} + \frac{1}{t^{\epsilon_q }}\right)  \qq{(Using \eqref{eq actor convergence stationary points})} \\
  \geq& 1 - C_t,
\end{align}
where 
\begin{align}
  C_t \doteq \fO\left(\frac{1}{t^{1-\epsilon_\beta - 2\epsilon_\lambda}} + \frac{\log^2 t}{t^{\epsilon_\beta- 2\epsilon_\lambda}} + \frac{1}{t^{\epsilon_q- 2\epsilon_\lambda }}\right).
\end{align}
Since $\lambda_k \geq \lambda_t$,
we have
\begin{align}
  \norm{\nabla J_{\lambda_k}(\theta_k)} \leq \frac{\lambda_t}{2 \nsa} \implies \norm{\nabla J_{\lambda_k}(\theta_k)} \leq \frac{\lambda_k}{2 \nsa}.
\end{align}
Consequently,
\begin{align}
  \Pr(\norm{\nabla J_{\lambda_k}(\theta_k)} \leq \frac{\lambda_k}{2 \nsa}) \geq \Pr(\norm{\nabla J_{\lambda_k}(\theta_k)} \leq \frac{\lambda_t}{2 \nsa}) \geq 1 - C_t.
\end{align}
Let $d=p_0, d' = p_0', \eta = \lambda_k$ in Lemma~\ref{lem agarwal} and recall \eqref{eq actor convergence shorthard},
we get 
\begin{align}
  J(\theta_k; p_0) \geq J(\pi_*; p_0) - 2\frac{\lambda_k}{1 - \gamma} \max_s {\frac{d_{\pi_*, \gamma, p_0}(s)}{p_0'(s)}}.
\end{align}
holds with at least probability
\begin{align}
  1 - C_t,
\end{align}
which completes the proof.

\end{proof} 

\section{Proofs of Section~\ref{sec sac}}
\subsection{Proof of Proposition~\ref{prop critic convergence sac}}
\label{sec proof prop critic convergence sac}
\propcriticconvergencesac*
\begin{proof}
  The proof is similar to the proof of Proposition~\ref{prop critic convergence}.
  To start with,
  define
  \begin{align}
    \fY &\doteq \qty{(s, a, s') \mid s \in \fS, a \in \fA, s' \in \fS, p(s'|s, a) > 0}, \\
    Y_t &\doteq (S_t, A_t, S_{t+1}), \\
    P_{\zeta}((s_1, a_1, s_1'), (s_2, a_2, s_2')) &\doteq \begin{cases}
      0 & s_1' \neq s_2 \\
      \mu_\theta(a_2|s_2) p(s_2'|s_2, a_2) & s_1' = s_2.
    \end{cases}
  \end{align} 
  According to the action selection rule for $A_{t}$ specified in Algorithm~\ref{alg sac},
  we have
  \begin{align}
    \Pr(Y_{t+1} = y) = P_{\zeta_{t+1}}(Y_t, y),
  \end{align}
  Assumption~\ref{assu makovian} is then fulfilled.

  Assumption~\ref{assu uniform ergodicity} is immediately implied by Assumption~\ref{assu mu uniform ergodicity}.
  In particular, 
  for any $\zeta$,
  the invariant distribution of 
  the chain induced by $P_\zeta$ is $d_{\mu_\theta}(s)\mu_\theta(a|s)p(s'|s, a)$.

  To verify Assumption~\ref{assu uniform contraction},
  first notice that
  \begin{align}
    \bar F_\zeta(q) =& \sum_{s, a, s'} d_{\mu_\theta}(s) \mu_\theta(a|s)p(s'|s, a) F_\zeta(q, s, a, s') \\
    &= D_{\mu_\theta}\left(r + \gamma P_{\pi_\theta} \left(q - \eta \log \pi_\theta\right)  -  q\right) + q \\
    &= (I - D_{\mu_\theta}(I - \gamma P_{\pi_\theta})) q + D_{\mu_\theta}(r - \eta \gamma P_{\pi_\theta} \log \pi_\theta),
  \end{align}
  where $\pi_\theta$ denotes a vector in $\R^{\nsa}$ whose $(s, a)$-indexed element is $\pi_\theta(a|s)$ and $\log \pi_\theta$ is the elementwise logarithm of $\pi_\theta$.
  Then, we have
  \begin{align}
    \bar F_{\zeta}(q) - \bar F_{\zeta}(q') = A_{\theta} (q - q'),
  \end{align}
  where $A_{\theta}$ is defined as in \eqref{eq definition of a theta}:
  \begin{align*}
      A_\theta \doteq I - D_{\mu_\theta}(I - \gamma P_{\pi_\theta}).
  \end{align*}
  According to the proof of Lemma~\ref{lem uniform contraction opac},
  there exist a $\kappa \in (0, 1)$ and an $\ell_p$ norm such that $\forall x$,
  \begin{align}
    \norm{A_\theta x}_p \leq \kappa \norm{x}_p,
  \end{align}
  implying
  \begin{align}
    \norm{\bar F_{\zeta}(q) - \bar F_{\zeta}(q')}_p \leq \kappa \norm{q - q'}_p.
  \end{align}
  Further,
  \begin{align}
    \bar F_\zeta(q) &=q \\
    \iff r + \gamma P_{\pi_\theta}(q - \eta \log \pi_\theta) - q &= 0 \\
    \iff q &= \tilde q_{\pi_\theta, \eta} \qq{(Lemma 1 of \citet{haarnoja2018soft}),}
  \end{align}
  which completes the verification of Assumption~\ref{assu uniform contraction}.

  We now verify Assumption~\ref{assu regularization}.
  In particular,
  the norm $\norm{\cdot}_c$ in Section~\ref{sec sa} is now realized as the $\ell_p$ norm above.

  To verify Assumption~\ref{assu regularization} (i), for any $y = (s_0, a_0, s_1)$,
  we have
  \begin{align}
    \left(F_\zeta(q, y) - F_\zeta(q', y)\right)(s, a) = \begin{cases}
      q(s, a) - q'(s, a), & (s, a) \neq (s_0, a_0) \\
      \gamma \sum_{a_1}\pi_\theta(a_1| s_1) \left(q(s_1, a_1) - q'(s_1, a_1)\right), &(s, a) = (s_0, a_0)
    \end{cases}.
  \end{align}
  Hence
  \begin{align}
    \norm{F_\zeta(q, y) - F_\zeta(q', y)}_\infty \leq \norm{q - q'}_\infty,
  \end{align}
  implying
  \begin{align}
    \norm{F_\zeta(q, y) - F_\zeta(q', y)}_p \leq \frac{u_{\infty, p}}{l_{\infty, p}} \norm{q - q'}_p.
  \end{align}
  Assumption~\ref{assu regularization} (i) is then fulfilled.

  To verify Assumption~\ref{assu regularization} (ii),
  for any $y = (s_0, a_0, s_1)$,
  we have
  \begin{align}
    &\left(F_{\zeta_t}(q, y) - F_{\zeta_k}(q, y)\right)(s, a) \\
    =& \begin{cases}
      0, & (s, a) \neq (s_0, a_0) \\
      \gamma \sum_{a_1} \left(\pi_{\theta_t}(a_1| s_1) - \pi_{\theta_k}(a_1|s_1) \right) q(s_1, a_1) + \lambda_t \ent{\pi_{\theta_t}(\cdot|s_1)} - \lambda_k \ent{\pi_{\theta_k}(\cdot|s_1)}, &(s, a) = (s_0, a_0)
    \end{cases}.
  \end{align}
  Since 
  \begin{align}
    &\abs{\lambda_t \ent{\pi_{\theta_t}(\cdot|s_1)} - \lambda_k \ent{\pi_{\theta_k}(\cdot|s_1)}} \\
    \leq & \abs{\lambda_t - \lambda_k} \ent{\pi_{\theta_t}(\cdot | s_1)} + \lambda_k \abs{\ent{\pi_{\theta_t}(\cdot | s_1)} - \ent{\pi_{\theta_k}(\cdot | s_1)}} \\
    \leq & \abs{\lambda_t - \lambda_k} \log \na + \lambda_k \left(\log \na + e^{-1}\right) \norm{\theta_t - \theta_k} \qq{(Lemma \ref{lem softmax policy gradient})} \\
    \leq & \left(\log \na + \lambda \log \na + \lambda e^{-1}\right) \norm{\zeta_t - \zeta_k},
  \end{align}
  we have
  \begin{align}
    &\norm{F_{\zeta_k}(q, y) - F_{\zeta_t}(q, y)}_\infty \\ 
    \leq& \gamma \na L_\pi \norm{\theta_t - \theta_k}\norm{q}_\infty + \left(\log \na + \lambda \log \na + \lambda e^{-1}\right) \norm{\zeta_t - \zeta_k} \\
    \leq& \gamma \na L_\pi \norm{\zeta_t - \zeta_k}\norm{q}_\infty + \left(\log \na + \lambda \log \na + \lambda e^{-1}\right) \norm{\zeta_t - \zeta_k} \\
    \leq& \frac{\gamma \na L_\pi}{l_{2, p} l_{\infty, p}} \norm{\zeta_t - \zeta_k}_p \norm{q}_p + \frac{(\log \na + \lambda \log \na + \lambda e^{-1})}{l_{2, p}} \norm{\zeta_t - \zeta_k}_p
  \end{align}
  Assumption~\ref{assu regularization} (ii) is then fulfilled. 

  To verify Assumption~\ref{assu regularization} (iii), 
  for any $y = (s_0, a_0, s_1)$,
  we have
  \begin{align}
    \left(F_{\zeta_t}(0, y)\right)(s, a) = \begin{cases}
      0, & (s, a) \neq (s_0, a_0) \\
      r(s_0, a_0) + \gamma \lambda_t \ent{\pi_{\theta_t}(\cdot | s_1)}, &(s, a) = (s_0, a_0)
    \end{cases}.
  \end{align}
  Then 
  \begin{align}
    \norm{F_{\zeta_t}(0, y)}_p \leq u_{\infty, p} \norm{F_\theta(0, y)}_\infty \leq  u_{\infty, p}\left(r_{max} + \gamma \lambda \log \na \right).
  \end{align}
  Assumption~\ref{assu regularization} (iii) is then fulfilled.

  To verify Assumption~\ref{assu regularization} (iv), 
  we have
  \begin{align}
    \bar F_{\zeta_t}(q) - \bar F_{\zeta_k}(q)
    = \bar F_{\theta_t}(q) - \bar F_{\theta_k}(q) - \gamma \lambda_t D_{\mu_{\theta_t}} P_{\pi_{\theta_t}} \log \pi_{\theta_t} + \gamma \lambda_k D_{\mu_{\theta_k}} P_{\pi_{\theta_k}} \log \pi_{\theta_k},
  \end{align}
  where $\bar F_\theta$ is defined in \eqref{eq definition of bar f theta}.
  In the proof of Proposition~\ref{prop critic convergence},
  we already show that
  there exist constants $C_1$ and $C_2$ such that
  \begin{align}
    \norm{\bar F_{\theta_t}(q) - \bar F_{\theta_k}(q)}_p \leq C_1 \norm{\theta_t - \theta_k}_p \left(\norm{q}_p + C_2\right) \leq C_1 \norm{\zeta_t - \zeta_k}_p \left(\norm{q}_p + C_2\right).
  \end{align}
  We now bound the remaining parts $- \gamma \lambda_t D_{\mu_{\theta_t}} P_{\pi_{\theta_t}} \log \pi_{\theta_t} + \gamma \lambda_k D_{\mu_{\theta_k}} P_{\pi_{\theta_k}} \log \pi_{\theta_k}$.
  First, notice that
  \begin{align}
    \left(P_{\pi_\theta} \log \pi_\theta\right)(s, a) =& \sum_{s'} p(s'|s, a) \sum_{a'} \pi_\theta(a'|s') \log \pi_\theta(a'|s') \\
    =& -\sum_{s'}p(s'|s,a)\ent{\pi_{\theta}(\cdot |s')}.
  \end{align}
  It is easy to see $\ent{\pi_\theta(\cdot|s')}$ is Lipschitz continuous in $\theta$ (Lemma~\ref{lem softmax policy gradient}) 
  and is bounded by $\log \na$.
  We, therefore, conclude that $P_{\pi_\theta} \log \pi_\theta$ is Lipschitz continuous in $\theta$ and is bounded from above.
  Since $D_{\mu_\theta}$ is also Lipschitz continuous in $\theta$ (Lemma~\ref{lem continuity of ergodic distribution}) and is bounded from above,
  Lemma~\ref{lem product of lipschitz functions} asserts that there exists constants $C_3$ and $C_4$ such that
  \begin{align}
    \norm{D_{\mu_\theta} P_{\pi_\theta} \log \pi_\theta} &\leq C_3, \\
    \norm{D_{\mu_\theta} P_{\pi_\theta} \log \pi_\theta - D_{\mu_{\theta'}} P_{\pi_{\theta'}} \log \pi_{\theta'}} &\leq C_4 \norm{\theta - \theta'},
  \end{align}
  implying
  \begin{align}
    &\norm{\lambda_k D_{\mu_{\theta_k}} P_{\pi_{\theta_k}} \log \pi_{\theta_k} - \lambda_t D_{\mu_{\theta_t}} P_{\pi_{\theta_t}} \log \pi_{\theta_t}} \\
    \leq & \norm{\lambda_k - \lambda_t} \norm{D_{\mu_{\theta_k}} P_{\pi_{\theta_k}} \log \pi_{\theta_k}} + \lambda_t \norm{D_{\mu_{\theta_k}} P_{\pi_{\theta_k}} \log \pi_{\theta_k} - D_{\mu_{\theta_t}} P_{\pi_{\theta_t}} \log \pi_{\theta_t}} \\
    \leq &C_3 \norm{\lambda_k - \lambda_t} + \lambda C_4 \norm{\theta_t - \theta_k} \\
    \leq &(C_3 + \lambda C_4) \norm{\zeta_t - \zeta_k},
  \end{align}
  which completes the verification of Assumption~\ref{assu regularization} (iv).

  To verify Assumptions~\ref{assu regularization} (v), 
  it suffices to show that
  \begin{align}
    \norm{\tilde q_{\pi_{\theta_t}, \lambda_t} - \tilde q_{\pi_{\theta_k}, \lambda_k}}_p \leq C_5 \norm{\zeta_t - \zeta_k}_p
  \end{align}
  holds for some positive constant $C_5$.
  According to \eqref{eq definition of soft q},
  it suffices to show that for some positive constant $C_6$,
  \begin{align}
    \norm{\tilde v_{\pi_{\theta_t}, \lambda_t} - \tilde v_{\pi_{\theta_k}, \lambda_k}}_p \leq C_6 \norm{\zeta_t - \zeta_k}_p.
  \end{align}
  Recall by definition
  \begin{align}
    \label{eq tmp 12}
    \tilde v_{\pi_{\theta_t}, \lambda_t}(s) =& v_{\pi_{\theta_t}}(s) + \lambda_t \underbrace{\E_{\pi_{\theta_t}}\left[\sum_{i=0}^\infty \gamma^{i}  \ent{\pi_{\theta_t}(\cdot | S_{t+i})}   \mid S_t = s\right]}_{\mathbb{H}_{\theta_t}(s)}.
  \end{align}
  Clearly,
  \begin{align}
    \abs{\mathbb{H}_{\theta}(s)} \leq \frac{\log \na}{1 - \gamma}.
  \end{align}
  We now show that $\mathbb{H}_{\theta}(s)$ is Lipschitz continuous in $\theta$.
  Let $p_{0, s}$ denote the distribution on $\fS$ such that all its mass concentrates on the state $s$, i.e.,
  $p_{0, s}(s) = 1$.
  We can then express $\mathbb{H}_\theta(s)$ as
  \begin{align}
    \mathbb{H}_{\theta}(s) = \frac{1}{1-\gamma} \sum_{s} d_{\pi_\theta, \gamma, p_{0, s}}(s) \ent{\pi_\theta(\cdot |s)}.
  \end{align}
  It is easy to see that
  \begin{align}
    &d_{\pi_\theta, \gamma, p_{0, s}}(s'') \\
    =& (1 - \gamma) \sum_{t=0}^\infty \gamma^t \Pr(S_t=s'' | S_0 \sim p_{0, s}) \\
    =& (1 - \gamma) p_{0, s}(s'') + (1 - \gamma) \sum_{t=1}^{\infty} \gamma^t \Pr(S_t = s'' | S_0 \sim p_{0, s}) \\
    =& (1 - \gamma) p_{0, s}(s'') + (1 - \gamma) \sum_{t=0}^{\infty} \gamma^{t+1} \Pr(S_{t+1} = s'' | S_0 \sim p_{0, s}) \\
    =& (1 - \gamma) p_{0, s}(s'') + \gamma (1 - \gamma) \sum_{t=0}^{\infty} \gamma^t \sum_{s'} \Pr(S_t = s' | S_0 \sim p_{0, s}) \Pr(S_{t+1} = s'' | S_t = s') \\
    =& (1 - \gamma) p_{0, s}(s'') + \gamma \sum_{s'}\Pr(S_{t+1} = s'' | S_t = s') d_{\pi_\theta, \gamma, p_{0,s}}(s').
  \end{align}
  In a matrix form,
  we have
  \begin{align}
    d_{\pi_\theta, \gamma, p_{0, s}} =& (1 - \gamma) p_{0, s} + \gamma P_{\pi_\theta}^\top d_{\pi_\theta, \gamma, p_{0, s}} \\
    \implies d_{\pi_\theta, \gamma, p_{0, s}} =& (1 - \gamma) (I - \gamma P_{\pi_\theta}^\top)^{-1} p_{0, s},
  \end{align}
  where we have abused the notation a bit to use $P_{\pi_\theta}$ to also denote the state transition matrix under the policy $\pi_\theta$.
  Similar to \eqref{eq lipschitz continuity of q},
  we conclude that $d_{\pi_\theta, \gamma, p_{0, s}}$ is Lipschitz continuous in $\theta$.
  Lemma~\ref{lem softmax policy gradient} confirms that 
  $\ent{\pi_\theta(\cdot |s)}$ is Lipschitz continuous in $\theta$.
  Hence Lemma~\ref{lem product of lipschitz functions} asserts that $\mathbb{H}_{\theta}(s)$ is Lipschitz continuous in $\theta$, 
  i.e.,
  there exists a positive constant such that
  \begin{align}
    \label{eq tmp 9}
    \abs{\mathbb{H}_\theta(s) - \mathbb{H}_{\theta'}(s)} \leq C_7 \norm{\theta - \theta'}.
  \end{align}
  Similar to \eqref{eq lipschitz continuity of q},
  we can also show that there exists a constant $C_8$ such that
  \begin{align}
    \label{eq tmp 10}
    \abs{v_{\pi_\theta}(s) - v_{\pi_{\theta'}}(s)} \leq C_8 \norm{\theta - \theta'}.
  \end{align}
  We, therefore, have
  \begin{align}
    &\abs{{\tilde v_{\pi_{\theta_t}, \lambda_t}(s) - \tilde v_{\pi_{\theta_k}, \lambda_k}}(s)} \\
    \leq& \abs{v_{\pi_{\theta_t}}(s) - v_{\pi_{\theta_k}}(s)} + \abs{\lambda_t - \lambda_k} \abs{\mathbb{H}_{\theta_t}(s)} + \lambda_k \abs{\mathbb{H}_{\theta_t}(s) - \mathbb{H}_{\theta_k}(s)} \\
    \leq& C_8 \norm{\theta_t - \theta_k} + \abs{\lambda_t - \lambda_k} \frac{\log \na}{1 - \gamma} + \lambda C_7 \norm{\theta_t - \theta_k} \\
    \leq& \left(C_8 + \frac{\log \na}{1 - \gamma} + \lambda C_7\right) \norm{\zeta_t - \zeta_k},
  \end{align}
  which completes the verification of Assumption~\ref{assu regularization} (v).

  For Assumption~\ref{assu regularization} (vi),
  first notice that
  the soft action value function $\tilde q_{\pi, \eta}$ can be regarded as the normal action value function $q_\pi$ w.r.t. to the reward
  \begin{align}
    r(s, a) + \eta \sum_{s'}p(s'|s,a) \ent{\pi(\cdot|s')}
  \end{align}
  Then it is easy to see 
  \begin{align}
    \label{eq soft q bound}
    \abs{\tilde q_{\pi_{\theta_t, \lambda_t}}(s, a)} \leq U_{\tilde J} \doteq \frac{r_{max} + \lambda \log \na}{1 - \gamma},
  \end{align}
  which completes the verification of Assumption~\ref{assu regularization} (vi).

  For Assumption~\ref{assu regularization} (vii),
  we have
  \begin{align}
    \abs{P_{\zeta_t}(y, y') - P_{\zeta_k}(y, y')} \leq C_9 \norm{\theta_t - \theta_k} \leq  C_9 \norm{\zeta_t - \zeta_k},
  \end{align}
  where the existence of the positive constant $C_9$ is ensured by Assumption~\ref{assu lipschitz mu}.

  Assumption~\ref{assu mds} is automatically fulfilled since in our setting we have $\epsilon_t \equiv 0$.

  Assumption~\ref{assu twotimescale} is automatically implied by Assumption~\ref{assu three timescales} except for \eqref{eq konda actor update}.
  Ssimilar to~\eqref{eq boundedness q},
  we can show that
  \begin{align}
    \norm{q_t}_\infty \leq \tilde C_q \doteq \max\qty{\norm{q_0}_\infty, \frac{r_{max} + \lambda \log \na}{1 - \gamma}}.
  \end{align}
  According to the updates of $\qty{\theta_t}$ in Algorithm~\ref{alg sac},
  we have
  \begin{align}
    \label{eq ltheta sac detailed}
    &\norm{\theta_{t+1} - \theta_t} \\
    =& \beta_t \norm{\sum_a \pi_{\theta_t}(a|S_t)\nabla_\theta \log \pi_{\theta_t}(a | S_t) \Big( q_t(S_t, a) - \lambda_t \log \pi_{\theta_t}(a | S_t)\Big)}  \\
    =& \beta_t \norm{\sum_a \nabla \pi_{\theta_t}(a|S_t) \Big( q_t(S_t, a) - \lambda_t \log \pi_{\theta_t}(a | S_t)\Big)}  \\
    \leq& \beta_t \norm{\sum_a \nabla \pi_{\theta_t}(a|S_t) q_t(S_t, a)} + \beta_t \lambda_t \norm{\sum_a \nabla \pi_{\theta_t}(a|S_t)\log \pi_{\theta_t}(a | S_t)}  \\
    \leq& \beta_t \na \sqrt{\na} \tilde C_q + \beta_t \lambda_t \norm{\sum_a \nabla \pi_{\theta_t}(a|S_t)\log \pi_{\theta_t}(a | S_t)} \qq{(Lemma~\ref{lem softmax policy gradient})}  \\
    =& \beta_t \na \sqrt{\na} \tilde C_q  + \beta_t \lambda_t \norm{\sum_a \nabla \pi_{\theta_t}(a|S_t)\log \pi_{\theta_t}(a | S_t) + \pi_{\theta_t}(a|S_t) \nabla \log \pi_{\theta_t}(a|S_t)} \\
    =& \beta_t \na \sqrt{\na} \tilde C_q + \beta_t \lambda_t \norm{\nabla \ent{\pi_{\theta_t}(\cdot|S_t)}} \\
    \leq& \beta_t \na \sqrt{\na} \tilde C_q  + \beta_t \lambda_t \left(\log \na + e^{-1}\right) \qq{(Lemma~\ref{lem softmax policy gradient})} \\
    \label{eq ltheta sac}
    \leq& \beta_t \underbrace{\left( \na \sqrt{\na} \tilde C_q  + \lambda \log \na + \lambda e^{-1}\right)}_{L_\theta}.
  \end{align}
  Further,
  \begin{align}
    &\abs{\lambda_{t+1} - \lambda_t} \\
    =& \lambda\left(\frac{1}{(t+t_0)^{\epsilon_\lambda}} - \frac{1}{(t+t_0+1)^{\epsilon_\lambda}}\right) \\
    =& \lambda\frac{ (t+t_0+1)^{\epsilon_\lambda} - (t+t_0)^{\epsilon_\lambda} }{(t+t_0)^{\epsilon_\lambda} (t+t_0+1)^{\epsilon_\lambda}} \\
    =& \lambda\frac{ (t+t_0+1)^{\epsilon_\lambda} (t+t_0)^{1-\epsilon_\lambda} - (t+t_0) }{(t+t_0) (t+t_0+1)^{\epsilon_\lambda}} \\
    \leq & \lambda\frac{ (t+t_0+1)^{\epsilon_\lambda} (t+t_0+1)^{1-\epsilon_\lambda} - (t+t_0) }{(t+t_0) (t+t_0)^{\epsilon_\lambda}} \\
    = & \frac{ \lambda }{(t+t_0)^{1 + \epsilon_\lambda} } \\
    = & \beta_t \frac{ \lambda (t+t_0)^{\epsilon_\beta} }{ \beta (t+t_0)^{1 + \epsilon_\lambda} } \\
    \leq &\beta_t \frac{\lambda}{\beta}.
  \end{align}
  We, therefore, conclude that
  that there exists a constant $L_\zeta$ 
  such that 
  \begin{align}
    \label{eq lzeta sac}
    \norm{\zeta_{t+1} - \zeta_t}_p \leq \beta_t L_\zeta,
  \end{align}
  which completes the verification of Assumption~\ref{assu twotimescale}.

  With Assumptions \ref{assu makovian} - \ref{assu twotimescale} satisfied,
  invoking Theorem~\ref{thm sa convergence} completes the proof. 
\end{proof}

\subsection{Proof of Theorem~\ref{thm convergence actor sac}}
\thmconvergenceactorsac*
\label{sec proof thm convergence actor sac}
\begin{proof}
  In this proof,
  we use as shorthand
  \begin{align}
    \label{eq actor convergence shorthard sac}
    J(\theta) &\doteq J(\pi_\theta; p_0'), \\
    \tilde J_\eta(\theta) &\doteq \tilde J_\eta(\pi_\theta; p_0'), \\
    d_{\pi, \gamma}(s) &\doteq d_{\pi, \gamma, p_0'}(s),
  \end{align}
  i.e.,
  we work on the initial distribution $p_0'$ (instead of $p_0$).
  Recall the entropy regularized discounted total rewards is defined as 
  \begin{align}
    \tilde J_\eta(\pi_\theta;p_0') = J(\pi_\theta;p_0') + \eta \underbrace{\frac{1}{1-\gamma} \sum_{s} d_{\pi_\theta, \gamma, p_0'}(s) \ent{\pi_\theta(\cdot | s)}}_{\mH(\pi_\theta)}.
  \end{align}
  According to Lemma 7 of \citet{mei2020global}, $J(\theta)$ is $L_J$-smooth for some positive constant $L_J$ w.r.t $\norm{\cdot}$.
  According to Lemma 14 of \citet{mei2020global}, $\mH(\pi_\theta)$ is $L_H$-smooth for some positive constant $L_H$ w.r.t. $\norm{\cdot}$.
  Hence $\tilde J_\eta(\theta)$ is $(L_J + \eta L_H)$-smooth.
  With $\eta = \lambda_t$,
  Lemma~\ref{lem smooth definition} then implies
  \begin{align}
    \tilde J_{\lambda_t}(\theta_{t+1}) \geq& \tilde J_{\lambda_t}(\theta_t) + \indot{\nabla \tilde J_{\lambda_t}(\theta_t)}{\theta_{t+1} - \theta_t} - (L_J + \lambda_t L_H) \norm{\theta_{t+1}-\theta_t}^2 \\
    \geq& \tilde J_{\lambda_t}(\theta_t) + \underbrace{\indot{\nabla \tilde J_{\lambda_t}(\theta_t)}{\theta_{t+1} - \theta_t}}_{\tilde M_1} - \tilde L_J' \underbrace{\norm{\theta_{t+1}-\theta_t}^2}_{\tilde M_2},
  \end{align}
  where $\tilde L_J' \doteq L_J + \lambda L_H$.
  Using \eqref{eq ltheta sac} to bound $\tilde M_2$ yields
  \begin{align}
    \tilde M_2 \leq \frac{1}{l_{2,p}} \beta_t^2 L_\theta^2.
  \end{align}
  To bound $\tilde M_1$, 
  we reuse the $Y_t$ and $\tilde Y_t$ defined in \eqref{eq actor yt} and \eqref{eq actor tilde yt}.
  For any $s$, we define
  \begin{align}
    \label{eq actor helper function sac}
    \Lambda_1(\theta, s, \eta) &\doteq \sum_a \pi_\theta(s, a)  \nabla \log  \pi_\theta(a|s) \left(\tilde q_{\pi_\theta, \eta}(s, a) - \eta \log \pi_\theta(a | s) \right), \\
    \bar \Lambda_1(\theta, \eta) &\doteq \sum_{s} d_{\mu_\theta}(s) \Lambda(\theta, s, \eta).
  \end{align}
  Then we have
  \begin{align}
    &\tilde M_1 \\
    = &\indot{\nabla \tilde J_{\lambda_t}(\theta_t)}{\theta_{t+1} - \theta_t} \\
    =&\beta_t \indot{\nabla \tilde J_{\lambda_t}(\theta_t)}{ \sum_a \pi_{\theta_t}(a|S_t) \nabla \log \pi_{\theta_t}(a | S_t) \Big(q_t(S_t, a) - \lambda_t \log \pi_{\theta_t}(a |S_t) \Big)} \\
    =&\beta_t \underbrace{\indot{\nabla \tilde J_{\lambda_t}(\theta_t)}{\bar \Lambda_1(\theta_t, \lambda_t)}}_{\tilde M_{11}} \\
    &+ \beta_t \underbrace{\indot{\nabla \tilde J_{\lambda_t}(\theta_t)}{\sum_a \nabla \pi_{\theta_t}(a|S_t) \Big(\tilde q_{\pi_{\theta_t}, \lambda_t}(S_t, a) - \lambda_t \log \pi_{\theta_t}(a |S_t) \Big)  - \bar \Lambda_1(\theta_t, \lambda_t)}}_{\tilde M_{12}} \\
    &+ \beta_t \underbrace{\indot{\nabla \tilde J_{\lambda_t}(\theta_t)}{\sum_a \nabla \pi_{\theta_t}(a | S_t) \left(q_t(S_t, a) -  \tilde q_{\pi_{\theta_t}, \lambda_t}(S_t, a) \right)}}_{\tilde M_{13}}.
  \end{align}
  To bound $\tilde M_{12}$,
  define 
  \begin{align}
    \label{eq actor helper function 2 sac}
    \Lambda_1'(\theta, s, \eta) \doteq \indot{\nabla \tilde J_{\eta}(\theta)}{ \Lambda_1(\theta, s, \eta) - \bar \Lambda_1(\theta, \eta)}.
  \end{align}
  We then decompose $\tilde M_{12}$ as
  \begin{align}
    &\tilde M_{12} = \Lambda_1'(\theta_t, S_t, \lambda_t)  \\
    =& \underbrace{\Lambda_1'(\theta_t, S_t, \lambda_t) - \Lambda_1'(\theta_{t-\tau_{\beta_t}}, S_t, \lambda_t)}_{\tilde M_{121}} + \underbrace{\Lambda_1'(\theta_{t-\tau_{\beta_t}}, S_t, \lambda_t) - \Lambda_1'(\theta_{t-\tau_{\beta_t}}, \tilde S_t, \lambda_t)}_{\tilde M_{122}} \\
    &+ \underbrace{\Lambda_1'(\theta_{t-\tau_{\beta_t}}, \tilde S_t, \lambda_t)}_{\tilde M_{123}},
  \end{align}
  where we recall that $\tilde S_t$ is defined as part of $\tilde Y_t$ in \eqref{eq actor tilde yt}.
  Let us proceed to bounding each term defined above.

\begin{restatable}{lemma}{lemboundofmoneonesac}
  \label{lem bound of m11 sac}
(Bound of $\tilde M_{11}$)
There exists a constant $\chi_{11} > 0$ such that,
\begin{align}
  \tilde M_{11} \geq \chi_{11} \norm{\nabla \tilde J_{\lambda_t}(\theta_t)}^2.
\end{align}
\end{restatable}
\noindent
The proof of Lemma~\ref{lem bound of m11 sac} is provided in Section~\ref{sec proof lem bound of m11 sac}.

\begin{restatable}{lemma}{lemboundofmonetwoonesac}
  \label{lem bound of m121 sac}
  (Bound of $\tilde M_{121}$)
  There exist constants $L_{\Lambda_1'}^* > 0$ such that
  \begin{align}
    \norm{\tilde M_{121}} \leq L_{\Lambda_1'}^* L_\theta \beta_{t-\tau_{\beta_t}, t-1},
  \end{align}
  where $L_\theta$ is defined in \eqref{eq ltheta sac}.
\end{restatable}
\noindent
The proof of Lemma~\ref{lem bound of m121 sac} is provided in Section~\ref{sec proof lem bound of m121 sac}

\begin{restatable}{lemma}{lemboundofmonetwotwosac}
  \label{lem bound of m122 sac}
  (Bound of $\tilde M_{122}$)
  There exists a constant $U_{\Lambda_1'}^* > 0$ such that
  \begin{align}
      \norm{\E\left[\tilde M_{122}\right]} \leq U_{\Lambda'}^* \ns \na L_\mu L_\theta \sum_{j=t-\tau_{\beta_t}}^{t-1} \beta_{t-\tau_{\beta_t}, j}.
  \end{align}
\end{restatable}
\noindent
The proof of Lemma~\ref{lem bound of m122 sac} is identical to the proof of Lemma~\ref{lem bound of m122} 
in Section~\ref{sec proof lem bound of m122}
up to change of notations 
and is thus omitted.

\begin{restatable}{lemma}{lemboundofmonetwothreesac}
  \label{lem bound of m123 sac}
  (Bound of $\tilde M_{123}$)
  \begin{align}
    \norm{\E\left[\tilde M_{123}\right]} \leq U_{\Lambda_1'}^* \beta_t.
  \end{align}
\end{restatable}
\noindent
The proof of Lemma~\ref{lem bound of m123 sac} is identical to the proof of Lemma~\ref{lem bound of m123} in Section~\ref{sec proof lem bound of m123} up to change of notations and is thus omitted.

\begin{restatable}{lemma}{lemboundofmonethreesac}
  \label{lem bound of m13 sac}
  (Bound of $\tilde M_{13}$)
  The exists a constant $\chi_{13} > 0$ such that
  \begin{align}
    \norm{\E\left[\tilde M_{13}\right]}\leq& \chi_{13} \sqrt{\E\left[\norm{q_t - \tilde q_{\pi_{\theta_t, \lambda_t}}}_\infty^2 \right]} \sqrt{\E\left[\norm{\nabla \tilde J_{\lambda_t}(\theta_t)}^2\right]}.
  \end{align}
\end{restatable}
\noindent
The proof of Lemma~\ref{lem bound of m13 sac} is identical to the proof Lemma~\ref{lem bound of m13} in Section~\ref{sec proof lem bound of m13} 
up to change of notations and is thus omitted.

Now using exactly the same routine as the proof of Theorem~\ref{thm:optimality-AC} in Section~\ref{sec proof thm:optimality-AC},
we obtain that there exists some positive constants $\chi_3, \chi_4$ and $\chi_5$ such that
\begin{align}
  \sum_{k=\ceil{\frac{t}{2}}}^t \E\left[\norm{\nabla \tilde J_{\lambda_k}(\theta_k)}^2\right] \leq& \chi_3 \sum_{k=\ceil{\frac{t}{2}}}^t \frac{1}{\beta_k} \left(\E\left[\tilde J_{\lambda_k}(\theta_{k+1})\right] - \E\left[\tilde J_{\lambda_k}(\theta_k)\right]\right) \\
  &+ \chi_4 \sum_{k=\ceil{\frac{t}{2}}}^t \sqrt{\E\left[\norm{q_k - \tilde q_{\pi_{\theta_k}, \lambda_k}}_p^2 \right]} \sqrt{\E\left[\norm{\nabla \tilde J_{\lambda_k}(\theta_k)}^2\right]} \\
  &+ \chi_5 \sum_{k=\ceil{\frac{t}{2}}}^t \frac{\log^2(k+t_0)}{(k+t_0)^{\epsilon_\beta}},
\end{align}
where the $\ell_p$ norm is defined in Proposition~\ref{prop critic convergence sac}.
To continue mimicing the proof of Theorem~\ref{thm:optimality-AC},
we need to establish counterparts of Lemmas~\ref{lem bound j lambda} and~\ref{lem bound j lambda diff} to bound the first summation in the RHS of the above inequality.
The counterpart of Lemma~\ref{lem bound j lambda} is trivial since 
by the definition of $\tilde J_\eta(\theta)$ we have $\forall t, \theta$
\begin{align}
  \abs{\tilde J_{\lambda_t}(\theta)} \leq U_{\tilde J},
\end{align}
where $U_{\tilde J}$ is defined in \eqref{eq soft q bound}.
This simplification is because that $\ent{\pi(\cdot|s)}$ is always bounded by $\log \na$ but $\kl{\fU_{\fA}}{\pi(\cdot|s)}$ can be unbounded.
Then we have
\begin{restatable}{lemma}{lemboundjlambdadiffsac}
  \label{lem bound j lambda diff sac}
  \begin{align}
  \sum_{k=\ceil{\frac{t}{2}}}^t \frac{1}{\beta_k}\left(\tilde J_{\lambda_k}(\theta_{k+1}) - \tilde J_{\lambda_k}(\theta_k) \right) \leq \frac{3\lambda\beta \log \na}{1-\gamma} + \frac{2U_{\tilde J}}{\beta} (t+t_0)^{\epsilon_\beta}
  \end{align}
\end{restatable}
\noindent The proof of Lemma~\ref{lem bound j lambda diff sac} is provided in Section~\ref{sec proof lem bound j lambda diff sac}.
Using the same routine as the proof of Theorem~\ref{thm:optimality-AC} yields
\begin{align}
  \label{eq actor convergence stationary points sac}
  \frac{\sum_{k=\ceil{\frac{t}{2}}}^t \E\left[\norm{ \nabla \tilde J_{\lambda_k}(\theta_k)}^2\right] }{t-\ceil{\frac{t}{2}}+1} = \underbrace{\fO\left(\frac{1}{t^{1-\epsilon_\beta}} + \frac{\log^2t}{t^{\epsilon_\beta}} + \frac{1}{t^{\epsilon_q}}\right)}_{\tilde C_t}.
\end{align}
We now analyze the above equality from a probabilistic perspective.
Consider a positive non-increasing sequence $\qty{\delta_t}$ to be tuned.
Fix any $t > 0$.
Then select a $k$ uniformly randomly from $\qty{\ceil{\frac{t}{2}}, \ceil{\frac{t}{2}} + 1, \dots, t-1, t}$.
Now the random variable $\norm{\nabla \tilde J_{\lambda_k}(\theta_k)}$ has randomness from both the random selection of $k$ and the learning of $\theta_k$.
Using Markov's inequality
yields
\begin{align}
  \Pr(\norm{\nabla \tilde J_{\lambda_k}(\theta_k)}^2 \leq \delta_t)
  \geq& 1 - \frac{1}{\delta_t} \E\left[\norm{\nabla \tilde J_{\lambda_k}(\theta_k)}^2\right] \\
  =& 1 - \frac{1}{\delta_t} \E\left[ \E\left[\norm{\nabla \tilde J_{\lambda_k}(\theta_k)}^2 \mid k \right] \right]\\
  =& 1 - \frac{1}{\delta_t} \frac{\sum_{i=\ceil{\frac{t}{2}}}^t \E\left[\norm{ \nabla \tilde J_{\lambda_k}(\theta_k)}^2 \mid k=i \right] }{t-\ceil{\frac{t}{2}}+1} \\
  \geq & 1 - \frac{1}{\delta_t} \tilde C_t. \qq{(Using \eqref{eq actor convergence stationary points sac})}
\end{align}
Since $\delta_k \geq \delta_t$,
we have
\begin{align}
  \norm{\nabla \tilde J_{\lambda_k}(\theta_k)}^2 \leq \delta_t \implies \norm{\nabla \tilde J_{\lambda_k}(\theta_k)}^2 \leq \delta_k.
\end{align}
Consequently,
\begin{align}
  \Pr(\norm{\nabla \tilde J_{\lambda_k}(\theta_k)}^2 \leq \delta_k) \geq \Pr(\norm{\nabla \tilde J_{\lambda_k}(\theta_k)}^2 \leq \delta_t) \geq 1 - \frac{1}{\delta_t} \tilde C_t.
\end{align}
Letting  
\begin{align}
  \delta_t \doteq \frac{1}{t^{\epsilon_0}}
\end{align}
then completes the proof.
\end{proof}

\subsection{Proof of Corollary~\ref{cor optimality of actor sac}}
\label{sec proof cor optimality of actor sac}
\coroptimalityofactorsac*
\begin{proof}
  Fix any state distribution $p_0'$ satisfying $\forall s, p_0'(s) > 0$. 
  Then, from the proof of Theorem~\ref{thm convergence actor sac} in Section~\ref{sec proof thm convergence actor sac}, we conclude that
  \begin{align}
    \label{eq tmp 16}
    \norm{\nabla \tilde J_{\lambda_k}(\pi_{\theta_k}; p_0')}^2 \leq \delta_k
  \end{align}
  holds with probability at least
  \begin{align}
    1 - \frac{\tilde C_t}{\delta_t}.
  \end{align}
  With the convergence to stationary points established in \eqref{eq tmp 16}, 
  we now use the following lemma from \citet{mei2020global} to study the optimality. 
  Let $\pi_{*, \eta}$ be the optimal policy w.r.t. the soft value function, i.e., $\forall \pi, s$,
  \begin{align}
    \tilde v_{\pi, \eta}(s) \leq \tilde v_{\pi_{*, \eta}, \eta}(s),
  \end{align}
  then we have
  \begin{lemma}
    \label{lem mei}
    (Lemma 15 of \citet{mei2020global})
    For any state distribution $d$ and $d'$,
    \begin{align}
      \tilde J_{\eta}(\pi_\theta; d) \geq \tilde J_{\eta}(\pi_{*, \eta}; d) - \frac{\ns \norm{\nabla \tilde J_\eta(\pi_\theta; d')}^2}{2 \eta \min_s d'(s) \left(\min_{s,a} \pi_\theta(a|s)\right)^2} \max_s \frac{d_{\pi_{*,\eta}, \gamma, d}(s)}{d_{\pi_\theta, \gamma, d'}(s)}.
    \end{align} 
  \end{lemma}
  Obviously, for Lemma~\ref{lem mei} to be nontrivial, 
  we have to ensure $\forall s, d'(s) > 0$.

  Letting $d=p_0, d'=p_0'$ in Lemma~\ref{lem mei} and using \eqref{eq tmp 16} yield that
  \begin{align}
    \tilde J_{\lambda_k}(\pi_{\theta_k}; p_0) \geq& \tilde J_{\lambda_k}(\pi_{*, \lambda_k}; p_0) - \frac{\ns \norm{\nabla \tilde J_{\lambda_k}(\pi_{\theta_k}; p_0')}^2}{2\lambda_k \min_s p_0'(s) \left(\min_{s, a} \pi_{\theta_k}(a|s)\right)^2} \max_s \frac{d_{\pi_{*, \lambda_k}, \gamma, p_0}(s)}{d_{\pi_{\theta_k}, \gamma, p_0'}(s)} \\
    \label{eq tmp 15}
    \geq& \tilde J_{\lambda_k}(\pi_{*, \lambda_k}; p_0) - \frac{\ns \delta_k}{2\lambda_k \min_s p_0'(s)\left( \min_{s, a} \pi_{\theta_k}(a|s)\right)^2} \frac{1}{\frac{\min_s p_0'(s)}{1-\gamma}}
  \end{align}
  holds with probability at least
  \begin{align}
    1 - \frac{\tilde C_t}{\delta_t}.
  \end{align}
  According to Proposition 2 of \citet{dai2017sbeed},
  we have 
  \begin{align}
    \max_s \abs{\tilde v_{\pi_{*, \eta}, \eta}(s) - v_{\pi_*}(s)} \leq \frac{\eta \log \na}{1 - \gamma},
  \end{align}
  implying
  \begin{align}
    \abs{\tilde J_\eta(\pi_{*, \eta};p_0) - J(\pi_*; p_0)} \leq& \frac{\eta \log \na}{1 - \gamma}, 
  \end{align}
  i.e.,
  \begin{align}
    \label{eq tmp 13}
    \tilde J_\eta(\pi_{*, \eta}; p_0) \geq J(\pi_*; p_0) - \frac{\eta \log \na}{1 - \gamma}.
  \end{align}
  From \eqref{eq relationship J tilde J},
  it is easy to see
  \begin{align}
    \label{eq tmp 14}
    \tilde J_\eta(\pi; p_0) \leq J(\pi; p_0) + \frac{\eta \log \na}{1 - \gamma}.
  \end{align}
  Putting \eqref{eq tmp 13} and \eqref{eq tmp 14} back to \eqref{eq tmp 15} yields
  \begin{align}
    J(\pi_{\theta_k};p_0) \geq J(\pi_*; p_0) - \frac{2\lambda_k \log\na}{1 - \gamma} - \frac{(1 - \gamma)\ns \delta_k}{2\lambda_k \left(\min_s p_0'(s) \min_{s, a} \pi_{\theta_k}(a|s)\right)^2},
  \end{align}
  which completes the proof.
\end{proof}

\section{Technical Lemmas}
\begin{lemma}
  \label{lem product of lipschitz functions}
  Let $f_1(x), f_2(x)$ be two Lipschitz continuous functions with Lipschitz constants $L_1, L_2$.
  Assume $\norm{f_1(x)} \leq U_1, \norm{f_2(x)} \leq U_2$,
  then
  $L_1U_2 + L_2U_1$ is a Lipschitz constant of $f(x) \doteq f_1(x)f_2(x)$. 
\end{lemma}
\begin{proof}
  \begin{align}
    &\norm{f_1(x)f_2(x) - f_1(y)f_2(y)} \\
    \leq& \norm{f_1(x)} \norm{f_2(x) - f_2(y)} + \norm{f_2(y)}\norm{f_1(x) - f_1(y)} \\
    \leq& (U_1 L_2 + U_2L_1) \norm{x-y}.
  \end{align}
\end{proof}

\begin{lemma}
  \label{lem smooth definition}
  The following statements about a differentiable function $f(x)$ are equivalent:
  \begin{enumerate} [(i).]
    \item $f(x)$ is $L$-smooth w.r.t. a norm $\norm{\cdot}_s$.
    \item $\norm{\nabla f(x) - \nabla f(y)}_s^* \leq L \norm{x - y}_s$.
    \item $\abs{f(y) - f(x) - \indot{\nabla f(x)}{y - x}} \leq \frac{L}{2}\norm{x - y}^2_s$.
  \end{enumerate}
\end{lemma}
\begin{proof}
  See e.g. Definition 5.1 and Lemma 5.7 of \citet{beck2017first}.
\end{proof}

\begin{lemma}
  \label{lem gradient of M}
  For any $x, x'$,
  \begin{align}
      \indot{\nabla M(x)}{x'} &\leq \norm{x}_m \norm{x'}_m, \\
      \indot{\nabla M(x)}{x} &\geq \norm{x}_m^2.
  \end{align}
\end{lemma}
\begin{proof}
  The proof is taken from 
  Section A.2 of \citet{chen2020finite} and we include it here for completeness.
  Since $M(x) = \frac{1}{2} \norm{x}_m^2$,
  by Theorem 3.47 of \citet{beck2017first},
  \begin{align}
      \nabla M(x) = \norm{x}_m v_{x},
  \end{align}
  where $v_{x}$ is a subgradient of $\norm{x}_m$ at $x$.
  Consequently,
  \begin{align}
      \indot{\nabla M(x)}{x'} &= \norm{x}_m \indot{v_x}{x'} \\
      &\leq \norm{x}_m \norm{v_x}_m^* \norm{x'}_m \\
      &\leq \norm{x}_m \norm{x'}_m,
  \end{align}
  where the first inequality results from Holder's inequality and the last inequality results from the fact that $\norm{v_x}_m^* \leq 1$ (Lemma A.1 of \citet{chen2020finite}).

  Further, notice that $\norm{x}_m$ is convex, we thus have
  \begin{align}
      \norm{x}_m \leq \norm{0}_m + \indot{v_{x}}{x-0},
  \end{align}
  implying
  \begin{align}
      \indot{\nabla M(x)}{x} &= \norm{x}_m \indot{v_x}{x} \geq \norm{x}_m^2.
  \end{align}
\end{proof}

\begin{lemma}
  \label{lem bound of xk diff}
  Given positive integers $t_1 < t_2$ satisfying
  \begin{align}
      \alpha_{t_1, t_2 - 1} \leq \frac{1}{4A},
  \end{align}
  we have, for any $t \in [t_1, t_2]$,
  \begin{align}
      \label{eq bound of xk diff1}
      \norm{w_t - w_{t_1}}_c &\leq 2 \alpha_{t_1, t_2 - 1}(A\norm{w_{t_1}}_c + B), \\
      \label{eq bound of xk diff2}
      \norm{w_t - w_{t_1}}_c &\leq 4 \alpha_{t_1, t_2 - 1}(A \norm{w_{t_2}}_c + B), \\
      \label{eq bound of xk diff3}
      \norm{w_t - w_{t_1}}_c &\leq \min \qty{\norm{w_{t_1}}_c, \norm{w_{t_2}}_c } + \frac{B}{A}.
  \end{align}
\end{lemma}
\begin{proof}
  Notice that 
  \begin{align}
      &\norm{w_{t+1}}_c - \norm{w_t}_c \\
      \leq& \norm{w_{t+1} - w_t}_c \\
      \leq& \alpha_t \norm{F_{\theta_t}(w_t, Y_t) - w_t + \epsilon_t}_c  \\
      \leq& \alpha_t \left( \norm{F_{\theta_t}(w_t, Y_t)}_c + \norm{w_t}_c + \norm{\epsilon_t}_c \right) \\
      \leq&\alpha_t (U_F + (L_F + 1)\norm{w_t}_c + \norm{\epsilon_t}_c) \qq{(Lemma~\ref{lem bound of fxy})} \\
      \leq&\alpha_t (U_F + U_\epsilon' + (U_\epsilon + L_F + 1)\norm{w_t}_c) \qq{(Assumption~\ref{assu mds})} \\
      \label{eq tmp 8}
      \leq& \alpha_t (A\norm{w_t}_c + B) \qq{(Using \eqref{eq shorthand a and b})}
  \end{align}
  The rest of the proof is exactly the same as the proof of Lemma A.2 of \citet{chen2021lyapunov} up to changes of notations. 
  We include it for completeness.
  Rearranging terms of the above inequality yields 
  \begin{align}
    \norm{w_{t+1}}_c + \frac{B}{A} \leq (1 + \alpha_t A)\left(\norm{w_t}_c + \frac{B}{A}\right),
  \end{align}
  implying that for any $t \in (t_1, t_2]$,
  \begin{align}
    \norm{w_t}_c + \frac{B}{A} \leq \prod_{j=t_1}^{t-1} (1 + A\alpha_j) \left(\norm{w_{t_1}}_c + \frac{B}{A}\right) .
  \end{align}
  Notice that for any $x \in [0, \frac{1}{2}]$,
  $1 + x \leq \exp(x) \leq 1 + 2x$ always hold.
  Hence 
  \begin{align}
    \alpha_{t_1, t_2-1} \leq \frac{1}{4A} 
  \end{align}
  implies
  \begin{align}
    \prod_{j=t_1}^{t-1} (1 + A\alpha_j) \leq \exp(A\alpha_{t_1, t-1}) \leq 1 + 2A \alpha_{t_1, t-1}.
  \end{align}
  Consequently, for any $t \in (t_1, t_2]$,
  we have
  \begin{align}
    \norm{w_t}_c + \frac{B}{A} &\leq \left(1 + 2A \alpha_{t_1, t-1}\right) \left(\norm{w_{t_1}}_c + \frac{B}{A} \right) \\
    \implies \norm{w_t}_c & \leq \left(1 + 2A \alpha_{t_1, t-1}\right) \norm{w_{t_1}}_c + 2B\alpha_{t_1, t-1},
  \end{align}
  which together with \eqref{eq tmp 8} yields that for any $t \in (t_1, t_2 - 1]$
  \begin{align}
    \norm{w_{t+1} - w_t}_c &\leq \alpha_t \left(A \norm{w_t}_c + B\right) \\
    &\leq \alpha_t \left(A \left(1 + 2A \alpha_{t_1, t-1}\right) \norm{w_{t_1}}_c + 2AB\alpha_{t_1, t-1} + B\right) \\
    &\leq 2 \alpha_t(A\norm{w_{t_1}}_c + B) \qq{(Using $\alpha_{t_1, t-1} \leq \frac{1}{4A}$)}.
  \end{align}
  Consequently, for any $t \in (t_1, t_2]$, we have
  \begin{align}
    \norm{w_t - w_{t_1}}_c &\leq \sum_{j=t_1}^{t-1} \norm{w_{j+1} - w_j}_c \leq \sum_{j=t_1}^{t-1} 2 \alpha_j (A \norm{w_{t_1}}_c + B) \\
    &= 2 \alpha_{t_1, t-1}(A \norm{w_{t_1}}_c + B) \leq 2 \alpha_{t_1, t_2-1} (A\norm{w_{t_1}}_c +B),
  \end{align}
  which completes the proof of \eqref{eq bound of xk diff1}. 
  For \eqref{eq bound of xk diff2},
  we have
  \begin{align}
    \norm{w_{t_2} - w_{t_1}}_c \leq& 2\alpha_{t_1, t_2-1} (A\norm{w_{t_1}}_c + B) \\
    \leq &2\alpha_{t_1, t_2-1} (A\norm{w_{t_1} - w_{t_2}}_c + A\norm{w_{t_2}}_c + B) \\
    \leq& \frac{1}{2}\norm{w_{t_1} - w_{t_2}}_c + 2 \alpha_{t_1, t_2 -1} (A\norm{w_{t_2}}_c + B),
  \end{align}
  implying
  \begin{align}
    \norm{w_{t_2} - w_{t_1}}_c \leq 4\alpha_{t_1, t_2 -1} (A\norm{w_{t_2}}_c + B).
  \end{align}
  Consequently, for any $t \in [t_1, t_2]$,
  \begin{align}
    \norm{w_t - w_{t_1}}_c\leq& 2\alpha_{t_1, t_2-1} (A\norm{w_{t_1}}_c + B) \\
    \leq &2\alpha_{t_1, t_2-1} (A\norm{w_{t_1} - w_{t_2}}_c + A\norm{w_{t_2}}_c + B) \\
    \leq&2\alpha_{t_1, t_2-1} \left(A4\alpha_{t_1, t_2 -1} (A\norm{w_{t_2}}_c + B) + A\norm{w_{t_2}}_c + B\right) \\
    \leq& 4\alpha_{t_1, t_2-1}(A\norm{w_{t_2}}_c + B) \qq{(Using $\alpha_{t_1, t_2-1} \leq \frac{1}{4A}$)},
  \end{align}
  which completes the proof of \eqref{eq bound of xk diff2}.
  \eqref{eq bound of xk diff1} implies
  \begin{align}
    \norm{w_t - w_{t_1}} \leq \norm{w_{t_1}}_c + \frac{B}{A},
  \end{align}
  \eqref{eq bound of xk diff2} implies
  \begin{align}
    \norm{w_t - w_{t_1}} \leq \norm{w_{t_2}}_c + \frac{B}{A},
  \end{align}
  then \eqref{eq bound of xk diff3} follows immediately,
  which completes the proof.
\end{proof}

\begin{lemma}
  \label{lem continuity of ergodic distribution}
  Let Assumptions \ref{assu lipschitz mu} and \ref{assu mu uniform ergodicity} hold.
  Then there exists a constant $L_{\mu}'$ such that $\forall \theta, \theta', a, s$,
  \begin{align}
    \abs{d_{\mu_\theta}(s, a) - d_{\mu_{\theta'}}(s, a)} \leq L_\mu' \norm{\theta - \theta'}.
  \end{align}
\end{lemma}
\begin{proof}
  See, e.g., Lemma 9 of \citet{zhang2021breaking}.
\end{proof}

\begin{lemma}
  \label{lem bound of matrix inverse diff}
  For any $\norm{\cdot}$,
  we have
  \begin{align}
    \norm{X^{-1} - Y^{-1}} \leq \norm{X^{-1}} \norm{X-Y} \norm{Y^{-1}}.
  \end{align}
\end{lemma}
\begin{proof}
  \begin{align}
    \norm{X^{-1} - Y^{-1}} &= \norm{X^{-1}YY^{-1} - X^{-1}XY^{-1}} \leq \norm{X^{-1}} \norm{X-Y} \norm{Y^{-1}}.
  \end{align}
\end{proof}

\begin{lemma}
  \label{lem softmax policy gradient}
  With softmax parameterization,
  \begin{align}
    \label{eq softmax pg 1}
    \dv{\pi_\theta(a|s)}{\theta_{s',a'}} &= \mathbb{I}_{s=s'} \pi_\theta(a|s) \left(\mathbb{I}_{a = a'} - \pi_\theta(a' | s)\right), \\
    \label{eq softmax pg 2}
    \dv{\log \pi_\theta(a|s)}{\theta_{s', a'}} &= \mathbb{I}_{s = s'}\left(\mathbb{I}_{a = a'} - \pi_\theta(a' | s)\right), \\
    \label{eq softmax pg 3}
    \dv{\kl{\fU_\fA}{\pi_\theta(\cdot | s)}}{\theta_{s', a'}} &= \mathbb{I}_{s=s'}(\pi_\theta(a' | s) - \frac{1}{\na}), \\
    \label{eq softmax pg 3.5}
    \sum_a \dv{\pi_\theta(a|s)}{\theta_{s',a'}} q_{\pi_\theta}(s, a) &= \mathbb{I}_{s=s'} \pi_\theta(a'|s){\adv_{\pi_\theta}(s, a')}, \\
    \label{eq softmax pg 4}
    \dv{J(\pi_\theta; p_0)}{\theta_{s, a}} &= \frac{1}{1 - \gamma} d_{\pi_\theta, \gamma, p_0}(s) \pi_\theta(a|s) {\adv_{\pi_\theta}(s, a)}, \\
    \label{eq softmax pg 5}
    \norm{\nabla \ent{\pi_\theta(\cdot | s)}} &\leq \log \na + e^{-1}, \\
    \label{eq softmax pg 6}
    \sum_a \dv{\pi_\theta(a|s)}{\theta_{s',a'}} \left( \tilde q_{\pi_\theta, \eta}(s, a) - \eta \log \pi_\theta(a|s) \right) &= \mathbb{I}_{s=s'} \pi_\theta(a'|s) \tilde \adv_{\pi_\theta, \eta}(s, a)  \\
    \label{eq softmax pg 7}
    \dv{\tilde J_\eta(\pi_\theta; p_0)}{\theta_{s, a}} &= \frac{1}{1 - \gamma} d_{\pi_\theta, \gamma, p_0}(s) \pi_\theta(a|s) {\tilde \adv_{\pi_\theta, \eta}(s, a)},
  \end{align}
  where
  \begin{align}
    \tilde \adv_{\pi_\theta, \eta}(s, a) &\doteq \tilde q_{\pi_\theta, \eta}(s, a') - \eta \log \pi_{\theta}(a'|s) -  \tilde v_{\pi_\theta, \eta}(s), \\
    \adv_{\pi_\theta, \eta}(s, a) &\doteq q_{\pi_\theta}(s, a) - v_{\pi_\theta}(s).
  \end{align}
  Further, for any $s$, $\ent{\pi_\theta(\cdot | s)}$ is $(4 + 8 \log \na)$-smooth.
\end{lemma}
\begin{proof}
  \eqref{eq softmax pg 1} is well-known.
  For \eqref{eq softmax pg 2}, we have
  \begin{align}
    \dv{\log \pi_\theta(a|s)}{\theta_{s', a'}} = \frac{1}{\pi_\theta(a|s)} \dv{\pi_\theta(a|s)}{\theta_{s',a'}} = \mathbb{I}_{s = s'}\left(\mathbb{I}_{a = a'} - \pi_\theta(a' | s)\right).
  \end{align}
  For \eqref{eq softmax pg 3}, we have
  \begin{align}
    \dv{\kl{\fU_\fA}{\pi_\theta(\cdot | s)}}{\theta_{s', a'}} =& - \frac{\mathbb{I}_{s=s'}}{\na} \sum_a \dv{\log \pi_\theta(a|s)}{\theta_{s', a'}}  \\
    =& - \frac{\mathbb{I}_{s=s'}}{\na} \sum_a \left(\mathbb{I}_{a = a'} - \pi_\theta(a' | s)\right).
  \end{align}
  Since
  \begin{align}
    \sum_a \left(\mathbb{I}_{a = a'} - \pi_\theta(a' | s)\right) = \left(\sum_a \left(0 - \pi_\theta(a' | s)\right) \right) + 1 = 1 - \na \pi_{\theta}(a'|s),
  \end{align}
  we have
  \begin{align}
    \dv{\kl{\fU}{\pi_\theta(\cdot | s)}}{\theta_{s', a'}} = \mathbb{I}_{s=s'}(\pi(a' | s) - \frac{1}{\na}).
  \end{align}
  For \eqref{eq softmax pg 3.5},
  \begin{align}
    &\sum_a \dv{\pi_\theta(a|s)}{\theta_{s',a'}} q_{\pi_\theta}(s, a) \\
    =& \sum_a \mathbb{I}_{s=s'} \pi_\theta(a|s) \left(\mathbb{I}_{a = a'} - \pi_\theta(a' | s)\right) q_{\pi_\theta}(s, a) \\
    =& \mathbb{I}_{s=s'} \left(\pi_\theta(a'|s)q_{\pi_\theta}(s, a') + \sum_a \pi_\theta(a|s) \left(0 - \pi_\theta(a' | s)\right) q_{\pi_\theta}(s, a) \right) \\
    =& \mathbb{I}_{s=s'} \left(\pi_\theta(a'|s)q_{\pi_\theta}(s, a') - \pi_\theta(a'|s) v_{\pi_\theta}(s) \right).
  \end{align}
  For \eqref{eq softmax pg 4}, see, e.g., Lemma C.1 of \citet{agarwal2019optimality}. 
  For \eqref{eq softmax pg 5}, we have
  \begin{align}
    &\dv{\ent{\pi_\theta(\cdot |s)}}{\theta_{s', a'}} \\
    =& -\mathbb{I}_{s=s'} \sum_a \dv{\pi_\theta(a|s)}{\theta_{s',a'}} \log \pi_\theta(a|s) + 0 \\
    =& -\mathbb{I}_{s=s'} \sum_a \pi_\theta(a|s) \left(\mathbb{I}_{a = a'} - \pi_\theta(a' | s)\right) \log \pi_\theta(a|s) \\
    =& -\mathbb{I}_{s=s'} \left(\pi_\theta(a'|s) \ent{\pi_\theta(\cdot | s)} + \pi_\theta(a'|s) \log \pi_\theta(a'|s) \right),
  \end{align}
  implying
  \begin{align}
    \norm{\nabla \ent{\pi_\theta(\cdot | s)}} \leq \log \na + e^{-1}.
  \end{align}
  By setting $\gamma = 0$ and putting all the mass of $\rho$ (initial distribution) in $s$ in Lemma 14 of \citet{mei2020global},
  we obtain that $\ent{\pi_\theta(\cdot |s)}$ is $(4+8\log \na)$-smooth.
  For \eqref{eq softmax pg 6}, we have
  \begin{align}
    &\sum_a \dv{\pi_\theta(a|s)}{\theta_{s',a'}} \left( \tilde q_{\pi_\theta, \eta}(s, a) - \eta \log \pi_\theta(a|s) \right)\\
    =& \sum_a \mathbb{I}_{s=s'} \pi_\theta(a|s) \left(\mathbb{I}_{a = a'} - \pi_\theta(a' | s)\right) \left( \tilde q_{\pi_\theta, \eta}(s, a) - \eta \log \pi_\theta(a|s) \right) \\
    =& \mathbb{I}_{s=s'} \left(\pi_\theta(a'|s) \left( \tilde q_{\pi_\theta, \eta}(s, a') - \eta \log \pi_\theta(a'|s) \right) -\sum_a \pi_\theta(a|s)  \pi_\theta(a' | s) \left( \tilde q_{\pi_\theta, \eta}(s, a) - \eta \log \pi_\theta(a|s) \right) \right) \\
    =& \mathbb{I}_{s=s'} \pi_\theta(a'|s) \left(\tilde q_{\pi_\theta, \eta}(s, a') - \eta \log \pi_{\theta}(a'|s) -  \tilde v_{\pi_\theta, \eta}(s) \right).
  \end{align}
  Since \eqref{eq softmax pg 7} is identical to Lemma 10 of \citet{mei2020global},
  we have completed the proof.
\end{proof}

\section{Proof of Auxiliary Lemmas}
\subsection{Proof of Lemma~\ref{lem bound t1}}
\label{sec proof lem bound t1}
\lemboundtone*
\begin{proof}
  \begin{align}
      T_1 =&\indot{\nabla M(w_t - w^*_{\theta_t})}{w^*_{\theta_t} - w^*_{\theta_{t+1}}} \\
      \leq & \norm{w_t - w^*_{\theta_t}}_m \norm{w^*_{\theta_t} - w^*_{\theta_{t+1}}}_m  \qq{(Lemma \ref{lem gradient of M})}\\
      \leq & \norm{w_t - w^*_{\theta_t}}_m \frac{L_w L_\theta \beta_t}{l_{cm}} \qq{(Assumptions \ref{assu regularization}, \ref{assu twotimescale} and Lemma \ref{lem property of M})}.
  \end{align}
\end{proof}

\subsection{Proof of Lemma~\ref{lem bound t2}}
\label{sec proof lem bound t2}
\lemboundttwo*
\begin{proof}
\begin{align}
  T_2 = &\indot{\nabla M(w_t - w^*_{\theta_t})}{\bar F_{\theta_t}(w_t) - w_t} \\
  =& \indot{\nabla M (w_t - w^*_{\theta_t})}{\bar F_{\theta_t}(w_t) - \bar F_{\theta_t}(w^*_{\theta_t})} - \indot{\nabla M(w_t - w^*_{\theta_t})}{w_t - w^*_{\theta_t}}
  \intertext{\hfill ($w^*_{\theta_t}$ is the fixed point).}
\end{align}
To bound the first inner product, we have
\begin{align}
  &\indot{\nabla M (w_t - w^*_{\theta_t})}{\bar F_{\theta_t}(w_t) - \bar F_{\theta_t}(w^*_{\theta_t})} \\ 
  \leq & \norm{w_t - w^*_{\theta_t}}_m  \norm{\bar F_{\theta_t}(w_t) - \bar F_{\theta_t}(w^*_{\theta_t})}_m \qq{(Lemma \ref{lem gradient of M})} \\
  \leq & \norm{w_t - w^*_{\theta_t}}_m \frac{1}{l_{cm}} \kappa \norm{w_t - w^*_{\theta_t}}_c \\
  \leq & \frac{u_{cm} \kappa}{l_{cm}} \norm{w_t - w^*_{\theta_t}}_m^2 
\end{align}
For the second inner product,
Lemma \ref{lem gradient of M} implies that
\begin{align}
  \indot{\nabla M(w_t - w^*_{\theta_t})}{w_t - w^*_{\theta_t}} \geq \norm{w_t - w^*_{\theta_t}}^2_m.
\end{align}
Putting the bounds for the two inner products together completes the proof.
\end{proof}

\subsection{Proof of Lemma~\ref{lem bound of t31}}
\label{sec proof lem bound t31}
\lemboundoftthreeone*
\begin{proof}
  \begin{align}
      T_{31} = &\indot{\nabla M(w_t - w^*_{\theta_t}) - \nabla M(w_{t-\tau_{\alpha_t}} - w^*_{\theta_{t-\tau_{\alpha_t}}})}{F_{\theta_t}(w_t, Y_t) - \bar F_{\theta_t}(w_t)} \\
      \leq &\norm{\nabla M(w_t - w^*_{\theta_t}) - \nabla M(w_{t-\tau_{\alpha_t}} - w^*_{\theta_{t-\tau_{\alpha_t}}})}_s^*\norm{F_{\theta_t}(w_t, Y_t) - \bar F_{\theta_t}(w_t)}_s.
  \end{align}
To bound the first term,
\begin{align}
  &\norm{\nabla M(w_t - w^*_{\theta_t}) - \nabla M(w_{t-\tau_{\alpha_t}} - w^*_{\theta_{t-\tau_{\alpha_t}}})}_s^* \\
  \leq& \frac{L}{\xi} \norm{w_t - w_{t-\tau_{\alpha_t}} + w_{\theta_{t-\tau_{\alpha_t}}}^* - w_{\theta_{t}}^*}_s \qq{(Lemmas \ref{lem property of M} and \ref{lem smooth definition})} \\
  \leq& \frac{L}{\xi} \norm{w_t - w_{t-\tau_{\alpha_t}}}_s + \frac{L}{\xi} \norm{w^*_{\theta_t} - w_{\theta_{t-\tau_{\alpha_t}}}^*}_s \\
  \leq& \frac{L}{\xi l_{cs}} \norm{w_t - w_{t-\tau_{\alpha_t}}}_c + \frac{L}{\xi l_{cs}} L_w L_\theta \beta_{t-\tau_{\alpha_t}, t-1} \\
  \leq& \frac{4L \alpha_{t-\tau_{\alpha_t}, t-1}}{\xi l_{cs}} (A\norm{w_t}_c + B)+ \frac{L}{\xi l_{cs}} L_w L_\theta \beta_{t-\tau_{\alpha_t}, t-1} \qq{(Lemma \ref{lem bound of xk diff})}\\
  \leq& \frac{4L \alpha_{t-\tau_{\alpha_t}, t-1}}{\xi l_{cs}} (A\norm{w_t - w^*_{\theta_t}}_c + A\norm{w^*_{\theta_t}}_c + B)+ \frac{L}{\xi l_{cs}} L_w L_\theta \beta_{t-\tau_{\alpha_t}, t-1} \\
  \leq& \frac{4L \alpha_{t-\tau_{\alpha_t}, t-1}}{\xi l_{cs}} (A\norm{w_t - w^*_{\theta_t}}_c + A\norm{w^*_{\theta_t}}_c + B)+ \frac{4L}{\xi l_{cs}} (L_w L_\theta + 1) \alpha_{t-\tau_{\alpha_t}, t-1} \\
  \leq& \frac{4L (L_wL_\theta+1) \alpha_{t-\tau_{\alpha_t}, t-1}}{\xi l_{cs}} (A\norm{w_t - w^*_{\theta_t}}_c + AU_w + B + 1).
\end{align}
To bound the second term,
\begin{align}
  &\norm{F_{\theta_t}(w_t, Y_t) - \bar F_{\theta_t}(w_t)}_s \\
  \leq& \frac{1}{l_{cs}}\norm{F_{\theta_t}(w_t, Y_t) - \bar F_{\theta_t}(w_t)}_c \\
  \leq& \frac{1}{l_{cs}}\left(\norm{F_{\theta_t}(w_t, Y_t)}_c + \norm{\bar F_{\theta_t}(w_t) - \bar F_{\theta_t}(w^*_{\theta_t})}_c + \norm{w^*_{\theta_t}}_c \right)\\
  \leq&\frac{1}{l_{cs}}  \left( U_F + L_F \norm{w_t}_c + \norm{w_t - w^*_{\theta_t}}_c + \norm{w^*_{\theta_t}}_c \right) \qq{(Lemma~\ref{lem bound of fxy})} \\
  \leq&\frac{1}{l_{cs}} \left( U_F + L_F \norm{w_t - w^*_{\theta_t}}_c + L_F \norm{w^*_{\theta_t}}_c +  \norm{w_t - w^*_{\theta_t}}_c + \norm{w^*_{\theta_t}}_c \right) \\
  \leq&\frac{1}{l_{cs}} \left(A\norm{w_t - w^*_{\theta_t}}_c + A \norm{w^*_{\theta_t}}_c + B\right).
\end{align}
Combining the two inequalities together yields
\begin{align}
  &\indot{\nabla M(w_t - w^*_{\theta_t}) - \nabla M(w_{t-\tau_{\alpha_t}} - w^*_{\theta_t})}{F_{\theta_t}(w_t, Y_t) - \bar F_{\theta_t}(w_t)} \\
  \leq& \frac{4L (L_wL_\theta + 1)  \alpha_{t-\tau_{\alpha_t}, t-1}}{\xi l_{cs}^2} (A\norm{w_t - w^*_{\theta_t}}_c + C)^2 \\
  \leq& \frac{8L (L_wL_\theta + 1)  \alpha_{t-\tau_{\alpha_t}, t-1}}{\xi l_{cs}^2} (A^2 u_{cm}^2\norm{w_t - w^*_{\theta_t}}_m^2 + C^2),
\end{align}
which completes the proof.
\end{proof}

\subsection{Proof of Lemma~\ref{lem bound t32}}
\label{sec proof lem bound t32}
\lemboundtthreetwo*
\begin{proof}
  \begin{align}
T_{32} = &\indot{\nabla M(w_{t-\tau_{\alpha_t}} - w^*_{\theta_{t-\tau_{\alpha_t}}})}{F_{\theta_t}(w_t, Y_t) - F_{\theta_t}(w_{t- \tau_{\alpha_t}}, Y_t) + \bar F_{\theta_t}(w_{t- \tau_{\alpha_t}}) - \bar F_{\theta_t}(w_t)} \\
\leq & \norm{\nabla M(w_{t-\tau_{\alpha_t}} - w^*_{\theta_{t-\tau_{\alpha_t}}})}_s^*\norm{F_{\theta_t}(w_t, Y_t) - F_{\theta_t}(w_{t- \tau_{\alpha_t}}, Y_t) + \bar F_{\theta_t}(w_{t- \tau_{\alpha_t}}) - \bar F_{\theta_t}(w_t)}_s \\
\leq & \frac{1}{l_{cs}} \norm{\nabla M(w_{t-\tau_{\alpha_t}} - w^*_{\theta_{t-\tau_{\alpha_t}}})}_s^*\norm{F_{\theta_t}(w_t, Y_t) - F_{\theta_t}(w_{t- \tau_{\alpha_t}}, Y_t) + \bar F_{\theta_t}(w_{t- \tau_{\alpha_t}}) - \bar F_{\theta_t}(w_t)}_c
  \end{align}
For the first term,
\begin{align}
  \label{eq gradient bound dual norm}
  &\norm{\nabla M(w_{t-\tau_{\alpha_t}} - w^*_{\theta_{t-\tau_{\alpha_t}}})}_s^* \\
  =&\norm{\nabla M(w_{t-\tau_{\alpha_t}} - w^*_{\theta_{t-\tau_{\alpha_t}}}) - \nabla M(w^*_{\theta_t} - w^*_{\theta_t})}_s^*  \\
  \intertext{\hfill (Using $\nabla M(0) = 0$, see the proof of Lemma~\ref{lem gradient of M})}
  \leq&\frac{L}{\xi}\norm{(w_{t-\tau_{\alpha_t}} - w^*_{\theta_{t-\tau_{\alpha_t}}}) - (w^*_{\theta_t} - w^*_{\theta_t})}_s \qq{(Lemmas~\ref{lem property of M} and \ref{lem smooth definition})}\\
  \leq& \frac{L}{\xi} \norm{w_{t-\tau_{\alpha_t}} - w^*_{\theta_t}}_s + \frac{L}{\xi} \norm{w^*_{\theta_t} - w^*_{\theta_{t-\tau_{\alpha_t}}}}_s  \\
  \leq& \frac{L }{\xi l_{cs}} \norm{w_{t-\tau_{\alpha_t}} - w^*_{\theta_t}}_c + \frac{L}{\xi l_{cs}} L_w L_\theta \beta_{t-\tau_{\alpha_t}, t-1} \\
  \leq& \frac{L }{\xi l_{cs}}\left(\norm{w_{t-\tau_{\alpha_t}} - w_t}_c + \norm{w_t - w^*_{\theta_t}}_c \right)+ \frac{L}{\xi l_{cs}} L_w L_\theta \beta_{t-\tau_{\alpha_t}, t-1} \\
  \leq& \frac{L}{\xi l_{cs}}\left(\norm{w_t}_c + \frac{B}{A} + \norm{w_t - w^*_{\theta_t}}_c \right)+ \frac{L}{\xi l_{cs}} L_w L_\theta \beta_{t-\tau_{\alpha_t}, t-1} \qq{(Lemma \ref{lem bound of xk diff})} \\
  \leq& \frac{L (1 + L_wL_\theta \beta_{t-\tau_{\alpha_t}, t-1})}{\xi l_{cs}}\left(\norm{w^*_{\theta_t}}_c + \norm{w_t - w^*_{\theta_t}}_c + \frac{B}{A} + \norm{w_t - w^*_{\theta_t}}_c + 1\right) \\
  \leq& \frac{2L (1 + L_wL_\theta \beta_{t-\tau_{\alpha_t}, t-1})}{\xi l_{cs}}\left(U_w + \frac{B}{A} + \norm{w_t - w^*_{\theta_t}}_c + 1\right).
\end{align}
For the second term,
\begin{align}
  &\norm{F_{\theta_t}(w_t, Y_t) - F_{\theta_t}(w_{t- \tau_{\alpha_t}}, Y_t) + \bar F_{\theta_t}(w_{t- \tau_{\alpha_t}}) - \bar F_{\theta_t}(w_t)}_c \\
  \leq &\norm{F_{\theta_t}(w_t, Y_t) - F_{\theta_t}(w_{t- \tau_{\alpha_t}}, Y_t)}_c + \norm{\bar F_{\theta_t}(w_{t- \tau_{\alpha_t}}) - \bar F_{\theta_t}(w_t)}_c \\
  \leq& L_F \norm{w_{t-\tau_{\alpha_t}} - w_t}_c + \norm{\sum_{y} d_{\theta_t}(y) \left(F_{\theta_t}(w_{t-\tau_{\alpha_t}}, y) - F_{\theta_t}(w_t, y)\right)}_c \\
  \leq& 2L_F \norm{w_{t-\tau_{\alpha_t}} - w_t}_c \\
  \leq& 2A \norm{w_{t-\tau_{\alpha_t}} - w_t}_c \\
  \leq& 8A \alpha_{t-\tau_{\alpha_t}, t-1} \left(A \norm{w_t}_c + B\right) \qq{(Lemma~\ref{lem bound of xk diff})} \\
  \leq& 8A \alpha_{t-\tau_{\alpha_t}, t-1} (A\norm{w_t - w^*_{\theta_t}}_c + A \norm{w^*_{\theta_t}}_c + B).
\end{align}
Combining the two inequalities together yields
\begin{align}
  &\indot{\nabla M(w_{t-\tau_{\alpha_t}} - w^*_{\theta_t})}{F_{\theta_t}(w_t, Y_t) - F_{\theta_t}(w_{t- \tau_{\alpha_t}}, Y_t) + \bar F_{\theta_t}(w_{t- \tau_{\alpha_t}}) - \bar F_{\theta_t}(w_t)} \\
  \leq&\frac{16L  \alpha_{t-\tau_{\alpha_t}, t-1}(1 + L_wL_\theta \beta_{t-\tau_{\alpha_t}, t-1})}{\xi l_{cs}^2}(A\norm{w_t - w^*_{\theta_t}}_c + A U_w + B + A)^2 \\
  \leq&\frac{32L  \alpha_{t-\tau_{\alpha_t}, t-1}(1 + L_wL_\theta \beta_{t-\tau_{\alpha_t}, t-1})}{\xi l_{cs}^2}(u_{cm}^2A^2 \norm{w_t - w^*_{\theta_t}}_m^2 + C^2) 
\end{align}
which completes the proof.
\end{proof}

\subsection{Proof of Lemma~\ref{lem bound t331}}
\label{sec proof lem bound t331}
\lemboundtthreethreeone*
\begin{proof}
  \begin{align}
      \label{eq conditional independence t331}
      &\E\left[T_{331}\right] \\
      = &\E\left[\indot{\nabla M(w_{t-\tau_{\alpha_t}} - w^*_{\theta_{t-\tau_{\alpha_t}}})}{F_{\theta_{t - \tau_{\alpha_t}}}(w_{t- \tau_{\alpha_t}}, \tilde Y_t) - \bar F_{\theta_{t-\tau_{\alpha_t}}}(w_{t- \tau_{\alpha_t}})}\right] \\
      =&\E \left[ \E\left[\indot{\nabla M(w_{t-\tau_{\alpha_t}} - w^*_{\theta_{t-\tau_{\alpha_t}}})}{F_{\theta_{t - \tau_{\alpha_t}}}(w_{t- \tau_{\alpha_t}}, \tilde Y_t) - \bar F_{\theta_{t-\tau_{\alpha_t}}}(w_{t- \tau_{\alpha_t}})} \mid \substack{\theta_{t-\tau_{\alpha_t}} \\ w_{t-\tau_{\alpha_t}} \\ Y_{t-\tau_{\alpha_t}}} \right] \right]\\
      =&\E \left[ \indot{\nabla M(w_{t-\tau_{\alpha_t}} - w^*_{\theta_{t-\tau_{\alpha_t}}})}{\E\left[F_{\theta_{t - \tau_{\alpha_t}}}(w_{t- \tau_{\alpha_t}}, \tilde Y_t) - \bar F_{\theta_{t-\tau_{\alpha_t}}}(w_{t- \tau_{\alpha_t}})\mid \substack{\theta_{t-\tau_{\alpha_t}} \\ w_{t-\tau_{\alpha_t}} \\ Y_{t-\tau_{\alpha_t}}} \right] }\right] \\
      \leq&\E \left[ \norm{\nabla M(w_{t-\tau_{\alpha_t}} - w^*_{\theta_{t-\tau_{\alpha_t}}})}_s^* \norm{\E\left[F_{\theta_{t - \tau_{\alpha_t}}}(w_{t- \tau_{\alpha_t}}, \tilde Y_t) - \bar F_{\theta_{t-\tau_{\alpha_t}}}(w_{t- \tau_{\alpha_t}})\mid \substack{\theta_{t-\tau_{\alpha_t}} \\ w_{t-\tau_{\alpha_t}} \\ Y_{t-\tau_{\alpha_t}}} \right] }_s \right] \\
      \leq&\frac{1}{l_{cs}}\E \left[ \norm{\nabla M(w_{t-\tau_{\alpha_t}} - w^*_{\theta_{t-\tau_{\alpha_t}}})}_s^* \norm{\E\left[F_{\theta_{t - \tau_{\alpha_t}}}(w_{t- \tau_{\alpha_t}}, \tilde Y_t) - \bar F_{\theta_{t-\tau_{\alpha_t}}}(w_{t- \tau_{\alpha_t}})\mid \substack{\theta_{t-\tau_{\alpha_t}} \\ w_{t-\tau_{\alpha_t}} \\ Y_{t-\tau_{\alpha_t}}} \right] }_c \right]
  \end{align}
We now bound the inner expectation.
\begin{align}
  &\norm{\E\left[F_{\theta_{t - \tau_{\alpha_t}}}(w_{t- \tau_{\alpha_t}}, \tilde Y_t) - \bar F_{\theta_{t-\tau_{\alpha_t}}}(w_{t- \tau_{\alpha_t}}) \mid \substack{\theta_{t-\tau_{\alpha_t}} \\w_{t-\tau_{\alpha_t}} \\ Y_{t-\tau_{\alpha_t}}}\right]}_c \\
  =&\norm{\sum_y \left(\Pr(\tilde Y_t = y \mid \substack{\theta_{t-\tau_{\alpha_t}} \\w_{t-\tau_{\alpha_t}} \\ Y_{t-\tau_{\alpha_t}}}) - d_{\theta_{t-\tau_{\alpha_t}}}(y) \right) F_{\theta_{t - \tau_{\alpha_t}}}(w_{t- \tau_{\alpha_t}}, y) }_c \\
  \leq &\max_y \norm{ F_{\theta_{t - \tau_{\alpha_t}}}(w_{t- \tau_{\alpha_t}}, y) }_c \sum_y \left|\Pr(\tilde Y_t = y \mid \substack{\theta_{t-\tau_{\alpha_t}} \\w_{t-\tau_{\alpha_t}} \\ Y_{t-\tau_{\alpha_t}}}) - d_{\theta_{t-\tau_{\alpha_t}}}(y) \right| \\
  \label{eq bound of xk-tk}
  \leq &\max_y \norm{ F_{\theta_{t - \tau_{\alpha_t}}}(w_{t- \tau_{\alpha_t}}, y) }_c \alpha_t \qq{(Definition of $\tau_{\alpha_t}$)} \\
  \leq & \alpha_t (U_F + L_F \norm{w_{t-\tau_{\alpha_t}}}_c) \qq{(Lemma \ref{lem bound of fxy})} \\
  \leq & \alpha_t (U_F + L_F \norm{w_{t-\tau_{\alpha_t}} - w_t}_c + L_F \norm{w_t}_c) \\
  \leq & \alpha_t (B + A (\norm{w_t}_c + \frac{B}{A}) + A \norm{w_t}_c) \qq{(Lemma \ref{lem bound of xk diff})} \\
  \leq & \alpha_t (2B + (A + 1) \norm{w_t}_c) \\
  \leq & 2\alpha_t (B + A \norm{w_t}_c) \\
  \leq & 2\alpha_t (B + A \norm{w_t - w^*_{\theta_t}}_c + A \norm{w^*_{\theta_t}}_c).
\end{align}
Using the above inequality and \eqref{eq gradient bound dual norm} yields
\begin{align}
  &\E\left[T_{331}\right] \\
  \leq & \E \left[\frac{4L  \alpha_t(1 + L_wL_\theta \beta_{t-\tau_{\alpha_t}, t-1})}{A\xi l_{cs}^2}\left(AU_w + B + A\norm{w_t - w^*_{\theta_t}}_c + A\right)^2\right] \\
  \leq & \E \left[\frac{8L \alpha_t(1 + L_wL_\theta \beta_{t-\tau_{\alpha_t}, t-1})}{A\xi l_{cs}^2}\left(A^2 u_{cm}^2 \norm{w_t - w^*_{\theta_t}}_m^2 +C^2\right)\right],
\end{align}
which completes the proof.
\end{proof}

\subsection{Proof of Lemma~\ref{lem bound t332}}
\label{sec proof lem bound t332}
\lemboundtthreethreetwo*
\begin{proof}
  \begin{align}
      &\E\left[T_{332}\right] \\
      = &\E\left[\indot{\nabla M(w_{t-\tau_{\alpha_t}} - w^*_{\theta_{t-\tau_{\alpha_t}}})}{F_{\theta_{t - \tau_{\alpha_t}}}(w_{t- \tau_{\alpha_t}}, Y_t) -F_{\theta_{t - \tau_{\alpha_t}}}(w_{t- \tau_{\alpha_t}}, \tilde Y_t)}\right] \\
      \leq & \frac{1}{l_{cs}} \E\left[\norm{\nabla M(w_{t-\tau_{\alpha_t}} - w^*_{\theta_{t-\tau_{\alpha_t}}})}_s^* \norm{ \E \left[{F_{\theta_{t - \tau_{\alpha_t}}}(w_{t- \tau_{\alpha_t}}, Y_t) -F_{\theta_{t - \tau_{\alpha_t}}}(w_{t- \tau_{\alpha_t}}, \tilde Y_t)} \mid \substack{w_{t-\tau_{\alpha_t}} \\ \theta_{t-\tau_{\alpha_t}} \\ Y_{t-\tau_{\alpha_t}}} \right]}_s\right] \\
      \intertext{\hfill (Similar to \eqref{eq conditional independence t331})}
      \leq& \E \Bigg[\frac{2L(1 + L_wL_\theta \beta_{t-\tau_{\alpha_t}, t-1})}{\xi l_{cs}^2}\left(\norm{w^*_{\theta_t}}_c + \frac{B}{A} + \norm{w_t - w^*_{\theta_t}}_c + 1\right) \\
      &\times 2\ny L_PL_\theta \sum_{j=t-\tau_{\alpha_t}}^{t-1}\beta_{t-\tau_{\alpha_t}, j}(B + A \norm{w_t - w^*_{\theta_t}}_c + A \norm{w^*_{\theta_t}}_c) \Bigg] \\
      \intertext{\hfill (Using \eqref{eq gradient bound dual norm} and Lemma \ref{lem chain difference bound})}
      \leq & \frac{8 \ny L_P L_\theta \sum_{j=t-\tau_{\alpha_t}}^{t-1}\beta_{t-\tau_{\alpha_t}, j} L(1 + L_wL_\theta \beta_{t-\tau_{\alpha_t}, t-1})}{A \xi l_{cs}^2} \left(u_{cm}^2 A^2 \E\left[\norm{w_t - w^*_{\theta_t}}_m^2\right] + C^2\right),
  \end{align}
  which completes the proof.
\end{proof}

\subsection{Proof of Lemma~\ref{lem bound of t333}}
\label{sec proof lem bound t333}
\lemboundoftthreethreethree*
\begin{proof}
  \begin{align}
      T_{333} = &\indot{\nabla M(w_{t-\tau_{\alpha_t}} - w^*_{\theta_{t-\tau_{\alpha_t}}})}{F_{\theta_{t}}(w_{t- \tau_{\alpha_t}}, Y_t) -F_{\theta_{t - \tau_{\alpha_t}}}(w_{t- \tau_{\alpha_t}}, Y_t)}  \\
      \leq &\norm{\nabla M(w_{t-\tau_{\alpha_t}} - w^*_{\theta_{t-\tau_{\alpha_t}}})}_s^* \norm{F_{\theta_{t}}(w_{t- \tau_{\alpha_t}}, Y_t) -F_{\theta_{t - \tau_{\alpha_t}}}(w_{t- \tau_{\alpha_t}}, Y_t)}_s  \\
      \leq & \frac{2L(1 + L_wL_\theta \beta_{t-\tau_{\alpha_t}, t-1})}{\xi l_{cs}^2}\left(\norm{w^*_{\theta_t}}_c + \frac{B}{A} + \norm{w_t - w^*_{\theta_t}}_c + 1\right) \\
      &\times L_F' L_\theta \beta_{t-\tau_{\alpha_t}, t-1} \left(\norm{w_{t-\tau_{\alpha_t}}}_c + U_F' \right) \qq{(Using \eqref{eq gradient bound dual norm} and Assumption \ref{assu regularization})}.
  \end{align}
Since 
\begin{align}
  \label{eq tmp 3}
  &\norm{w_{t-\tau_{\alpha_t}}}_c \\
  \leq & \norm{w_{t-\tau_{\alpha_t}} - w_t}_c + \norm{w_t}_c \\
  \leq & 2\norm{w_t}_c + \frac{B}{A} \qq{(Lemma \ref{lem bound of xk diff})} \\
  \leq & 2\norm{w_t - w^*_{\theta_t}}_c + 2\norm{w^*_{\theta_t}}_c + \frac{B}{A},
\end{align}
we have
\begin{align}
  T_{333} \leq \frac{8LL_F' L_\theta \beta_{t-\tau_{\alpha_t}, t-1} (1 + L_wL_\theta \beta_{t-\tau_{\alpha_t}, t-1})}{A^2 \xi l_{cs}^2}\left(u_{cm}^2 A^2 \norm{w_t - w^*_{\theta_t}}_m^2 + (AU_x + A + B + AU_F')^2\right),
\end{align}
which completes the proof.
\end{proof}

\subsection{Proof of Lemma~\ref{lem bound t334}}
\label{sec proof lem bound t334}
\lemboundtthreethreefour*
\begin{proof}
  \begin{align}
      T_{334} = &\indot{\nabla M(w_{t-\tau_{\alpha_t}} - w^*_{\theta_{t-\tau_{\alpha_t}}})}{\bar F_{\theta_{t - \tau_{\alpha_t}}}(w_{t- \tau_{\alpha_t}}) - \bar F_{\theta_{t}}(w_{t- \tau_{\alpha_t}})}  \\
      \leq &\norm{\nabla M(w_{t-\tau_{\alpha_t}} - w^*_{\theta_{t-\tau_{\alpha_t}}})}_s^* \norm{\bar F_{\theta_{t}}(w_{t- \tau_{\alpha_t}}) - \bar F_{\theta_{t - \tau_{\alpha_t}}}(w_{t- \tau_{\alpha_t}})}_s  \\
      \leq & \frac{2L(1 + L_wL_\theta \beta_{t-\tau_{\alpha_t}, t-1})}{\xi l_{cs}^2}\left(\norm{w^*_{\theta_t}}_c + \frac{B}{A} + \norm{w_t - w^*_{\theta_t}}_c + 1\right) \\
      &\times L_F'' L_\theta \beta_{t-\tau_{\alpha_t}, t-1} \left(\norm{w_{t-\tau_{\alpha_t}}}_c + U_F''\right) \qq{(Using \eqref{eq gradient bound dual norm} and Assumption \ref{assu regularization})}.
  \end{align}
  Using \eqref{eq tmp 3} completes the proof.
\end{proof}

\subsection{Proof of Lemma~\ref{lem bound t4}}
\label{sec proof lem bound t4}
\lemboundtfour*
\begin{proof}
  \begin{align}
    &\E\left[T_4\right] \\
    =& \E\left[\indot{\nabla M(w_t - w^*_{\theta_t})}{\epsilon_t} \right] \\
    =& \E\left[\E\left[\indot{\nabla M(w_t - w^*_{\theta_t})}{\epsilon_t} \mid \fF_t \right] \right] \qq{(Tower law of expectation)} \\
    =& \E\left[\indot{\nabla M(w_t - w^*_{\theta_t})}{\E\left[\epsilon_t \mid \fF_t \right]} \right] \qq{(Conditional independence)} \\
    =&0\qq{(Assumption~\ref{assu mds})}
  \end{align}
\end{proof}

\subsection{Proof of Lemma~\ref{lem bound t5}}
\label{sec proof lem bound t5}
\lemboundtfive*
\begin{proof}
  \begin{align}
      T_5 =&\frac{L }{\xi}\norm{F_{\theta_t}(w_t, Y_t) - w_t + \epsilon_t}_s^2 \\
      \leq& \frac{L }{\xi l_{cs}^2}\norm{F_{\theta_t}(w_t, Y_t) - w_t + \epsilon_t}_c^2  \\
      \leq& \frac{L^2 }{\xi l_{cs}}\left(\norm{F_{\theta_t}(w_t, Y_t)}_c + \norm{w_t}_c + \norm{\epsilon_t}_c \right)^2  \\
      \leq& \frac{L }{\xi l_{cs}^2}\left(U_F + (L_F + 1) \norm{w_t}_c + U_\epsilon \norm{w_t}_c + U_\epsilon'\right)^2 \qq{(Lemma~\ref{lem bound of fxy} and Assumption~\ref{assu mds})} \\
      \leq& \frac{L }{\xi l_{cs}^2}\left(B + A \norm{w_t}_c\right)^2  \\
      \leq& \frac{L }{\xi l_{cs}^2}\left(B + A \norm{w_t - w^*_{\theta_t}}_c + A\norm{w^*_{\theta_t}}_c\right)^2  \\
      \leq& \frac{2L}{\xi l_{cs}^2}\left(A^2 u_{cm}^2 \norm{w_t - w^*_{\theta_t}}_m^2 + C^2\right)
  \end{align}
\end{proof}

\begin{lemma}
  \label{lem bound of fxy}
  For any time step $t$, almost surely,
  \begin{align}
      \norm{F_{\theta_t}(w, y)}_c \leq U_F + L_F \norm{w}_c
  \end{align}
\end{lemma}
\begin{proof}
  Assumption \ref{assu regularization} implies that
  \begin{align}
      \norm{F_{\theta_t}(w, y)}_c - \norm{F_{\theta_t}(0, y)}_c  &\leq \norm{F_{\theta_t}(0, y) - F_{\theta_t}(w, y)}_c \\
      &\leq L_F \norm{w - 0}_c,
  \end{align}
  which completes the proof.
\end{proof}

\begin{lemma}
  \label{lem chain difference bound}
  \begin{align}
      &\norm{ \E \left[{F_{\theta_{t - \tau_{\alpha_t}}}(w_{t - \tau_{\alpha_t}}, Y_t) -F_{\theta_{t - \tau_{\alpha_t}}}(w_{t - \tau_{\alpha_t}}, \tilde Y_t)} \mid \substack{w_{t-\tau_{\alpha_t}} \\ \theta_{t-\tau_{\alpha_t}} \\ Y_{t-\tau_{\alpha_t}}} \right]} \\
      \leq &2\ny L_PL_\theta \sum_{j=t-\tau_{\alpha_t}}^{t-1}\beta_{t-\tau_{\alpha_t}, j}(B + A \norm{w_t - w^*_{\theta_t}}_c + A \norm{w^*_{\theta_t}}_c)
  \end{align}
\end{lemma}
\begin{proof}
  In this proof, all $\Pr$ and $\E$ are implicitly conditioned on $w_{t-\tau_{\alpha_t}}, \theta_{t-\tau_{\alpha_t}}, Y_{t-\tau_{\alpha_t}}$. 
  We use $\Theta_t$ to denote the set of all possible $\theta_t$.
  \begin{align}
      &\Pr(Y_t = y') \\
      =& \sum_{y} \int_{\Theta_{t}} \Pr(Y_t = y' , Y_{t-1} = y, \theta_{t} = z) dz \\
      =& \sum_{y} \int_{\Theta_{t}} \Pr(Y_t = y' \mid Y_{t-1} = y, \theta_{t} = z) \Pr(Y_{t-1} = y, \theta_{t} = z) dz \\
      =& \sum_{y} \int_{\Theta_{t}} P_{z}(y, y') \Pr(Y_{t-1} = y) \Pr(\theta_{t} = z | Y_{t-1} = y) dz
  \end{align}
  \begin{align}
      &\Pr(\tilde Y_t = y') \\
      =& \sum_{y} \Pr(\tilde Y_{t-1} = y) P_{\theta_{t-\tau_{\alpha_t}}}(y, y') \\
      =& \sum_{y} \Pr(\tilde Y_{t-1} = y) P_{\theta_{t-\tau_{\alpha_t}}}(y, y') \int_{\Theta_{t}} \Pr(\theta_{t} = z | Y_{t-1} = y) dz \\
      =& \sum_{y} \int_{\Theta_{t}} \Pr(\tilde Y_{t-1} = y) P_{\theta_{t-\tau_{\alpha_t}}}(y, y')  \Pr(\theta_{t} = z | Y_{t-1} = y) dz \\
  \end{align}
  Consequently,
  \begin{align}
      &\sum_{y'} \left|\Pr(Y_t = y') - \Pr(\tilde Y_t = y') \right| \\
      \leq &\sum_{y, y'} \int_{\Theta_{t}} \left| \Pr(Y_{t-1} = y) P_z(y, y') - \Pr(\tilde Y_{t-1} = y) P_{\theta_{t-\tau_{\alpha_t}}}(y, y')\right| \Pr(\theta_{t} = z \mid Y_{t-1} = y) dz.
  \end{align}
Since for any $z \in \Theta_{t}$,
\begin{align}
  &\left| \Pr(Y_{t-1} = y) P_z(y, y') - \Pr(\tilde Y_{t-1} = y) P_{\theta_{t-\tau_{\alpha_t}}}(y, y')\right| \\
  \leq & \left| \Pr(Y_{t-1} = y) P_z(y, y') - \Pr(\tilde Y_{t-1} = y) P_{z}(y, y')\right| \\
  &+ \left| \Pr(\tilde Y_{t-1} = y) P_z(y, y') - \Pr(\tilde Y_{t-1} = y) P_{\theta_{t-\tau_{\alpha_t}}}(y, y')\right| \\
  \leq & \left| \Pr(Y_{t-1} = y) - \Pr(\tilde Y_{t-1} = y) \right| P_z(y, y') + L_PL_\theta \beta_{t-\tau_{\alpha_t}, t-1} \Pr(\tilde Y_{t-1} = y),
\end{align}
we have
\begin{align}
  &\sum_{y'} \left|\Pr(Y_t = y') - \Pr(\tilde Y_t = y') \right| \\
  \leq & \sum_y \left| \Pr(Y_{t-1} = y) - \Pr(\tilde Y_{t-1} = y) \right| + \ny L_PL_\theta \beta_{t-\tau_{\alpha_t}, t-1}.
\end{align}
Applying the above inequality recursively yields
\begin{align}
  \label{eq y difference two chains}
  \sum_{y'} \left|\Pr(Y_t = y') - \Pr(\tilde Y_t = y') \right| \leq \ny L_PL_\theta \sum_{j=t-\tau_{\alpha_t}}^{t-1} \beta_{t-\tau_{\alpha_t}, j}.
\end{align}
Consequently,
\begin{align}
  &\norm{ \E \left[{F_{\theta_{t - \tau_{\alpha_t}}}(w_{t - \tau_{\alpha_t}}, Y_t) -F_{\theta_{t - \tau_{\alpha_t}}}(w_{t - \tau_{\alpha_t}}, \tilde Y_t)} \right]}_c \\
  =& \norm{\sum_y \left(\Pr(Y_t = y) - \Pr(\tilde Y_t = y)\right) F_{\theta_{t-\tau_{\alpha_t}}}(w_{t-\tau_{\alpha_t}}, y)}_c \\
  \leq &\max_y \norm{F_{\theta_{t-\tau_{\alpha_t}}}(w_{t-\tau_{\alpha_t}}, y)}_c \ny L_PL_\theta \sum_{j=t-\tau_{\alpha_t}}^{t-1}\beta_{t-\tau_{\alpha_t}, j} \\
  \leq & 2\ny L_PL_\theta \sum_{j=t-\tau_{\alpha_t}}^{t-1}\beta_{t-\tau_{\alpha_t}, j}(B + A \norm{w_t - w^*_{\theta_t}}_c + A \norm{w^*_{\theta_t}}_c) \qq{(Using \eqref{eq bound of xk-tk}),}
\end{align}
which completes the proof.
\end{proof}

\subsection{Proof of Lemma~\ref{lem learning rates}}
\label{sec proof lem learning rates}
\lemlearningrates*
\begin{proof}
  By the definition of $\tau_{\alpha_t}$ in \eqref{eq definition of tau alpha t},
  it is easy to see
  \begin{align}
    \tau_{\alpha_t} = \ceil*{\frac{\log \alpha_t - \log C_0}{\log \tau}} = \fO(\log (t+t_0)),
  \end{align}
  where $\ceil{\cdot}$ is the ceiling function.
  Consequently,
  \begin{align}
    \alpha_{t-{\tau_{\alpha_t}}, t-1} \leq \tau_{\alpha_t} \alpha_{t-\tau_{\alpha_t}} = \fO\left(\frac{\log (t+t_0)}{(t+t_0 - \tau_{\alpha_t})^{\epsilon_\alpha}}\right) = \fO\left(\frac{\log (t+t_0)}{(t+t_0)^{\epsilon_\alpha}}\right),
  \end{align}
  implying
  \begin{align}
    \frac{\alpha_t \alpha_{t-{\tau_{\alpha_t}}, t-1}}{\beta_t} = \fO\left(\frac{\log (t+t_0)}{(t+t_0)^{2\epsilon_\alpha - \epsilon_\beta}}\right).
  \end{align}
  Assumption~\ref{assu twotimescale} ensures $\beta_t < \alpha_t$ holds for all $t$.
  Consequently,
  \begin{align}
    \beta_{t-{\tau_{\alpha_t}}, t-1} < \alpha_{t-{\tau_{\alpha_t}}, t-1} = \fO\left(\frac{\log (t+t_0)}{(t+t_0)^{\epsilon_\alpha}}\right),
  \end{align}
  which completes the proof.
\end{proof}

\subsection{Proof of Lemma~\ref{lem sa error bound 1}}
\label{sec proof lem sa error bound 1}
\lemsaerrorboundone*
\begin{proof}
  According to \eqref{eq sa iterates},
  we have
  \begin{align}
    \norm{w_{t+1}}_c \leq& \norm{w_t}_c + \alpha_t \left( \norm{F_{\theta_t}(w_t, Y_t)}_c + \norm{w_t}_c + \norm{\epsilon_t}_c \right) \\
    \leq& \norm{w_t}_c + \alpha_t \left(U_F + L_F \norm{w_t}_c + \norm{w_t}_c + U_\epsilon \norm{w_t}_c + U_\epsilon' \right) \\
    \intertext{\hfill(Lemma~\ref{lem bound of fxy} and Assumption~\ref{assu mds}).}
  \end{align}
  Consequently,
  it is easy to see that there exists a constant $C_{t_0, w_0}$ such that
  for all $t \leq t_0$,
  \begin{align}
    \E\left[\norm{w_{t} - w^*_{\theta_{t}}}_m^2\right] \leq C_{t_0, w_0}.
  \end{align}
\end{proof}

\subsection{Proof of Lemma~\ref{lem bound of m11}}
\label{sec proof lem bound of m11}
\lemboundofmoneone*
\begin{proof}
  \begin{align}
    &M_{11} \\
    =& \sum_{s, a} \left(d_{\mu_{\theta_t}}(s) \pi_{\theta_t}(a|s) \nabla \log \pi_{\theta_t}(a|s) q_{\pi_{\theta_t}}(s, a) +  \frac{\lambda_t}{\na} d_{\mu_{\theta_t}}(s) \nabla \log \pi_{\theta_t}(a|s) \right)^\top \nabla J_{\lambda_t}(\theta_t) \\
    =& \sum_{s', a'} \sum_{s, a} \left(d_{\mu_{\theta_t}}(s) \dv{\pi_{\theta_t}(a | s)}{\theta_{s',a'}} q_{\pi_{\theta_t}}(s, a) + \frac{\lambda_t}{\na} d_{\mu_{\theta_t}}(s) \dv{\log \pi_{\theta_t} (a|s)}{\theta_{s',a'}}   \right)  \dv{J_{\lambda_t}(\theta_t)}{\theta_{s',a'}} \\
    =& \sum_{s', a'} \left( d_{\mu_{\theta_t}}(s') \pi_{\theta_t}(a'|s') \adv_{\pi_{\theta_t}}(s', a') + \frac{\lambda_t}{\na} d_{\mu_{\theta_t}}(s')  (1 -  \na\pi_{\theta_t}(a'|s')) \right) \dv{J_{\lambda_t}(\theta_t)}{\theta_{s',a'}} \\
    \intertext{\hfill (Lemma~\ref{lem softmax policy gradient})}
    =& \sum_{s, a} \left( d_{\mu_{\theta_t}}(s) \pi_{\theta_t}(a|s) \adv_{\pi_{\theta_t}}(s, a) + \lambda_t d_{\mu_{\theta_t}}(s) (\frac{1}{\na} - \pi_{\theta_t}(a|s)) \right) \\
    &\times \left(\underbrace{\frac{1}{1 - \gamma} d_{\pi_{\theta_t}, \gamma}(s) \pi_{\theta_t}(a|s) \adv_{\pi_{\theta_t}}(s, a)}_{M_{111}} + \underbrace{\frac{\lambda_t}{\ns} (\frac{1}{\na} - \pi_{\theta_t}(a | s))}_{M_{112}}\right) \\
    \intertext{\hfill (Lemma~\ref{lem softmax policy gradient})}
    =& \sum_{s, a} \left( \frac{d_{\mu_{\theta_t}}(s)(1-\gamma)}{d_{\pi_{\theta_t}, \gamma}(s)}M_{111} + d_{\mu_{\theta_t}}(s) \ns M_{112} \right)(M_{111} + M_{112}) \\
    =& \sum_{s, a} \frac{d_{\mu_{\theta_t}}(s)(1-\gamma)}{d_{\pi_{\theta_t}, \gamma}(s)}M_{111}^2 + d_{\mu_{\theta_t}}(s) \ns M_{112}^2 + \left(\frac{d_{\mu_{\theta_t}}(s)(1-\gamma)}{d_{\pi_{\theta_t}, \gamma}(s)} + d_{\mu_{\theta_t}}(s) \ns\right) M_{111}M_{112} \\
    \geq& \sum_{s, a} \underbrace{\chi_{11} M_{111}^2 +  \chi_{12} M_{112}^2 + \left(\frac{d_{\mu_{\theta_t}}(s)(1-\gamma)}{d_{\pi_{\theta_t}, \gamma}(s)} + d_{\mu_{\theta_t}}(s) \ns\right) M_{111}M_{112}}_{M_{113}} \\
\end{align}
where 
\begin{align}
    \chi_{11} &\doteq \inf_{\theta, s} \frac{d_{\mu_{\theta}}(s) (1 - \gamma)}{d_{\pi_{\theta}, \gamma}(s)}, \\
    \chi_{12} &\doteq \inf_{\theta, s} {d_{\mu_{\theta}}(s) \ns}.
\end{align}
Assumption~\ref{assu mu uniform ergodicity},
the continuity of $d_{\mu_\theta}$ w.r.t. $\theta$ (Lemma~\ref{lem continuity of ergodic distribution}),
and the extreme value theorem ensures that
\begin{align}
  \chi_{11} > 0, \, \chi_{12} > 0.
\end{align}
If $M_{111}M_{112} < 0$,
then 
\begin{align}
  M_{113} \geq \chi_{11} M_{111}^2 +  \chi_{12} M_{112}^2 \geq \frac{\min\qty{\chi_{11}, \chi_{12}}}{2}(M_{111}+M_{112})^2.
\end{align}
If $M_{111}M_{112} \geq 0$,
then
\begin{align}
  M_{113} \geq \chi_{11} M_{111}^2 +  \chi_{12} M_{112}^2 + (\chi_{11} + \chi_{12})M_{111}M_{112} \geq \min\qty{\chi_{11}, \chi_{12}} (M_{111}+M_{112})^2
\end{align}
Let 
\begin{align}
  \chi_1 \doteq \frac{\min\qty{\chi_{11}, \chi_{12}}}{2} > 0,
\end{align}
then we always have
\begin{align}
  M_{113} \geq \chi_1 (M_{111}+M_{112})^2,
\end{align}
implying
\begin{align}
  M_{11} &\geq \chi_1 \sum_{s,a} (M_{111} + M_{112})^2 \\
  &= \chi_1 \norm{\nabla J_{\lambda_t}(\theta_t)}^2 \qq{(Lemma~\ref{lem softmax policy gradient})},
\end{align}
which completes the proof.
\end{proof}

\subsection{Proof of Lemma~\ref{lem bound of m121}}
\label{sec proof lem bound of m121}
\lemboundofmonetwoone*
\begin{proof}
  We first study the Lipschitz continuity of $\Lambda(\theta, y, \eta)$ defined in \eqref{eq actor helper function}.
  As shown in the verification of Assumption~\ref{assu regularization} (v) in Section~\ref{sec proof prop critic convergence},
  $q_{\pi_\theta}$ is Lipschitz continuous in $\theta$ and bounded.
  According to Lemma~\ref{lem softmax policy gradient},
  it is easy to see $\nabla \log \pi_\theta(a|s)$ is also Lipschitz continuous in $\theta$ and bounded.
  Assumption~\ref{assu mu uniform ergodicity} ensures that $\inf_{\theta, a, s} \mu_\theta(a|s) > 0$,
  hence it is easy to see $\frac{\pi_\theta(a|s)}{\mu_\theta(a|s)}$ is also Lipschitz continuous and bounded from above.
  We, therefore,
  conclude via Lemma~\ref{lem product of lipschitz functions} that there exist continuous functions $L_\Lambda(\eta)$ and $U_{\Lambda}(\eta)$ such that
  for any $y$,
  \begin{align}
    \norm{\Lambda(\theta, y, \eta) - \Lambda(\theta', y, \eta)} &\leq L_{\Lambda}(\eta) \norm{\theta -\theta'}, \\
    \sup_{\theta} \norm{\Lambda(\theta, y, \eta)} &\leq U_{\Lambda}(\eta).
  \end{align}

  We now study the Lipschitz continuity of $\bar \Lambda(\theta, \eta)$ defined in \eqref{eq actor helper function}.
  Lemma~\ref{lem continuity of ergodic distribution} confirms the Lipschitz continuity of $d_{\mu_\theta}$.
  Consequently,
  Lemma~\ref{lem product of lipschitz functions} implies that 
  there exist continuous functions $L_{\bar \Lambda}(\eta)$ and $U_{\bar \Lambda}(\eta)$ such that
  \begin{align}
    \norm{\bar \Lambda(\theta, \eta) - \bar \Lambda(\theta', \eta)} &\leq L_{\bar \Lambda}(\eta) \norm{\theta -\theta'}, \\
    \sup_{\theta} \norm{\bar \Lambda(\theta, y, \eta)} &\leq U_{\bar \Lambda}(\eta).
  \end{align}

  We now study the Lipschitz continuity of $\Lambda'(\theta, y, \eta)$ defined in \eqref{eq actor helper function 2}.
  Since $J_\eta(\theta)$ is $L_J + \eta L_\text{KL}$ smooth,
  Lemma~\ref{lem smooth definition} implies that
  $L_J + \eta L_\text{KL}$ is a Lipschitz constant of $\nabla J_\eta(\theta)$.
  From Lemma~\ref{lem softmax policy gradient},
  it is easy to see the upper bound of $\nabla J_\eta(\theta)$ is also a continuous function of $\eta$.
  Consequently,
  Lemma~\ref{lem product of lipschitz functions} implies
  there exist continuous functions $L_{\Lambda'}(\eta)$ and $U_{\Lambda'}(\eta)$ such that
  for all $y$,
  \begin{align}
    \label{eq lambda prime constants}
    \norm{\Lambda'(\theta, y, \eta) - \Lambda'(\theta', y, \eta)} &\leq L_{\Lambda'}(\eta) \norm{\theta -\theta'}, \\
    \sup_\theta \norm{\Lambda'(\theta, y, \eta)} &\leq U_{\Lambda'}(\eta).
  \end{align}
  Hence
  \begin{align}
    \norm{M_{121}} &= \norm{\Lambda'(\theta_t, Y_t, \lambda_t) - \Lambda'(\theta_{t-\tau_{\beta_t}}, Y_t, \lambda_t)} \\
    &\leq L_{\Lambda'}(\lambda_t) \norm{\theta_t - \theta_{t-\tau_{\beta_t}}} \\
    &\leq \frac{L_{\Lambda'}(\lambda_t) L_\theta}{l_{2, p}} \beta_{t-\tau_{\beta_t}, t-1} \qq{(Using \eqref{eq ltheta})}.
  \end{align}
  Since $\lambda_t \in [0, \lambda]$, 
  $L_{\Lambda'}(\eta)$ is a continuous function and well defined in $[0, \lambda]$,
  the extreme value theorem asserts that $L_{\Lambda'}(\eta)$ obtains its maximum in $[0, \lambda]$,
  say, e.g., $L_{\Lambda'}^*$.
  Then 
  \begin{align}
    \norm{M_{121}} \leq \frac{L_{\Lambda'}^* L_\theta}{l_{2, p}} \beta_{t-\tau_{\beta_t}, t-1}.
  \end{align}
\end{proof}

\subsection{Proof of Lemma~\ref{lem bound of m122}}
\label{sec proof lem bound of m122}
\lemboundofmonetwotwo*
\begin{proof}
    \begin{align}
      &\norm{\E\left[M_{122}\right]} \\
      =&\norm{\E\left[\E\left[M_{122} \mid \theta_{t-\tau_{\beta_t}}, Y_{t-\tau_{\beta_t}}\right]\right]} \\
      \leq &\E\left[\norm{\E\left[M_{122} \mid \theta_{t-\tau_{\beta_t}}, Y_{t-\tau_{\beta_t}}\right]}\right].
    \end{align}
    We now bound the inner expectation.
    In the rest of the proof,
    all $\Pr$ and $\E$ are implicitly conditioned on $\theta_{t-\tau_{\beta_t}}$ and $Y_{t-\tau_{\beta_t}}$. 
    \begin{align}
        &\norm{\E\left[M_{122}\right]} \\
        =&\norm{\E\left[\Lambda'(\theta_{t-\tau_{\beta_t}}, Y_t, \lambda_t) - \Lambda'(\theta_{t-\tau_{\beta_t}}, \tilde Y_t, \lambda_t)\right]} \\
        = &\norm{\sum_{y} \left(\Pr(\tilde Y_t = y) - \Pr(Y_t = y)\right) \Lambda'(\theta_{t-\tau_{\beta_t}}, y, \lambda_t)} \\
        \leq &\max_y \norm{ \Lambda'(\theta_{t-\tau_{\beta_t}}, y, \lambda_t)}\sum_{y} \left|\Pr(\tilde Y_t = y) - \Pr(Y_t = y)\right| \\
        \leq &U_{\Lambda'}(\lambda_t)\sum_{y} \left|\Pr(\tilde Y_t = y) - \Pr(Y_t = y)\right| \qq{(Using \eqref{eq lambda prime constants})} \\
        \leq &U_{\Lambda'}(\lambda_t) \ns \na L_\mu L_\theta \sum_{j=t-\tau_{\beta_t}}^{t-1} \beta_{t-\tau_{\beta_t}, j} \qq{(Similar to \eqref{eq y difference two chains} with $L_\theta$ defined in \eqref{eq ltheta})}.
    \end{align}
    Since $\lambda_t \in [0, \lambda]$ and the continuous function $U_{\Lambda'}(\eta)$ obtains its maximum, say, e.g., $U_{\Lambda'}^*$, in the compact set $[0, \lambda]$,
    we have
    \begin{align}
      \norm{\E\left[M_{122}\right]} \leq U_{\Lambda'}^* \ns \na L_\mu L_\theta \sum_{j=t-\tau_{\beta_t}}^{t-1} \beta_{t-\tau_{\beta_t}, j},
    \end{align}
    which completes the proof.
\end{proof}

\subsection{Proof of Lemma~\ref{lem bound of m123}}
\label{sec proof lem bound of m123}
\lemboundofmonetwothree*
\begin{proof}
  \begin{align}
    &\norm{\E\left[M_{123}\right]} \\
    =&\norm{\E\left[\Lambda'(\theta_{t-\tau_{\beta_t}}, \tilde Y_t, \lambda_t)\right]} \\
    =&\norm{\E\left[\E\left[\Lambda'(\theta_{t-\tau_{\beta_t}}, \tilde Y_t, \lambda_t)\mid \theta_{t-\tau_{\beta_t}}, Y_{t-\tau_{\beta_t}}\right]\right]} \\
    \leq &\E\left[\norm{\E\left[\Lambda'(\theta_{t-\tau_{\beta_t}}, \tilde Y_t, \lambda_t)\mid \theta_{t-\tau_{\beta_t}}, Y_{t-\tau_{\beta_t}}\right]}\right].
  \end{align}
  We now bound the inner expectation.
  In the rest of the proof, 
  all $\Pr$ and $\E$ are implicitly conditioned on $\theta_{t-\tau_{\beta_t}}$ and $Y_{t-\tau_{\beta_t}}$. 
  Since $\tilde Y_t = (\tilde S_t, \tilde A_t)$ and 
  \begin{align}
    \sum_{s, a} d_{\mu_{\theta_{t-\tau_{\beta_t}}}}(s) \mu_{\theta_{t-\tau_{\beta_t}}}(a|s) \Lambda'(\theta_{t-\tau_{\beta_t}}, (s, a), \lambda_t) = 0,
  \end{align}
  we have
  \begin{align}
    &\norm{\E\left[\Lambda'(\theta_{t-\tau_{\beta_t}}, \tilde Y_t, \lambda_t)\right]} \\
    =& \norm{\sum_{s, a} \left(\Pr(\tilde S_t = s, \tilde A_t = a) - d_{\mu_{\theta_{t-\tau_{\beta_t}}}}(s) \mu_{\theta_{t-\tau_{\beta_t}}}(a|s)\right) \Lambda'(\theta_{t-\tau_{\beta_t}}, (s, a), \lambda_t)} \\
    \leq& \sup_{s, a, \theta} \norm{\Lambda'(\theta, (s, a), \lambda_t)} \sum_{s, a} \abs{\Pr(\tilde S_t = s, \tilde A_t = a) - d_{\mu_{\theta_{t-\tau_{\beta_t}}}}(s) \mu_{\theta_{t-\tau_{\beta_t}}}(a|s)} \\
    \leq& U_{\Lambda'}^* \beta_t \qq{(Using \eqref{eq definition of beta t})},
  \end{align}
  which completes the proof.
\end{proof}

\subsection{Proof of Lemma~\ref{lem bound of m13}}
\label{sec proof lem bound of m13}
\lemboundofmonethree*
\begin{proof}
  \begin{align}
    &\norm{\E\left[M_{13}\right]} \\
    =&\norm{\E\left[\indot{\nabla J_{\lambda_t}(\theta_t)}{ \rho_t  \nabla \log \pi_{\theta_t}(A_t | S_t) \left(q_t(S_t, A_t) - q_{\pi_{\theta_t}}(S_t, A_t) \right) }\right]} \\
    \leq & \sum_{s, a} \E\left[\abs{\dv{J_{\lambda_t}(\theta_t)}{\theta_{s, a}}\rho_t  \dv{\log \pi_{\theta_t}(A_t | S_t)}{\theta_{s, a}}  \left(q_t(S_t, A_t) - q_{\pi_{\theta_t}}(S_t, A_t) \right)}\right] \\
    \leq& \sum_{s, a} \sqrt{{\E\left[\left(\dv{J_{\lambda_t}(\theta_t)}{\theta_{s, a}}\right)^2\right]}{\E\left[\left(\rho_t  \dv{\log \pi_{\theta_t}(A_t | S_t)}{\theta_{s, a}}\right)^2  \left(q_t(S_t, A_t) - q_{\pi_{\theta_t}}(S_t, A_t) \right)^2 \right]}},
    \intertext{\hfill (Cauchy-Schwarz inequality)}
  \end{align}
  Lemma~\ref{lem softmax policy gradient} implies that
  \begin{align}
    \abs{\dv{\log \pi_\theta(a|s)}{\theta_{s', a'}}} < 2.
  \end{align}
  Assumption~\ref{assu mu uniform ergodicity} implies that
  \begin{align}
    \rho_{max} \doteq \sup_{\theta, s, a} \frac{\pi_{\theta}(a|s)}{\mu_\theta(a|s)} < \infty.
  \end{align}
  Hence
  \begin{align}
    &\norm{\E\left[M_{13}\right]} \\
    \leq& 2\rho_{max} \sum_{s, a} \sqrt{{\E\left[\left(\dv{J_{\lambda_t}(\theta_t)}{\theta_{s, a}}\right)^2\right]}{\E\left[\left(q_t(S_t, A_t) - q_{\pi_{\theta_t}}(S_t, A_t) \right)^2 \right]}} \\
    \leq& 2\rho_{max} \sum_{s, a} \sqrt{{\E\left[\left(\dv{J_{\lambda_t}(\theta_t)}{\theta_{s, a}}\right)^2\right]}{\E\left[\norm{q_t - q_{\pi_{\theta_t}}}_\infty^2 \right]}} \\
    \leq& 2\rho_{max} \sqrt{\E\left[\norm{q_t - q_{\pi_{\theta_t}}}_\infty^2 \right]} \sum_{s, a} \left( \sqrt{{\E\left[\left(\dv{J_{\lambda_t}(\theta_t)}{\theta_{s, a}}\right)^2\right]}{}} \times 1 \right) \\
    \leq& 2\rho_{max} \sqrt{\nsa}\sqrt{\E\left[\norm{q_t - q_{\pi_{\theta_t}}}_\infty^2 \right]} \sqrt{\E\left[\norm{\nabla J_{\lambda_t}(\theta_t)}^2\right]},
    \intertext{\hfill (Cauchy-Schwarz inequality)}
  \end{align}
  which completes the proof.
\end{proof}

\subsection{Proof of Lemma~\ref{lem bound j lambda}}
\label{sec proof lem bound j lambda}
\lemboundjlambda*
\begin{proof}
  Lemma~\ref{lem softmax policy gradient} implies that 
  \begin{align}
    \abs{\dv{J_{\lambda_t}(\theta_t)}{\theta_{s,a}}} \leq \abs{\frac{1}{1 - \gamma} d_{\pi_\theta, \gamma, p_0}(s) \pi_\theta(a|s) {\adv_{\pi_\theta}(s, a)}} + \frac{\lambda_t}{\ns} \abs{\pi_\theta(a | s) - \frac{1}{\na}}.
  \end{align}
  Since $\lambda_t < \lambda$,
  we conclude that
  there exists a constant $\chi_6$ (depending on $\lambda$) such that $\forall t, \theta$
  \begin{align}
    \norm{\nabla J_{\lambda_t}(\theta)}^2 \leq \chi_6.
  \end{align}
  Then \eqref{eq actor perf recursive bound} and Proposition \ref{prop critic convergence} imply that there exists some constant $\chi_7 > 0$ such that
\begin{align}
  \E\left[J_{\lambda_t}(\theta_{t+1})\right] \geq& \E\left[J_{\lambda_t}(\theta_t)\right] + \beta_t \chi_{11} \E\left[\norm{\nabla J_{\lambda_t}(\theta_t)}^2\right] \\
  &-\underbrace{\left(\beta_t \chi_{12} \frac{\log^2(t+t_0)}{(t+t_0)^{\epsilon_\beta}} +\beta_t \chi_{7} t^{-\frac{\epsilon_q}{2}} \sqrt{\chi_6} + \beta_t \chi_2 \frac{1}{(t+t_0)^{\epsilon_\beta}}\right)}_{z_t}.
\end{align}
Hence 
\begin{align}
  &\E\left[J_{\lambda_{t+1}}(\theta_{t+1})\right] \\
  \geq& \E\left[J_{\lambda_t}(\theta_t)\right] + \beta_t \chi_{11} \E\left[\norm{\nabla J_{\lambda_t}(\theta_t)}^2\right] + \E\left[J_{\lambda_{t+1}}(\theta_{t+1})\right] - \E\left[J_{\lambda_t}(\theta_{t+1})\right] - z_t \\
  = & \E\left[J_{\lambda_t}(\theta_t)\right] + \beta_t \chi_{11} \E\left[\norm{\nabla J_{\lambda_t}(\theta_t)}^2\right] + (\lambda_t - \lambda_{t+1}) \E_{s\sim\fU_\fS}\left[\kl{\fU_\fA}{\pi_{\theta_{t+1}}(\cdot|s)}\right] - z_t \\
  \intertext{\hfill (Using \eqref{eq definition of new objective})}
  \geq & \E\left[J_{\lambda_t}(\theta_t)\right] - z_t \qq{(Using $\lambda_t > \lambda_{t+1}$)}.
\end{align}
Telescoping the above inequality yields
\begin{align}
  \E\left[J_{\lambda_t}(\theta_t)\right] \geq \E\left[J_{\lambda_0}(\theta_0)\right] - \sum_{k=0}^t z_k \geq \E\left[J_{\lambda_0}(\theta_0)\right] - \sum_{k=0}^\infty z_k. 
\end{align}
Since $\epsilon_\beta > 0.5$,
we have
\begin{align}
  \sum_{t=0}^\infty \beta_t \frac{\log^2(t+t_0)}{(t+t_0)^{\epsilon_\beta}} = \sum_{t=0}^\infty \frac{\beta \log^2(t+t_0)}{(t+t_0)^{2\epsilon_\beta}} &< \infty, \\
  \sum_{t=0}^\infty \beta_t \frac{1}{(t+t_0)^{\epsilon_\beta}} = \sum_{t=0}^\infty \frac{\beta }{(t+t_0)^{2\epsilon_\beta}} &< \infty.
\end{align}
Since $\epsilon_q > 2(1 - \epsilon_\beta)$,
we have
\begin{align}
  \sum_{t=0}^\infty \beta_t t^{-\frac{\epsilon_q}{2}} < \sum_{t=0} \frac{\beta}{t^{\epsilon_\beta + \frac{\epsilon_q}{2}}} < \infty.
\end{align}
We, therefore, conclude that
\begin{align}
  \sum_{t=0}^{\infty} z_t < \infty, 
\end{align}
implying
$\E\left[J_{\lambda_{t+1}}(\theta_{t+1})\right]$ is bounded from the below by some constant.
By \eqref{eq definition of new objective},
\begin{align}
  \E\left[J_{\lambda_{t+1}}(\theta_{t+1})\right] \leq \frac{r_{max}}{1 - \gamma},
\end{align}
we, therefore, conclude that $\abs{\E\left[J_{\lambda_{t+1}}(\theta_{t+1})\right]}$ is bounded by some constant.
Similarly, we have
\begin{align}
  &\E\left[J_{\lambda_{t}}(\theta_{t+1})\right] \\
  \geq& \E\left[J_{\lambda_{t-1}}(\theta_t)\right] + \beta_t \chi_{11} \E\left[\norm{\nabla J_{\lambda_t}(\theta_t)}^2\right] + \E\left[J_{\lambda_{t}}(\theta_{t})\right] - \E\left[J_{\lambda_{t-1}}(\theta_{t})\right] - z_t \\
  = & \E\left[J_{\lambda_{t-1}}(\theta_t)\right] + \beta_t \chi_{11} \E\left[\norm{\nabla J_{\lambda_t}(\theta_t)}^2\right] + (\lambda_{t-1} - \lambda_{t}) \E_{s\sim\fU_\fS}\left[\kl{\fU_\fA}{\pi_{\theta_t}(\cdot|s)}\right] - z_t \\
  \geq & \E\left[J_{\lambda_{t-1}}(\theta_t)\right] -  z_t \qq{(Using $\lambda_{t-1} > \lambda_{t}$)}.
\end{align}
Hence
$\abs{\E\left[J_{\lambda_{t}}(\theta_{t+1})\right]}$ is also bounded,
which completes the proof.
\end{proof}

\subsection{Proof of Lemma~\ref{lem bound j lambda diff}}
\label{sec proof lem bound j lambda diff}
\lemboundjlambdadiff*
\begin{proof}
  \begin{align}
    &\E\left[\sum_{k=\ceil{\frac{t}{2}}}^t \left(\frac{1}{\beta_k}J_{\lambda_k}(\theta_{k+1}) - \frac{1}{\beta_k}J_{\lambda_k}(\theta_k) \right) \right]\\
    =& \E\left[\sum_{k=\ceil{\frac{t}{2}}}^t\left(\frac{1}{\beta_{k-1}}J_{\lambda_{k-1}}(\theta_k) - \frac{1}{\beta_k}J_{\lambda_k}(\theta_k) \right) + \frac{1}{\beta_t} J_{\lambda_t}(\theta_{t+1}) - \frac{1}{\beta_{\ceil{\frac{t}{2}}-1}} J_{\lambda_{\ceil{\frac{t}{2}}-1}}(\theta_{\ceil{\frac{t}{2}}}) \right] \\
    =&\E\left[ \sum_{k=\ceil{\frac{t}{2}}}^t\left(\frac{1}{\beta_{k-1}}J_{\lambda_{k-1}}(\theta_{k}) - \frac{1}{\beta_{k-1}} J_{\lambda_k}(\theta_k) + \frac{1}{\beta_{k-1}} J_{\lambda_k}(\theta_k) - \frac{1}{\beta_k}J_{\lambda_k}(\theta_k) \right) \right] \\
    &+ \E\left[\frac{1}{\beta_t} J_{\lambda_t}(\theta_{t+1}) - \frac{1}{\beta_{\ceil{\frac{t}{2}} - 1}} J_{\lambda_{\ceil{\frac{t}{2}}- 1}}(\theta_{\ceil{\frac{t}{2}}}) \right]\\
    =& \E\left[\sum_{k=\ceil{\frac{t}{2}}}^t\left(\frac{1}{\beta_{k-1}}(\lambda_k - \lambda_{k-1}) \E_{s\sim \fU_\fS}\left[\kl{\fU_\fA}{\pi_{\theta_k}(\cdot |s)}\right] + \left(\frac{1}{\beta_{k-1}} - \frac{1}{\beta_k}\right) J_{\lambda_k}(\theta_k) \right) \right]\\
    &+ \E\left[\frac{1}{\beta_t} J_{\lambda_t}(\theta_{t+1}) - \frac{1}{\beta_{\ceil{\frac{t}{2}}-1}} J_{\lambda_{\ceil{\frac{t}{2}} - 1}}(\theta_{\ceil{\frac{t}{2}}}) \right]\\
    \leq &\E\left[ \sum_{k=\ceil{\frac{t}{2}}}^t\left(\frac{1}{\beta_{k-1}} - \frac{1}{\beta_k}\right) J_{\lambda_k}(\theta_k)  + \frac{1}{\beta_t} J_{\lambda_t}(\theta_{t+1}) - \frac{1}{\beta_{\ceil{\frac{t}{2}}-1}} J_{\lambda_{\ceil{\frac{t}{2}} - 1}}(\theta_{\ceil{\frac{t}{2}}}) \right] \qq{(Using $\lambda_k < \lambda_{k-1}$)}, \\
    \leq & \sum_{k=\ceil{\frac{t}{2}}}^t\left(\frac{1}{\beta_k} - \frac{1}{\beta_{k-1}} \right) U_{J, \lambda}  + \frac{1}{\beta_t} U_{J, \lambda} + \frac{1}{\beta_{\ceil{\frac{t}{2}}-1}} U_{J, \lambda} \qq{\hfill (Using $\beta_{k-1} > \beta_k$ and Lemma~\ref{lem bound j lambda})} \\
    = & \frac{U_{J, \lambda}}{\beta_t} - \frac{U_{J, \lambda}}{\beta_{\ceil{\frac{t}{2}}-1}} + \frac{1}{\beta_t} U_{J, \lambda} + \frac{1}{\beta_{\ceil{\frac{t}{2}}-1}} U_{J, \lambda}  \\
    = & \frac{2U_{J, \lambda}}{\beta} (t+t_0)^{\epsilon_\beta}
\end{align}
\end{proof}

\subsection{Proof of Lemma~\ref{lem bound of m11 sac}}
\label{sec proof lem bound of m11 sac}
\lemboundofmoneonesac*
\begin{proof}
  \begin{align}
    &\tilde M_{11} \\
    =& \sum_{s, a} \left(d_{\mu_{\theta_t}}(s) \pi_{\theta_t}(a|s) \nabla \log \pi_{\theta_t}(a|s) \left(q_{\pi_{\theta_t}, \lambda_t}(s, a) - \lambda_t \log \pi_{\theta_t}(a|s) \right)  \right)^\top \nabla \tilde J_{\lambda_t}(\theta_t) \\
    =& \sum_{s', a'} \sum_{s, a} \left(d_{\mu_{\theta_t}}(s) \dv{\pi_{\theta_t}(a | s)}{\theta_{s',a'}} \left( q_{\pi_{\theta_t}, \lambda_t}(s, a) - \lambda_t \log \pi_{\theta_t}(a|s) \right) \right)  \dv{\tilde J_{\lambda_t}(\theta_t)}{\theta_{s',a'}} \\
    =& \sum_{s', a'} \left( d_{\mu_{\theta_t}}(s') \pi_{\theta_t}(a'|s') \tilde \adv_{\pi_{\theta_t}}(s', a')  \right) \dv{\tilde J_{\lambda_t}(\theta_t)}{\theta_{s',a'}} \qq{(Lemma~\ref{lem softmax policy gradient})} \\
    =& \sum_{s', a'}  \frac{(1 - \gamma) d_{\mu_{\theta_t}}(s')}{d_{\pi_{\theta_t}, \gamma}(s')} \left(\dv{\tilde J_{\lambda_t}(\theta_t)}{\theta_{s',a'}}\right)^2 \qq{(Lemma~\ref{lem softmax policy gradient})} \\
    \geq & \inf_{\theta, s} \frac{(1-\gamma)d_{\mu_\theta}(s)}{d_{\pi_\theta, \gamma}(s)} \norm{\nabla \tilde J_{\lambda_t}(\theta_t)}^2.
\end{align}
Assumption~\ref{assu mu uniform ergodicity},
the continuity of $d_{\mu_\theta}$ w.r.t. $\theta$ (Lemma~\ref{lem continuity of ergodic distribution}),
and the extreme value theorem ensures that
the above $\inf$ is strictly positive,
which completes the proof.
\end{proof}

\subsection{Proof of Lemma~\ref{lem bound of m121 sac}}
\label{sec proof lem bound of m121 sac}
\lemboundofmonetwoonesac*
\begin{proof}
  We first study the Lipschitz continuity of $\Lambda_1(\theta, s, \eta)$ defined in \eqref{eq actor helper function sac}.
  We have
  \begin{align}
    &\Lambda_1(\theta, s, \eta)\\
    =& \sum_a \pi_\theta(a|s) \nabla \log \pi_\theta(a|s) \left(\tilde q_{\pi_\theta, \eta}(s,a) - \eta \log \pi_\theta (a|s)\right) \\
    =& \sum_a \nabla \pi_\theta(a|s) \tilde q_{\pi_\theta, \eta}(s,a) - \eta \sum_a \nabla \pi_\theta(a|s) \log \pi_\theta (a|s) \\
    =& \sum_a \nabla \pi_\theta(a|s) \tilde q_{\pi_\theta, \eta}(s,a) - \eta \sum_a \nabla \pi_\theta(a|s) \log \pi_\theta (a|s) - \eta \nabla \sum_a \pi_\theta(a|s) \\
    =& \sum_a \nabla \pi_\theta(a|s) \tilde q_{\pi_\theta, \eta}(s,a) - \eta \sum_a \nabla \pi_\theta(a|s) \log \pi_\theta (a|s) - \eta \sum_a \pi_\theta(a|s) \nabla \log \pi_\theta(a|s) \\
    =& \sum_a \nabla \pi_\theta(a|s) \tilde q_{\pi_\theta, \eta}(s,a) + \eta \nabla \ent{\pi_\theta(\cdot | s)}.
  \end{align}
  From \eqref{eq tmp 12}, \eqref{eq tmp 9} and \eqref{eq tmp 10},
  it is easy to see that (1) $\tilde v_{\pi_\theta, \eta}(s)$, as well as $\tilde q_{\pi_\theta, \eta}(s, a)$,
  is Lipschitz continuous in $\theta$ with the Lipschitz constant being a continuous function of $\eta$; 
  (2) $\abs{\tilde v_{\pi_\theta, \eta}(s)}$, as well as $\abs{\tilde q_{\pi_\theta, \eta}(s, a)}$,
  is bounded with the bound being a continuous function of $\eta$.
  From Lemmas~\ref{lem softmax policy gradient} and~\ref{lem smooth definition},
  it is then easy to see both $\nabla \pi_\theta(a|s)$ and $\nabla \ent{\pi_\theta(\cdot | s)}$ are bounded and Lipschitz continuous in $\theta$.
  With Lemma~\ref{lem product of lipschitz functions},
  we, therefore, conclude that 
  there exist continuous functions $L_{\Lambda_1}(\eta)$ and $U_{\Lambda_1}(\eta)$ such that
  for any $s$,
  \begin{align}
    \norm{\Lambda_1(s, \theta, \eta) - \Lambda_1(s, \theta', \eta) } &\leq L_{\Lambda_1}(\eta) \norm{\theta -\theta'}, \\
    \sup_{\theta} \norm{\Lambda_1(s, \theta, \eta)} &\leq U_{\Lambda_1}(\eta).
  \end{align}
  We now study the Lipschitz continuity of $\bar \Lambda_1(\theta, \eta)$ defined in \eqref{eq actor helper function sac}.
  Lemma~\ref{lem continuity of ergodic distribution} confirms the Lipschitz continuity of $d_{\mu_\theta}$.
  Consequently,
  Lemma~\ref{lem product of lipschitz functions} implies that 
  there exist continuous functions $L_{\bar \Lambda_1}(\eta)$ and $U_{\bar \Lambda_1}(\eta)$ such that
  \begin{align}
    \norm{\bar \Lambda_1(\theta, \eta) - \bar \Lambda_1(\theta', \eta)} &\leq L_{\bar \Lambda_1}(\eta) \norm{\theta -\theta'}, \\
    \sup_{\theta} \norm{\bar \Lambda_1(\theta, y, \eta)} &\leq U_{\bar \Lambda_1}(\eta).
  \end{align}
  We now study the Lipschitz continuity of $\Lambda_1'(\theta, y, \eta)$ defined in \eqref{eq actor helper function 2 sac}.
  Since $\tilde J_\eta(\theta)$ is $L_J + \eta L_H$ smooth,
  Lemma~\ref{lem smooth definition} implies that
  $L_J + \eta L_H$ is a Lipschitz constant of $\nabla \tilde J_\eta(\theta)$.
  From Lemma~\ref{lem softmax policy gradient},
  it is easy to see the upper bound of $\tilde \nabla J_\eta(\theta)$ is also a continuous function of $\eta$.
  Consequently,
  Lemma~\ref{lem product of lipschitz functions} implies
  there exist continuous functions $L_{\Lambda_1'}(\eta)$ and $U_{\Lambda_1'}(\eta)$ such that
  for all $y$,
  \begin{align}
    \norm{\Lambda_1'(\theta, y, \eta) - \Lambda_1'(\theta', y, \eta)} &\leq L_{\Lambda_1'}(\eta) \norm{\theta -\theta'}, \\
    \sup_\theta \norm{\Lambda_1'(\theta, y, \eta)} &\leq U_{\Lambda_1'}(\eta).
  \end{align}
  Hence
  \begin{align}
    \norm{\E\left[\tilde M_{121}\right]} &= \norm{\Lambda_1'(\theta_t, S_t, \lambda_t) - \Lambda_1'(\theta_{t-\tau_{\beta_t}}, S_t, \lambda_t)} \\
    &\leq L_{\Lambda_1'}(\lambda_t) \norm{\theta_t - \theta_{t-\tau_{\beta_t}}} \\
    &\leq L_{\Lambda_1'}(\lambda_t) L_\theta \beta_{t-\tau_{\beta_t}, t-1} \qq{(Using \eqref{eq ltheta sac})}.
  \end{align}
  Since $\lambda_t \in [0, \lambda]$, 
  $L_{\Lambda_1'}(\eta)$ is a continuous function and well defined in $[0, \lambda]$,
  the extreme value theorem asserts that $L_{\Lambda_1'}(\eta)$ obtains its maximum in $[0, \lambda]$,
  say, e.g., $L_{\Lambda_1'}^*$.
  Then 
  \begin{align}
    \norm{\E\left[M_{121}\right]} \leq L_{\Lambda_1'}^* L_\theta \beta_{t-\tau_{\beta_t}, t-1},
  \end{align}
  which completes the proof.
\end{proof}

\subsection{Proof of Lemma~\ref{lem bound j lambda diff sac}}
\label{sec proof lem bound j lambda diff sac}
\lemboundjlambdadiffsac*
\begin{proof}
  \begin{align}
    &\sum_{k=\ceil{\frac{t}{2}}}^t \left(\frac{1}{\beta_k}\tilde J_{\lambda_k}(\theta_{k+1}) - \frac{1}{\beta_k}\tilde J_{\lambda_k}(\theta_k) \right) \\
    =& \sum_{k=\ceil{\frac{t}{2}}}^t\left(\frac{1}{\beta_{k-1}}\tilde J_{\lambda_{k-1}}(\theta_k) - \frac{1}{\beta_k}\tilde J_{\lambda_k}(\theta_k) \right) + \frac{1}{\beta_t} \tilde J_{\lambda_t}(\theta_{t+1}) - \frac{1}{\beta_{\ceil{\frac{t}{2}}-1}} \tilde J_{\lambda_{\ceil{\frac{t}{2}}-1}}(\theta_{\ceil{\frac{t}{2}}}) \\
    =& \sum_{k=\ceil{\frac{t}{2}}}^t\left(\frac{1}{\beta_{k-1}}\tilde J_{\lambda_{k-1}}(\theta_{k}) - \frac{1}{\beta_{k-1}} \tilde J_{\lambda_k}(\theta_k) + \frac{1}{\beta_{k-1}} \tilde J_{\lambda_k}(\theta_k) - \frac{1}{\beta_k}\tilde J_{\lambda_k}(\theta_k) \right) \\
    &+ \frac{1}{\beta_t} \tilde J_{\lambda_t}(\theta_{t+1}) - \frac{1}{\beta_{\ceil{\frac{t}{2}} - 1}} \tilde J_{\lambda_{\ceil{\frac{t}{2}}- 1}}(\theta_{\ceil{\frac{t}{2}}}) \\
    =& \sum_{k=\ceil{\frac{t}{2}}}^t\left(\frac{1}{\beta_{k-1}}(\lambda_{k-1} - \lambda_{k}) \frac{\sum_s d_{\pi_{\theta_k}, \gamma}(s) \ent{\pi_{\theta_k}(\cdot |s)}}{1 - \gamma}  + \left(\frac{1}{\beta_{k-1}} - \frac{1}{\beta_k}\right) \tilde J_{\lambda_k}(\theta_k) \right) \\
    &+ \frac{1}{\beta_t} \tilde J_{\lambda_t}(\theta_{t+1}) - \frac{1}{\beta_{\ceil{\frac{t}{2}}-1}} \tilde J_{\lambda_{\ceil{\frac{t}{2}} - 1}}(\theta_{\ceil{\frac{t}{2}}}) \\
    \leq& \sum_{k=\ceil{\frac{t}{2}}}^t\left(\frac{1}{\beta_{k-1}}\left(\frac{\lambda}{(k+t_0-1)^{\epsilon_\lambda}} - \frac{\lambda}{(k+t_0)^{\epsilon_\lambda}}\right) \frac{\log \na}{1-\gamma} + \left(\frac{1}{\beta_{k-1}} - \frac{1}{\beta_k}\right) \tilde J_{\lambda_k}(\theta_k) \right) \\
    &+ \frac{1}{\beta_t} \tilde J_{\lambda_t}(\theta_{t+1}) - \frac{1}{\beta_{\ceil{\frac{t}{2}}-1}} \tilde J_{\lambda_{\ceil{\frac{t}{2}} - 1}}(\theta_{\ceil{\frac{t}{2}}}) \\
    \overset{(i)}{\leq} & \sum_{k=\ceil{\frac{t}{2}}}^t \left(\frac{1}{\beta_{k-1}(k-1+t_0)^{1+\epsilon_\lambda}} \frac{\lambda \log \na}{1-\gamma} +  \left(\frac{1}{\beta_{k-1}} - \frac{1}{\beta_k}\right) \tilde J_{\lambda_k}(\theta_k) \right) \\
    &+ \frac{1}{\beta_t} \tilde J_{\lambda_t}(\theta_{t+1}) - \frac{1}{\beta_{\ceil{\frac{t}{2}}-1}} \tilde J_{\lambda_{\ceil{\frac{t}{2}} - 1}}(\theta_{\ceil{\frac{t}{2}}}), \\
    \overset{(ii)}{\leq} & \frac{3\lambda\beta \log \na}{1-\gamma} + \sum_{k=\ceil{\frac{t}{2}}}^t\left(\frac{1}{\beta_k} - \frac{1}{\beta_{k-1}} \right) U_{\tilde J}  + \frac{1}{\beta_t} U_{\tilde J} + \frac{1}{\beta_{\ceil{\frac{t}{2}}-1}} U_{\tilde J} \\
    = & \frac{3\lambda\beta \log \na}{1-\gamma} + \frac{U_{\tilde J}}{\beta_t} - \frac{U_{\tilde J}}{\beta_{\ceil{\frac{t}{2}}-1}} + \frac{1}{\beta_t} U_{\tilde J} + \frac{1}{\beta_{\ceil{\frac{t}{2}}-1}} U_{\tilde J}  \\
    = &\frac{3\lambda\beta \log \na}{1-\gamma} + \frac{2U_{\tilde J}}{\beta} (t+t_0)^{\epsilon_\beta},
\end{align}
where $(i)$ results from the inequality
\begin{align}
  \frac{1}{(t-1)^x} - \frac{1}{t^x} &= \frac{t^x - (t-1)^x}{(t-1)^x t^x} = \frac{t^x(t-1)^{1-x} - (t-1)}{(t-1)t^x} \\
  &\leq \frac{t^x t^{1-x} - (t-1)}{(t-1)(t-1)^x} = \frac{1}{(t-1)^{1+x}}
\end{align}
and $(ii)$ results from the inequality
\begin{align}
  \sum_{t=1}^\infty \frac{1}{t^{1+\epsilon_\lambda + \epsilon_\beta}} \leq \sum_{t=1}^\infty \frac{1}{t^{1.5}} < 3.
\end{align}
\end{proof} 
    
\vskip 0.2in
\bibliography{ref_updated}

\end{document}